\documentclass[letterpaper]{article} 
\usepackage[preprint]{aaai2027}  
\usepackage[hyphens]{url}  
\usepackage{graphicx} 
\usepackage{natbib}  
\usepackage{caption} 
\usepackage{amsmath,amssymb,amsthm}
\usepackage{algorithm}
\usepackage{algpseudocode}
\usepackage{subcaption}

\newcommand{\LongState}[1]{\State \parbox[t]{\dimexpr\linewidth-\algorithmicindent}{#1\strut}}

\newtheorem{assumption}{Assumption}
\newtheorem{proposition}{Proposition}
\newtheorem{theorem}{Theorem}
\newtheorem{lemma}{Lemma}
\newtheorem{corollary}{Corollary}

\usepackage{newfloat}
\usepackage{listings}
\DeclareCaptionStyle{ruled}{labelfont=normalfont,labelsep=colon,strut=off} 
\floatstyle{ruled}
\newfloat{listing}{tb}{lst}{}
\floatname{listing}{Listing}

\usepackage{booktabs}

\title{Robust General Utility for Reinforcement Learning}
\author{
    Zixuan Liu\corresponding, Fangzheng Wu, Brian Summa, Zizhan Zheng
}
\affiliations{

    Department of Computer Science, Tulane University, New Orleans, LA, 70118, USA\\
    \{zliu41, fwu6, bsumma, zzheng3\}@tulane.edu
}

\begin{document}

\maketitle

\begin{abstract}
Reinforcement learning (RL) with general utility extends classic RL by optimizing an arbitrary utility functional of the policy-induced occupancy measure, thereby enabling a broader range of applications. However, previous work on general utility RL typically assumes the evaluation utility is fixed and correctly specified. In practice, the utility used at deployment can deviate from the training one, creating a robustness gap that prior work does not address. Motivated by this, we propose \textbf{robust general-utility RL}, a minimax learning framework that trains policies against utility misspecification within a prescribed uncertainty set. Our framework strictly generalizes standard general-utility RL while also providing a unified view of many existing RL frameworks, including reward-robust RL and constrained RL, through appropriate choices of the utility uncertainty set. We further develop provably convergent stochastic algorithms for two regimes. For concave utilities, we develop a projected stochastic gradient descent-ascent method and establish stationarity guarantees. For the more challenging nonconcave regime, we propose a stochastic prox-extragradient algorithm that mitigates ill-posed behavior induced by nonconcavity, with convergence guarantees to approximate first-order stationarity. Experiments on LLM safety alignment and exploration maximization tasks further corroborate the convergence behavior consistent with our theory.
\end{abstract}


\section{Introduction}
The reinforcement learning (RL) framework seeks an optimal policy that maximizes the expected cumulative reward, which can be written as a linear function of the occupancy measure~\citep{manne1960linear,abbasi2019large,pmlr-v125-agarwal20a}. In more general settings, however, the objective may be a nonlinear functional of the state-action occupancy measure induced by the policy~\citep{zhang2022multi,hazan2019provably,zhang2020variational}. For instance, in exploration maximization, one aims to learn a policy that explores the state space by maximizing the entropy of the state occupancy measure induced by the policy~\citep{hazan2019provably,geist2021concave,zhang2020variational,barakat2023reinforcement}. Other examples include imitation learning~\citep{ho2016generative}, risk-sensitive RL~\citep{kallenberg1994survey,borkar2002risk,zhang2021cautious}, and safety-constrained RL~\citep{garcia2015comprehensive}. These settings motivate RL with general utility, which extends standard RL by allowing an arbitrary utility defined on the occupancy measure and seeks a policy that minimizes this utility~\citep{barakat2024towards,barakat2024global,ying2023scalable,wu2023risk}. A number of provably convergent and computationally efficient algorithms have been developed for this formulation~\citep{zhang2020variational,zhang2021convergence}.

Nevertheless, deploying RL-with-general-utility algorithms in the real world remains challenging: policies are often trained in simulation but executed in a different environment~\citep{peng2018sim}. Prior work~\citep{chen2024robust} primarily addresses this gap by making policies robust to dynamics changes, i.e., by training a general-utility policy against the worst-case transition model within a prescribed uncertainty set. However, real-world discrepancies extend beyond transition dynamics. In particular, the utility function used for policy evaluation may also differ from the one specified during training due to human biases, modeling errors, or distribution shift~\citep{mannor2007bias,jeon2020reward,enders2024risk}. Existing robust general-utility frameworks that focus on dynamics uncertainty therefore do not address robustness to such utility changes. Although reward-robust RL considers finding the optimal policy under worst-case rewards~\citep{morimoto2005robust,lim2013reinforcement,pinto2017robust,anonymous2025robust}, these approaches typically assume linear utility (the standard cumulative reward) and thus do not extend directly to general utility functionals. Building on this gap, we study robustness to general-utility misspecification and develop provably convergent algorithms for learning policies whose performance remains reliable under utility shifts at deployment.
\subsection{Main Contribution}

In this work, we propose \textbf{robust general-utility RL}, a new learning framework for obtaining policies that are robust to misspecification of the utility functional. We formulate the problem as the minimax objective $\min_{\theta \in \Theta}\max_{\xi \in \Xi} f_\xi(\lambda_\theta)$, where $f_\xi$ denotes a utility functional parameterized by $\xi \in \Xi$, and $\lambda_\theta$ is the occupancy measure induced by the policy parameter $\theta \in \Theta$. Notably, this formulation strictly generalizes RL with general utility, which is recovered by fixing $\xi$ and solving $\min_{\theta \in \Theta} f_\xi(\lambda_\theta)$. More importantly, our framework unifies several existing RL paradigms from a different view: rather than categorizing methods by how they modify rewards, we view them as different ways of specifying the uncertainty set $\Xi$. Reward-robust RL, for example, corresponds to choosing $\Xi$ to contain all plausible reward functions and restricting $f$ to be linear in $\xi$ and $\lambda_\theta$.
We further detail this unifying interpretation in Section~\ref{sec:examples} and Appendix~\ref{app:connection}.

Next, we develop \textit{provably convergent algorithms} for the proposed framework. We begin with a widely studied setting in general-utility RL, where the utility functional is strongly concave in $\xi$. To solve this, we propose a stochastic projected gradient descent-ascent method (PGDA, Algorithm~\ref{alg:pgda}). Our approach introduces the robust envelope $\Gamma(\theta) \triangleq \max_{\xi \in \Xi} f_\xi(\lambda_\theta)$ and alternates between approximately maximizing over $\xi \in \Xi$ 
and updating $\theta \in \Theta$ 
on $\Gamma(\theta)$. The convergence analysis is substantially more involved than in structured special cases such as reward-robust RL, where the utility is linear in both $\xi$ and $\lambda_\theta$. In our general setting, $f_\xi(\lambda_\theta)$ can couple $\xi$ and $\lambda_\theta$ nonlinearly, making the behavior of the robust envelope $\Gamma(\theta)$ considerably more complex. 
We address this challenge by first establishing the differentiability and smoothness of $\Gamma(\theta)$ (Proposition~\ref{prop:Gamma-smooth}), which provides the key regularity needed for the convergence analysis. We then prove that Algorithm~\ref{alg:pgda} converges to a first-order stationary point at rate $\mathcal{O}(1/\sqrt{K})$ (Theorem~\ref{theorem:pgda_converge}). 

Prior work has largely focused on utility functions with concavity structure. In our framework, however, such concavity need not be preserved 
and nonconcavity in $\xi$ can arise naturally from practically meaningful forms of utility misspecification. To illustrate this, we consider exploration maximization~\citep{zhang2020variational}, where exploration is characterized through state-action features. When these features change, modeling the resulting feature drift through the uncertainty parameter $\xi$ naturally induces a nonconcave dependence of $f_\xi(\lambda_\theta)$ on $\xi$ (Section~\ref{sec:nonconcave-example}). \textit{Motivated by this observation, we go beyond the settings considered in prior work and develop an algorithm with convergence guarantees for the substantially more challenging nonconcave case} (Algorithm~\ref{alg:pe-pgda}). 
In particular, the lack of concavity makes the inner maximization unstable and can cause PGDA to cycle rather than converge. To address these issues, we adopt proximal-point ideas~\citep{grimmer2023landscape}. Instead of solving the original nonconcave problem, we solve a prox-stabilized subproblem of the form $f_\xi(\lambda_\theta) + \frac{\sigma}{2}\|\theta-\theta'\|^2- \frac{\sigma}{2}\|\xi-\xi'\|^2$, where $\theta', \xi'$ are chosen anchor points and $\sigma>0$ is selected to make the objective well behaved. 
Our convergence analysis then proceeds by first relating stationarity of each prox-stabilized subproblem to stationarity of the original problem (Lemma~\ref{lem:prox_to_orig}). We next show that, with appropriately chosen hyperparemters, Algorithm~\ref{alg:pe-pgda} solves each prox-stabilized subproblem with an error that decreases at the rate $\mathcal{O}\!\left(1/(k+1)\right)$ (Proposition~\ref{prop:inner-gap-decay}). Combining these results, we establish convergence of Algorithm~\ref{alg:pe-pgda} 
(Theorem~\ref{thm:nonconcave-convergence}). 

We empirically evaluate the convergence behavior of both proposed algorithms. For the concave utility setting, we validate Algorithm~\ref{alg:pgda} on an LLM safety alignment task, \textit{extending empirical evaluation beyond the tabular settings considered in prior work}. For the nonconcave setting, we evaluate Algorithm~\ref{alg:pe-pgda} on the proposed \textit{exploration maximization} example. In both cases, the algorithms exhibit convergence behavior in practice that is consistent with our theoretical guarantees.

A detailed discussion of related work is in Appendix~\ref{sec:related}, with proofs and additional technical details in Appendix~\ref{app:proofs}. 

\section{Preliminary}

\textbf{Notations.} 
For a finite set \(\mathcal{Z}\), we write \(\mathbb{R}^{\mathcal{Z}}\) for the set of real-valued functions on \(\mathcal{Z}\), and \(\Delta^{\mathcal{Z}}\) for the probability simplex over \(\mathcal{Z}\). We denote by $|\mathcal{Z}|$ the cardinality of \(\mathcal{Z}\).  For \(a,b \in \mathbb{R}^{\mathcal{Z}}\), we define the standard inner product as \(\langle a,b\rangle := \sum_{z\in\mathcal{Z}} a(z)\,b(z)\). We use notation $\| \cdot \|_p$ for $p$-norm and $\| \cdot \| =\| \cdot \|_2 $ by default.

\noindent \textbf{RL with General Utility.}
Consider a discrete-time discounted Markov Decision Process (MDP), specified by a tuple
$\langle \mathcal{S}, \mathcal{A}, p, f_{\xi}, \rho, \gamma \rangle$,
with finite state space $\mathcal{S}$, finite action space $\mathcal{A}$, transition kernel
$p \in (\Delta^{\mathcal{S}})^{\mathcal{S}\times\mathcal{A}}$,
discount factor $\gamma \in (0,1)$, general utility function $f_{\xi} : \Delta^{\mathcal{S}\times\mathcal{A}} \to \mathbb{R}$ parameterized by $\xi \in \Xi$ and the distribution $\rho \in \Delta^{\mathcal{S}}$ of the initial state $s_0$.
At time $t$, given the environmental state $s_t$, the agent takes action
$a_t \sim \pi_{\theta}(\cdot \mid s_t)$ based on a policy
$\pi_{\theta} \in (\Delta^{\mathcal{A}})^{\mathcal{S}}$ parameterized by $\theta \in \Theta$.
Then the environment transitions to state $s_{t+1} \sim p(\cdot \mid s_t, a_t)$.
The occupancy measure $\lambda_{\theta} \in \Delta^{\mathcal{S}\times\mathcal{A}}$ at $(s,a)\in\mathcal{S}\times\mathcal{A}$ is defined as \(\lambda_{\theta}(s,a) \;\overset{\mathrm{def}}{=}\; (1-\gamma)\sum_{t=0}^{+\infty}\gamma^{t}\,\mathbb{P}_{\pi_{\theta},\,p}\!\big( s_t = s,\, a_t = a \,\big|\, s_0 \sim \rho \big),\) where $\mathbb{P}_{\pi_{\theta},p}$ denotes the probability measure of the Markov chain
$\{s_t,a_t\}_{t\ge 0}$ induced by policy $\pi_{\theta}$ and transition kernel $p$. Given a fixed utility function $f_{\xi}$, RL with general utility aims to find the optimal policy $\pi_{\theta}$ that solves:
\begin{equation}
\label{eq:general-utility}
\min_{\theta \in \Theta} f_{\xi}(\lambda_{\theta}) 
\end{equation}
Here, $f_{\xi}(\lambda_{\theta})$ can be seen as the overall cost of selecting policy $\pi_{\theta}$
in the environment.

\section{Robust General Utility RL}

As discussed above, most existing works assume a fixed utility function $f_\xi$. 
This assumption can fail at deployment time, when the utility used for evaluation shifts relative to the one used during training. In this section, we address this issue by introducing our robust general-utility RL framework in Section~\ref{sec:prob-form}. We then illustrate how the framework provides a unified perspective on several existing RL formulations in Section~\ref{sec:examples}, and we present the gradients and assumptions needed for algorithm design and analysis in Section~\ref{sec:grads-ass}.

\subsection{Problem Formulation}
\label{sec:prob-form}
The goal of our proposed robust general utility RL is to find an optimal robust policy under the worst possible utility function $f_\xi$ from an ambiguity set $\Xi$, as formulated by the following minimax optimization problem:
\begin{equation}
\label{eq:robust-utility}
\min_{\theta \in \Theta}\;\max_{\xi \in \Xi}\;f_\xi(\lambda_\theta)
\end{equation}
Similar to reward-robust RL, in practice, we typically assume that the utility used at test time is not too far from the utility employed for training, in the sense that their parameters differ by at most a known tolerance. In this case, a natural choice is to take $\Xi$ to be a ball centered at $\hat{\xi}$ (estimated from the simulator): \(
\Xi \;=\; \bigl\{\, \xi \in \mathbb{R}^{d_{\Xi}} : d(\xi,\hat{\xi}) \le d_0 \,\bigr\},
\)
where $d(\cdot,\cdot)$ is a user-chosen distance (e.g., an $\ell_p$ metric) and $d_0 \ge 0$ controls the size of the ambiguity set. We discuss a widely used $\ell_p$ norm ambiguity set for the proposed framework in Appendix~\ref{app:pgda-lp}. Moreover, throughout the paper, we assume the following property holds, which is also standard in prior work~\citep{chen2024robust}:

\begin{assumption}
\label{ass:xi-convex-compact} $\Theta \subset \mathbb{R}^{d_{\Theta}}$ is convex and compact with diameter $D_\Theta:= \max_{\theta,\theta'\in\Theta}\|\theta-\theta'\|$. $\Xi \subset \mathbb{R}^{d_{\Xi}}$ is convex and compact with diameter $D_\Xi := \max_{\xi,\xi'}\|\xi'-\xi\|$. 
\end{assumption}

\subsection{Examples of Robust Utility RL}
\label{sec:examples}

\textbf{Example 1: RL with General Utility.}
When $\Xi = \bigl\{\hat{\xi}\bigr\}$ for a fixed utility function parameter $\hat{\xi}$, then our proposed objective in~\eqref{eq:robust-utility} reduces to RL with general utility in~\eqref{eq:general-utility}.

\noindent\textbf{Example 2: Reward-Robust RL.} 
In reward-robust RL, the agent assumes the unknown reward function lies in a prescribed uncertainty set and seeks a policy that performs well against the worst-case reward realization~\citep{gadot2024solving,morimoto2005robust,wang2020reinforcement,he2023robotic,yan2024reward}. 
Formally, let $R:\mathcal{S}\times\mathcal{A}\to\mathbb{R}$ be a reward function and let $\mathcal{R}\subseteq \mathbb{R}^{\mathcal{S}\times\mathcal{A}}$ be the reward uncertainty set. Define the value function as follows: \(
V_\theta \overset{\text{def}}{=} \langle R,\lambda_\theta\rangle
= \sum_{s,a} R(s,a)\,\lambda_\theta(s,a).
\)
Then the goal of reward-robust RL is to maximize the following objective: 
\begin{equation}
\label{eq:reward-robust-rl}
    \max_{\theta \in \Theta} \min_{R \in \mathcal{R}} V_\theta
\end{equation}

\begin{proposition}
\label{prop:reward_robust_as_ru}
Let $\Xi \equiv \mathcal{R}$ be a nonempty, convex, compact set of reward functions, and define, for each $\xi\in\Xi$, the utility functional \(
f_\xi(\lambda_\theta) \overset{\text{def}}{=} - \langle \xi,\lambda_\theta\rangle.
\)
Then the reward-robust RL problem in~\eqref{eq:reward-robust-rl} is a linear special case of our robust utility problem in~\eqref{eq:robust-utility}.
\end{proposition}

\noindent\textbf{Example 3: Constrained RL (Lagrangian view).}
Constrained RL is a framework where an agent should obey safety conditions~\citep{gu2024review,altman2021constrained,achiam2017constrained}. Formally, denote $c^{(0)}, c^{(1)}, \ldots, c^{(K)}$ as the set of cost functions 
$\mathcal{S}\times\mathcal{A}\to\mathbb{R}$. 
At time $t$, the agent receives performance related cost $c^{(0)}(s_t,a_t)$ and 
safety-related costs $\{c^{(k)}(s_t,a_t)\}_{k=1}^K$. 
Define value functions $V_{\theta}^{(k)}$ as: \(
V_{\theta}^{(k)} \overset{\text{def}}{=} 
\langle c^{(k)}, \lambda_{\theta} \rangle 
= \sum_{s,a} c^{(k)}(s,a)\,\lambda_{\theta}(s,a).
\)
Then constrained RL is formulated as the following constrained policy optimization problem:
\begin{equation}
\label{eq:constrained_rl}
\min_{\theta\in\Theta} \; V_\theta^{(0)}
\quad \text{s.t.} \quad 
V_\theta^{(k)} \le \tau_k \;\; \text{for all } k=1,\ldots,K,
\end{equation}
where $\tau_k \in \mathbb{R}$ is the safety threshold.

\begin{proposition}
\label{prop:crl_as_ru}
Consider the constrained RL problem in \eqref{eq:constrained_rl}.
Assume Slater’s condition holds, i.e., there exists a policy $\tilde\theta$ such that
$V_{\tilde\theta}^{(k)} < \tau_k$ for all $k=1,\ldots,K$.
Let $\Xi \subset \mathbb{R}_+^{K}$ be any nonempty, convex, compact set that contains at least one optimal Lagrange multiplier vector (e.g., a sufficiently large $\ell_1$-ball intersected with $\mathbb{R}_+^K$).
For each $\xi=(\xi_1,\ldots,\xi_K)\in\Xi$, define the linear utility functional \(
f_{\xi}(\lambda_{\theta})
\overset{\text{def}}{=}\langle c^{(0)} + \sum_{k=1}^K \xi_kc^{(k)},\lambda_{\theta}\rangle
-\sum_{k=1}^K \xi_k\,\tau_k.\) Then the constrained RL problem \eqref{eq:constrained_rl} is equivalent to the robust utility problem in \eqref{eq:robust-utility} and hence is a linear special case of the robust utility formulation.
\end{proposition}

\subsection{Gradients and Assumptions for Robust Utility RL}
\label{sec:grads-ass}

\begin{theorem}
\label{theorem:gradient}
The gradients of~\eqref{eq:robust-utility} w.r.t. $\theta$ 
is \(
\nabla_{\theta} f_{\xi}(\lambda_{\theta})=
\mathbb{E}_{\pi_{\theta},\,p}[
\sum_{t=0}^{\infty} \gamma^{t}
\frac{\partial f_{\xi}(\lambda_{\theta})}{\partial \lambda_{\theta}(s_{t},a_{t})} 
(\sum_{h=0}^{t} \nabla_{\theta}\log \pi_{\theta}(a_{h}\mid s_{h})) |\ s_{0}\sim \rho].\)
\end{theorem}
In practice, we cannot get the exact gradients $\nabla_{\theta} f_{\xi}(\lambda_{\theta})$, $\nabla_{\xi} f_{\xi}(\lambda_{\theta})=\frac{\partial f_{\xi}(\lambda_{\theta})}{\partial \xi}$, and can only estimate them via stochastic samples. We include the details of how to estimate the gradients (Algorithm~\ref{alg:stochastic_gradient}) and the associated error bounds (Proposition~\ref{prop:error_grad_estimate}) in Appendix~\ref{app:grad_estimate}. Next, we make the following standard assumptions, which are commonly used in prior work~\citep{chen2024robust,zhang2021convergence,barakat2023reinforcement}. Then we derive several useful propositions that will be used throughout the paper.

\begin{assumption}
\label{ass:pi_theta_bound}
There exist constants $\ell_{\pi_\theta}, L_{\pi_\theta} > 0$ such that for all
$s \in \mathcal{S}$, $a \in \mathcal{A}$, and $\theta,\theta' \in \Theta$, \(
\big\| \nabla_{\theta} \log \pi_{\theta}(a \mid s) \big\| \;\le\; \ell_{\pi_\theta},
\) \(
\big\| \nabla_{\theta} \log \pi_{\theta'}(a \mid s)
      - \nabla_{\theta} \log \pi_{\theta}(a \mid s) \big\|
\;\le\; L_{\pi_\theta}\, \|\theta' - \theta\| .
\)
\end{assumption}
\begin{assumption}
\label{ass:lambda_bound}
There exist constants $\ell_\lambda,\,L_\lambda>0$
such that for all $\lambda,\lambda' \in \Delta^{\mathcal S\times\mathcal A}$, \(
\big\| \nabla_{\lambda} f_{\xi}(\lambda) \big\| \le \ell_\lambda,
\big\| \nabla_{\lambda} f_{\xi}(\lambda') - \nabla_{\lambda} f_{\xi}(\lambda) \big\|
\le L_\lambda \, \|\lambda' - \lambda\| .
\)
\end{assumption}

\begin{assumption}
\label{ass:xi_bound}
There exist constants $\ell_\xi, \, L_{\xi}>0$
such that for all $\xi,\xi' \in \Xi$,  \(
\|\nabla_{\xi} f_{\xi}(\lambda)\| \;\le\; \ell_{\xi},
\|\nabla_{\xi} f_{\xi'}(\lambda)-\nabla_{\xi} f_{\xi}(\lambda)\| \;\le\; L_\xi \|\xi' - \xi\|\)
\end{assumption}
\begin{assumption}
\label{ass:xi_lambda_bound}
There exist constants $L_{\lambda,\xi},\,L_{\xi,\lambda}>0$
such that for all $\xi,\xi' \in \Xi$ and $\lambda,\lambda' \in \Delta^{\mathcal S\times\mathcal A}$,
\(
\|\nabla_{\lambda} f_{\xi'}(\lambda) - \nabla_{\lambda} f_{\xi}(\lambda)\|
\;\le\; L_{\lambda,\xi}\,\|\xi' - \xi\|
\)
\(
\|\nabla_{\xi} f_{\xi}(\lambda') - \nabla_{\xi} f_{\xi}(\lambda)\|
\;\le\; L_{\xi,\lambda}\,\|\lambda' - \lambda\|
\)
\end{assumption}

\begin{proposition}
\label{prop:theta_bound} 
Under Assumptions~\ref{ass:pi_theta_bound} and~\ref{ass:lambda_bound}, we have the following bounds for any $\theta \in \Theta$.
\(
\big\|\nabla_{\theta} f_{\xi}(\lambda_{\theta})\big\| \;\le\; \ell_{\theta}, \text{with}\
\ell_{\theta} \;:=\; \frac{\ell_{\pi_\theta}\ell_{\lambda}}{(1-\gamma)^{2}}\, .
\)
\end{proposition}

\begin{proposition}
\label{prop:lips-theta-xi}
Under Assumptions~\ref{ass:pi_theta_bound} to~\ref{ass:xi_lambda_bound}, the gradients satisfy the following bounds for any \(\theta,\theta' \in \Theta\) and \(\xi,\xi' \in \Xi\):
\(
\|\nabla_\theta f_{\xi'}(\lambda_{\theta'}) - \nabla_\theta f_\xi(\lambda_\theta)\|
\le L_{\theta,\theta}\,\|\theta' - \theta\|
   + L_{\theta,\xi}\,\|\xi' - \xi\|,
\)
\(
\|\nabla_\xi f_{\xi'}(\lambda_{\theta'}) - \nabla_\xi f_\xi(\lambda_\theta)\|
\le L_{\xi,\theta}\,\|\theta' - \theta\|
   + L_{\xi,\xi}\,\|\xi' - \xi\|,
\)
where \(
L_{\theta,\theta}:= \frac{\ell_{\pi_\theta}^{2}
        \bigl(L_\lambda\sqrt{|{\mathcal A}|} + 2\ell_\lambda\bigr)}
        {(1-\gamma)^3}
   + \frac{L_{\pi_\theta}\,\ell_\lambda}{(1-\gamma)^2},
L_{\theta,\xi}:= \frac{\ell_{\pi_\theta}L_{\lambda,\xi}}
         {(1-\gamma)^2},
L_{\xi,\theta} := \frac{\ell_{\pi_\theta}L_{\xi,\lambda}\sqrt{|{\mathcal A}|}}{1-\gamma},
L_{\xi,\xi}:= L_\xi.
\)
\end{proposition}

\section{Algorithm and Convergence for Concave Utility}

In this section, we focus on the most common setting where the utility $f_\xi(\lambda_\theta)$ is strongly concave in $\xi$. Specifically, we present a provably convergent projected gradient descent ascent (PGDA) algorithm and establish its convergence guarantee. In addition, we show in Appendix~\ref{app:connection} how this algorithm recovers and unifies several prior methods across different RL literatures.

\begin{assumption}
\label{ass:f-convex-concave}
For every fixed $\xi$, the utility $f_\xi(\lambda)$ is convex in $\lambda$. For every fixed $\lambda$, the utility $f_\xi(\lambda)$ is $\mu_\xi$-strongly concave in $\xi$.
\end{assumption}
RL with concave utility functions subsumes many important special cases, as illustrated in Section~\ref{sec:examples}, and we adopt this setting throughout this section. Note that although $f$ is convex in $\lambda$, it is generally nonconvex in $\theta$ because the mapping $\theta \mapsto \lambda_{\theta}$ is typically nonconvex. Consequently, our robust utility objective remains a nonconvex-concave minimax optimization problem.

\begin{algorithm}[t]
\caption{Stochastic Projected Gradient Descent Ascent Algorithm (PGDA)}
\label{alg:pgda}
\begin{algorithmic}[1]
\State \textbf{Hyperparameters:} Iteration numbers $K,T$, stepsizes $\eta,\beta$, Monte-Carlo budgets $m, H, m', H'$.
\State \textbf{Initialize:} $\xi_0 \in \Xi,\ \theta_0,\in \Theta$.
\For{iterations $k = 0,1,\ldots,K-1$}
    \State Initialize $\xi_{k,0} \leftarrow \xi_k$
    \For{Inner steps $t=0,1,\ldots,T-1$}
    \LongState{Obtain $g^{(\xi)}_{k,t} \approx \nabla_\xi f_{\xi_{k,t}}(\lambda_{\theta_{k}})$ by Algorithm~\ref{alg:stochastic_gradient} with hyperparameters $m, H, m', H'$.}
    \LongState{Apply the projected stochastic gradient ascent step: \(
    \xi_{k, t+1}= \operatorname{proj}_\Xi\,\!\bigl(\xi_{k,t} + \beta\,g^{(\xi)}_{k,t}\bigr).
    \)}
    \EndFor
    \State Assign $\xi_{k+1} \leftarrow \xi_{k,T}$
    \LongState{Obtain $g^{(\theta)}_{k} \approx \nabla_\theta f_{\xi_{k+1}}(\lambda_{\theta_{k}})$ by Algorithm~\ref{alg:stochastic_gradient} with hyperparameters $m, H, m', H'$.}
    
    \LongState{Apply the projected stochastic gradient descent step: \(
        \theta_{k+1}
            = \operatorname{proj}_\Theta\,\!\bigl(\theta_{k} - \eta\,g^{(\theta)}_{k}\bigr).
    \)}
\EndFor
\State \textbf{Output:} $(\theta_K, \xi_K)$
\end{algorithmic}
\end{algorithm}

Inspired by prior gradient descent ascent methods~\citep{lin2020gradient,chen2024robust}, we design a general stochastic projected gradient descent ascent (PGDA) algorithm (Algorithm~\ref{alg:pgda}) for our nonconvex-concave problem under any ambiguity set $\Xi$ satisfying Assumption~\ref{ass:xi-convex-compact}. Note that for ambiguity sets with additional structure, for example, an $\ell_p$-norm ball, the inner update can be simplified (Appendix~\ref{app:pgda-lp}). We also discuss our algorithm in a model-based tabular setting (Appendix~\ref{app:pgda-tabular}), where we can optimize directly over the occupancy measure $\lambda$ and derive a global convergence guarantee. Nevertheless, Algorithm~\ref{alg:pgda} tackles the most general case by first explicitly resolving the inner maximization. Specifically, fix $\theta$, define the robust envelope $\Gamma(\theta) := \max_{\xi \in \Xi} f_\xi(\lambda_\theta)$ and $\Xi^\star(\theta):=\arg\max_{\xi\in\Xi} f_\xi(\lambda_\theta)$ be the maximizers. At each outer iteration, Algorithm~\ref{alg:pgda} approximately computes a maximizer $\xi^\star(\theta)\in\Xi^\star(\theta)$ by running projected stochastic gradient ascent on $\xi$ (lines~4--8). It then updates the policy parameter by projected stochastic gradient descent using the resulting inner solution (line~11), and repeats this alternating procedure (line~3--12). 

For the convergence analysis, a key step is to show that the envelope $\Gamma(\theta)$ is well-defined and admits a Lipschitz continuous gradient. 

\begin{proposition}
\label{prop:Gamma-smooth}
Let $\Gamma(\theta)$ be the robust envelop function, and suppose Assumptions~\ref{ass:xi-convex-compact} to~\ref{ass:f-convex-concave} hold. For each $\theta \in \Theta$, let \(
\xi^\star(\theta)\in\Xi^\star(\theta)
\)
be a maximizer. Then $\Gamma$ is differentiable on $\Theta$ with \(
  \nabla_\theta \Gamma(\theta)=
  \nabla_\theta f_{\xi^\star(\theta)}(\lambda_\theta),\) and its gradient is Lipschitz: \(
  \big\|\nabla_\theta \Gamma(\theta') - \nabla_\theta \Gamma(\theta)\big\|
  \le
  L_\Gamma \,\|\theta' - \theta\|,
  L_\Gamma := L_{\theta,\theta}
  + \frac{L_{\theta,\xi}L_{\xi,\theta}}{\mu_\xi}.\)

\end{proposition}
\textbf{Remark:} Differentiability and smoothness are not automatic for the robust envelope function, and in general they may fail to hold when the utility lacks strong concavity~\citep{nouiehed2019solving,wang2023policy}. Nevertheless, the proposition allows us to view Algorithm~\ref{alg:pgda} as a stochastic descent method on a single objective $\min_{\theta\in\Theta}\Gamma(\theta)$ and apply standard analysis to relate descent in $\Gamma$ to stationarity of $\theta$. Moreover, the Lipschitz gradient property provides quantitative control of the envelope bias incurred by solving the inner maximization only approximately. It also 
lets us bound the mismatch between the gradient used by the algorithm and the true envelope gradient $\nabla \Gamma(\theta)$.

Next, we show the gradient convergence result of Algorithm~\ref{alg:pgda} by the following theorem.
\begin{theorem}
\label{theorem:pgda_converge}
Consider Algorithm~\ref{alg:pgda} for the RL with robust utility problem in~\eqref{eq:robust-utility} and suppose Assumptions~\ref{ass:xi-convex-compact} to~\ref{ass:f-convex-concave} hold. Let the step size $\eta \le 1/L_\Gamma$ and $E_\theta$ be the error bounds of stochastic gradients estimation of $\theta$ (Proposition~\ref{prop:error_grad_estimate}), and let $\delta_k := \inf_{v \in N_\Xi(\xi_k)} \bigl\| -\nabla_\xi f_{\xi_k}(\lambda_{\theta_k}) + v \bigr\|$ be the residual of the $\xi$-iterate, where $N_\Xi(\xi_k):=\{v: \langle v, \xi-\xi_k \rangle \le 0, \forall \xi\in \Xi \}$ is the normal cone of $\Xi$ at $\xi_k$. Let $\Gamma(\theta)$ be the envelop function and $\Gamma^\star := \inf_{\theta\in\Theta}\Gamma(\theta)$ be the optimal envelop value. Define the outer projected gradient mapping for the envelope: \(
G_{\Gamma}(\theta)\;\triangleq\;\frac{1}{\eta}\left(\theta-\operatorname{proj}_{\Theta}\!\bigl(\theta-\eta\nabla\Gamma(\theta)\bigr)\right).
\) Then for any $K\ge 1$,
\(
\frac{1}{K}\sum_{k=0}^{K-1}\mathbb{E}\bigl\|G_{\Gamma}(\theta_k)\bigr\|^2
\le
\frac{8(\Gamma(\theta_0)-\Gamma^\star)}{\eta K}
+ 20 E_\theta + 20 L^2_{\theta,\xi}\frac{1}{\mu^2_\xi} \cdot \frac{1}{K}\sum_{k=0}^{K-1}\mathbb{E}[\delta_k^2].
\)
\end{theorem}

\textbf{Remark}: In particular, 
as the Monte-Carlo budgets $m, H, m', H' \rightarrow \infty$, the error $E_\theta \rightarrow 0$. If the inner $\xi$-update is run sufficiently long, then the averaged residual $\frac{1}{K}\sum_{k=0}^{K-1}\mathbb{E}[\delta_k^2] \rightarrow 0$. If we choose $\eta=\Theta(1/\sqrt K)$, then we have:
\(
\min_{0\le k<K}\,\mathbb{E}\bigl\|\nabla_{\theta}G_\Gamma(\theta_k)\bigr\|^2
\;\le\;
\frac{1}{K}\sum_{k=0}^{K-1}\mathbb{E}\bigl\|\nabla_{\theta}G_\Gamma(\theta_k)\bigr\|^2
\;=\;
\mathcal{O}\!\left(\frac{1}{\sqrt{K}}\right),
\) which means Algorithm~\ref{alg:pgda} converges to a first-order stationary point of the robust envelope $\Gamma(\theta)$ at a rate $\mathcal{O}\!\left(\frac{1}{\sqrt{K}}\right)$.

\section{Algorithm and Convergence for Nonconcave Utility}
\label{sec:nonconcave-u}

Although concave utilities have been widely studied in prior work, 
in general, $f_\xi(\lambda_\theta)$ need not be concave in $\xi$. In Section~\ref{sec:nonconcave-example}, we first construct a new, underexplored yet practically meaningful example that aims to maximize exploration and is nonconcave in $\xi$. Then, in Section~\ref{sec:pe-pgda}, we propose the Prox-Extragradient PGDA algorithm to solve this nonconcave case and derive a corresponding convergence guarantee.

\subsection{Example of Nonconcave Utility}
\label{sec:nonconcave-example}

\begin{algorithm}[t]
\caption{Stochastic Prox-Extragradient PGDA}
\label{alg:pe-pgda}
\begin{algorithmic}[1]

\Require
Outer iteration number $K$, inner iteration numbers $\{T_k\}_{k=0}^{K-1}$, stepsizes $\alpha$, proximal weight $\sigma$, target inner accuracies $\{\delta_k\}_{k=0}^{K-1}$, Monte-Carlo budgets $\{m_k,H_k,m_k',H_k'\}_{k=0}^{K-1}$.

\State \textbf{Initialize:} $\theta_0\in\Theta$ and $\xi_0\in\Xi$.

\For{$k=0,1,\ldots,K-1$}

    \LongState{Fix $(\theta_k,\xi_k)$ as the proximal center throughout the $k$-th inner loop.}

    \State Initialize the inner solver:
    \(\theta_{k,0}\leftarrow\theta_k,
    \xi_{k,0}\leftarrow\xi_k.
    \)

    \For{$t=0,1,\ldots,T_k-1$}

        \LongState{Obtain independent stochastic gradient estimates
        \( g_{k,t}^{(\theta)}
        \approx
        \nabla_\theta
        f_{\xi_{k,t}}(\lambda_{\theta_{k,t}}),
        g_{k,t}^{(\xi)}
        \approx
        \nabla_\xi
        f_{\xi_{k,t}}(\lambda_{\theta_{k,t}})
        \)
        using Algorithm~\ref{alg:stochastic_gradient} with
        Monte-Carlo budgets $m_k,H_k,m_k',H_k'$.}

        \LongState{Compute the prediction updates for $\theta$: \(\widetilde{\theta}_{k,t}=
    \operatorname{proj}_{\Theta}
        \left(
        \theta_{k,t}
        -
        \alpha
        \left[
        g_{k,t}^{(\theta)}
        +
        \sigma(\theta_{k,t}-\theta_k)
        \right]
        \right)\)
        }
        
        \LongState{Compute the prediction updates for $\xi$: \(
        \widetilde{\xi}_{k,t}=
        \operatorname{proj}_{\Xi}
        \left(
        \xi_{k,t}
        +
        \alpha
        \left[
        g_{k,t}^{(\xi)}
        -
        \sigma(\xi_{k,t}-\xi_k)
        \right]
        \right)
        \)
        }
        
        \LongState{Obtain independent stochastic gradient estimates at
        the prediction point:
        \(
        g_{k,t}^{(\theta)\prime}
        \approx
        \nabla_\theta
        f_{\widetilde{\xi}_{k,t}}(\lambda_{\widetilde{\theta}_{k,t}}),
        g_{k,t}^{(\xi)\prime}
        \approx
        \nabla_\xi
        f_{\widetilde{\xi}_{k,t}}(\lambda_{\widetilde{\theta}_{k,t}})
        \)
        using Algorithm~\ref{alg:stochastic_gradient}.
        }

        \LongState{Compute the correction updates for $\theta$: \(\theta_{k,t+1}=
    \operatorname{proj}_{\Theta}
        \left(
        \theta_{k,t}
        -
        \alpha
        \left[
        g_{k,t}^{(\theta)\prime}
        +
        \sigma(\widetilde{\theta}_{k,t}-\theta_k)
        \right]
        \right)\)}
        
        \LongState{Compute the correction updates for $\xi$: \(\xi_{k,t+1}=
        \operatorname{proj}_{\Xi}
        \left(
        \xi_{k,t}
        +
        \alpha
        \left[
        g_{k,t}^{(\xi)\prime}
        -
    \sigma(\widetilde{\xi}_{k,t}-\xi_k)
        \right]
        \right)\)}
    \EndFor

    \LongState{Assign the approximate solution of the $k$-th proximal subproblem: \(\theta_{k+1}\leftarrow\theta_{k,T_k},\xi_{k+1}\leftarrow\xi_{k,T_k}.
    \)}
\EndFor

\State Sample \(
\widehat{k}
\sim
\operatorname{Unif}\{0,1,\ldots,K-1\}.
\)

\State \Return
$(\theta_{\widehat{k}+1},\xi_{\widehat{k}+1})$.

\end{algorithmic}
\end{algorithm}

Consider Example 2.2 in~\citep{zhang2020variational}, where the utility promotes faster exploration by maximizing (equivalently, minimizing the negative of) the smallest eigenvalue of the following covariance matrix: \(
   \min_{\theta} -\sigma_{\min}\left(\sum_{s,a}\lambda_\theta(s,a)\phi(s,a)\phi(s,a)^\top \right),
   \)
where $\sigma_{\min}$ is the smallest singular value and $\phi(s,a) \in \mathbb{R}^d$ are state-action features. However, in deployment, the feature map can differ from training (e.g., replacing features, or adding/removing features), a common issue in reward hacking settings~\citep{pan2022effects,laidlaw2024correlated}. For example, in traffic management, the true features may combine commute time, vehicle acceleration, and headway, while training may use average speed as a proxy for commute time~\citep{laidlaw2024correlated}. To model such a feature shift, we consider a standard linear case. Let $\phi_\xi(s,a)=W_\xi \psi(s,a)$, where $\psi(s,a) \in \mathbb{R}^{d'}$ is a fixed base feature vector (typically provided by a simulator~\citep{pan2022effects,leike2017ai}), and $W_\xi :=(\xi_{i,j}) \in \mathbb{R}^{d\times d'}$ captures representation drift. Define \(
   f_\xi(\lambda_\theta):=-\sigma_{\min}(\sum_{s,a}\lambda_\theta(s,a)\phi_\xi(s,a)\phi_\xi(s,a)^\top).
\) Then $f_\xi$ is generally nonconcave on $\xi$. Concretely, with $\phi_\xi = W_\xi \psi$, \(
\sum_{s,a}\lambda_\theta(s,a)\,\phi_\xi \phi_\xi^{\top}
= W_\xi\left(\sum_{s,a}\lambda_\theta(s,a)\,\psi\psi^{\top}\right)W_\xi^{\top},
\)
so $f_\xi(\lambda)$ becomes \(
f_\xi(\lambda)
= -\sigma_{\min}\!\bigl(W_\xi\,M(\lambda)\,W_\xi^{\top}\bigr),
M(\lambda):=\sum_{s,a}\lambda_\theta(s,a)\,\psi\psi^{\top}\succeq 0.
\) Notice that the map $A \rightarrow -\sigma_{\min}(A)$ is convex over positive semidefinite matrices. However,  $W_\xi\mapsto W_\xi MW_\xi^{\top}$ is a quadratic mapping in $W_\xi$. Since convexity/concavity is generally preserved under composition only when the inner map is affine, composing $-\sigma_{\min}(\cdot)$ with this quadratic mapping does not, in general, preserve either convexity or concavity in $\xi$. Hence, robustness to such feature drift naturally leads to a nonconcave inner problem in $\xi$.

\subsection{Prox-Extragradient PGDA and Its Convergence Analysis}
\label{sec:pe-pgda}

In this section, we drop the concavity assumption in Assumption~\ref{ass:f-convex-concave} and instead assume the following assumption holds:

\begin{assumption}
\label{ass:f-convex-nonconcave}
For every fixed $\xi$, the utility $f_\xi(\lambda)$ is convex in $\lambda$. However, for every fixed $\lambda$, the utility $f_\xi(\lambda)$ is not concave in $\xi$.
\end{assumption}

When $f_{\xi}(\lambda_\theta)$ is not concave in $\xi$, the regularity behind
Proposition~\ref{prop:Gamma-smooth} breaks down: the robust envelope \(
\Gamma(\theta)
\) need not be differentiable, because the maximizer set
$\Xi^\star(\theta)$ can be non-unique and may change
discontinuously as $\theta$ varies. As a result, the inner maximization in Algorithm~\ref{alg:pgda} is no longer a well-defined oracle. For a fixed $\theta$,
projected ascent may converge to different local maximizers depending on initialization and stochasticity. This ill-posedness also propagates to the outer update and can lead PGDA to exhibit cycling behavior, i.e., iterates oscillate among multiple regions rather than converging.

To address these challenges, we propose Algorithm~\ref{alg:pe-pgda}, which makes the optimization problem well-behaved through proximal stabilization~\citep{grimmer2023landscape}. Instead of solving the original minimax problem in~\eqref{eq:robust-utility} directly, we solve the prox-regularized problem \(\Phi(\theta,\xi)=
\min_{\theta\in\Theta}\max_{\xi\in\Xi}
\left\{
f_{\xi}(\lambda_\theta)
+
\frac{\sigma}{2}|\theta-\theta'|^2-
\frac{\sigma}{2}|\xi-\xi'|^2
\right\}
\), where $\theta'$ and $\xi'$ are anchor points chosen as the initialization of the inner loop, e.g., $\theta'=\theta_k$ and $\xi'=\xi_k$. The negative quadratic regularization in \(\xi\) offsets the nonconcavity of the original inner maximization problem and stabilizes the adversarial update. We additionally regularize the \(\theta\)-variable to control the nonconvexity of the outer minimization problem. For a sufficiently large proximal weight \(\sigma\), the resulting prox-regularized problem is well behaved even when the original objective is nonconvex in \(\theta\) and nonconcave in \(\xi\). Algorithm~\ref{alg:pe-pgda} approximately solves each prox-regularized problem using stochastic extragradient updates. Throughout the \(k\)-th inner loop, the anchor points \((\theta_k,\xi_k)\) remain fixed (Line~3). At each inner iteration, the algorithm first computes the prediction points \((\widetilde{\theta}_{k,t},\widetilde{\xi}_{k,t})\) using stochastic gradients evaluated at the current inner iterates (lines~6--8). It then evaluates new stochastic gradients at the prediction points and uses them to compute the correction updates \((\theta_{k,t+1},\xi_{k,t+1})\) (lines~9--11). Evaluating the correction direction at the lookahead point mitigates the rotational and cycling behavior that may arise in the optimization. After sufficiently many inner iterations \(T_k\), the resulting approximate solution is assigned as the next proximal center \((\theta_{k+1},\xi_{k+1})\) (line~13), and the procedure is repeated.

To establish that the output $(\theta_{\widehat{k}+1},\xi_{\widehat{k}+1})$ of Algorithm~\ref{alg:pe-pgda} approaches a first-order stationary point of the original robust utility problem in~\eqref{eq:robust-utility}, we first introduce the standard one-step projected gradient mappings associated with~\eqref{eq:robust-utility}. For any $(\theta,\xi)$, define: 
\(
\mathcal{G}_\Theta(\theta,\xi):=\frac{1}{\alpha}\Bigl(\theta-\mathrm{proj}_{\Theta}\bigl(\theta-\alpha\nabla_\theta f_\xi(\lambda_\theta)\bigr)\Bigr), 
\)
\(
\mathcal{G}_\Xi(\theta,\xi):=\frac{1}{\alpha}\Bigl(\xi-\mathrm{proj}_{\Xi}\bigl(\xi+\alpha\nabla_\xi f_\xi(\lambda_\theta)\bigr)\Bigr),
\)
which are widely used to characterize first-order stationarity in constrained optimization~\citep{li2018simple,balashov2020gradient}. Based on these mappings, we define the gradient-mapping residual
\(
\mathcal{R}_{\rm GM}(\theta,\xi)
:=\|\mathcal{G}_\Theta(\theta,\xi)\|^2+\|\mathcal{G}_\Xi(\theta,\xi)\|^2.
\)
In particular, $\mathcal{R}_{\rm GM}(\theta,\xi)=0$ if and only if $(\theta,\xi)$ satisfies the first-order stationarity conditions of~\eqref{eq:robust-utility}. A key subtlety, however, is that Algorithm~\ref{alg:pe-pgda} does not directly apply projected gradient updates using only $\nabla_\theta f_\xi(\lambda_\theta)$ and $\nabla_\xi f_\xi(\lambda_\theta)$. Instead, at outer iteration $k$, it approximately solves a prox-regularized problem anchored at the current proximal center $(\theta_k,\xi_k)$, i.e. \(\Phi_k(\theta,\xi):=f_\xi(\lambda_\theta)+\frac{\sigma}{2}\|\theta-\theta_k\|^2-\frac{\sigma}{2}\|\xi-\xi_k\|^2.\) Accordingly, to characterize how accurately the iterate solves the \(k\)-th prox-regularized subproblem, we introduce the following proximal gradient mappings that are naturally aligned with the updates performed by Algorithm~\ref{alg:pe-pgda}:
\(
\mathcal{G}_{\Theta,\sigma}^{k}(\theta,\xi):=
\frac{1}{\alpha}
\left(
\theta
-
\operatorname{proj}_{\Theta}
\left(
\theta
-
\alpha
\left[
\nabla_\theta f_\xi(\lambda_\theta)
+
\sigma(\theta-\theta_k)
\right]
\right)
\right),
\)
\(
\mathcal{G}_{\Xi,\sigma}^{k}(\theta,\xi):=
\frac{1}{\alpha}
\left(
\xi
-
\operatorname{proj}_{\Xi}
\left(
\xi
+
\alpha
\left[
\nabla_\xi f_\xi(\lambda_\theta)
-
\sigma(\xi-\xi_k)
\right]
\right)
\right).
\)
These mappings measure first-order stationarity of the $k$-th prox-regularized subproblem. 

We next state our convergence analysis, which consists of three main steps. First, we relate the stationarity of each prox-regularized subproblem \(\Phi_k(\theta,\xi)\), measured by the proximal mapping \(\mathcal{G}_{\Theta,\sigma}^{k}
(\theta,\xi),\mathcal{G}^{k}_{\Xi,\sigma}(\theta,\xi)\), to stationarity of the original objective (Lemma~\ref{lem:prox_to_orig}). Second, we formulate each prox-regularized subproblem as a strongly monotone variational inequality (Lemma~\ref{lem:prox-saddle-VI-equivalence}) and show that, with sufficiently many inner iterations and appropriately chosen Monte-Carlo budgets, Algorithm~\ref{alg:pe-pgda} solves this variational inequality with an error that decreases at the rate \(\mathcal{O}\!\left(\frac{1}{k+1}\right)\) (Proposition~\ref{prop:inner-gap-decay}). Finally, under a standard Minty variational inequality assumption (Assumption~\ref{ass:minty-solution}), we establish convergence to a first-order stationary point of the original robust utility problem (Theorem~\ref{thm:nonconcave-convergence}).

\begin{lemma}
\label{lem:prox_to_orig}
Under Assumption~\ref{ass:xi-convex-compact}, for any $(\theta,\xi)$ and any anchor $\theta_k, \xi_k$, we have
\(
\left\|
\mathcal{G}_{\Theta}
(\theta,\xi)
\right\|
\le
\left\|
\mathcal{G}_{\Theta,\sigma}^{k}
(\theta,\xi)
\right\|
+
\sigma\|\theta-\theta_k\|.
\)
\(
\|\mathcal{G}_\Xi(\theta,\xi)\|
\;\le\;
\|\mathcal{G}^{k}_{\Xi,\sigma}(\theta,\xi)\|+\sigma\|\xi-\xi_k\|.
\)
\end{lemma}
Lemma~\ref{lem:prox_to_orig} shows that the gradient mapping norm for the original objective is controlled by two terms: the proximal mapping norm $\|
\mathcal{G}_{\Theta,\sigma}^{k}
(\theta,\xi)\|$, $\|\mathcal{G}^{k}_{\Xi,\sigma}(\theta,\xi)\|$ and the anchor discrepancy $\|\theta-\theta_k\|$, $\|\xi-\xi_k\|$. Consequently, in proving convergence, it suffices to first bound $\mathcal{G}_{\Theta,\sigma}^{k}
(\theta,\xi)$, $\mathcal{G}^{k}_{\Xi,\sigma}(\theta,\xi)$ and then convert this bound back to the original gradient mapping $\mathcal{G}_{\Theta}
(\theta,\xi)$, $\mathcal{G}_\Xi(\theta,\xi)$.

We next control the proximal stationarity error. For the subsequent analysis, let $\mathcal{Z}:=\Theta\times\Xi$, $z:=(\theta,\xi)$, $z_k:=(\theta_k,\xi_k)$ define the saddle-point operator associated with the original robust utility problem in~\eqref{eq:robust-utility} as \(
\mathcal{F}(z)
:=
\bigl(
\nabla_\theta f_\xi(\lambda_\theta),
-\nabla_\xi f_\xi(\lambda_\theta)
\bigr),
\) and define the saddle-point operator associated with the prox-regularized objective \(\Phi_k(\theta,\xi)\), for the $k$-th outer iteration, as \(
\mathcal{F}_k(z)
:=
\bigl(
\nabla_\theta \Phi_k(\theta,\xi),
-\nabla_\xi \Phi_k(\theta,\xi)
\bigr)
\). Define the joint gradient Lipschitz constants as \(L_F := |\begin{pmatrix} L_{\theta,\theta} & L_{\theta,\xi}\\ L_{\xi,\theta} & L_{\xi,\xi} \end{pmatrix}|_2,\) where $|\cdot|_2$ denotes the spectral norm. Choosing $\sigma>L_F$ makes $\mathcal{F}_k(z)$ $(\sigma-L_F)$-strongly monotone, which ensures that the corresponding prox-regularized subproblem $\Phi_k$ has a unique solution (Appendix~\ref{app:mon-saddle-oper}). The next lemma characterizes this solution through an equivalent variational inequality~\eqref{eq:kth-prox-VI}, which will be more convenient for analyzing the stochastic extragradient updates.

\begin{lemma}
\label{lem:prox-saddle-VI-equivalence}
Suppose Assumption~\ref{ass:xi-convex-compact} holds and $\sigma>L_{\mathcal F}$. Then a point $\bar z_{k+1}=(\bar\theta_{k+1},\bar\xi_{k+1})\in\mathcal Z$ is a saddle point of $\Phi_k$, i.e., \(\forall(\theta,\xi)\in\Theta\times\Xi,\) \(
\Phi_k(\bar\theta_{k+1},\xi)\le \Phi_k(\bar\theta_{k+1},\bar\xi_{k+1}) \le \Phi_k(\theta,\bar\xi_{k+1}), 
\) if and only if it solves the variational inequality
\begin{equation}
\label{eq:kth-prox-VI}
\left\langle
\mathcal F_k(\bar z_{k+1}),
z'-\bar z_{k+1}
\right\rangle
\ge 0,
\qquad
\forall z'\in\mathcal Z.
\end{equation}
Moreover, this solution is unique.
\end{lemma}

Lemma~\ref{lem:prox-saddle-VI-equivalence} shows that solving the prox-regularized saddle problem $\Phi_k(\theta,\xi)$ is equivalent to finding a solution of the variational inequality associated with $\mathcal{F}_k$~\eqref{eq:kth-prox-VI}. To measure the violation of~\eqref{eq:kth-prox-VI}, we use the classical variational-inequality gap function~\citep{juditsky2011solving,fukushima1992equivalent,larsson1994class}:
\begin{equation}
\label{eq:proximal-VI-gap}
\operatorname{Gap}_k(z)
:=
\max_{z'\in\mathcal{Z}}
\left\langle
\mathcal{F}_k(z),z-z'
\right\rangle.
\end{equation}
The equation selects the feasible test point that exhibits the largest violation of the first-order condition of~\eqref{eq:kth-prox-VI}. We show that $\operatorname{Gap}_k(z)$ is nonnegative and equals zero exactly at the unique solution $\bar z_{k+1}$ of the $k$-th prox-regularized subproblem (Appendix~\ref{app:gap_k_prop}). Consequently, a smaller value of $\operatorname{Gap}_k(z)$ indicates that $z$ more nearly satisfies the first-order conditions. It remains to determine how accurately the stochastic inner loop solves this variational inequality. 

\begin{proposition}
\label{prop:inner-gap-decay}
Suppose the conditions of
Lemma~\ref{lem:inner-stochastic-EG-convergence} hold. Choose the number of inner iterations $T_k$ and the Monte-Carlo budgets $m_k,H_k,m_k',H_k'$ as detailed in Appendix~\ref{app:inner-gap-decay}. Then \(\mathbb E_k\left[\operatorname{Gap}_k(z_{k+1})\right]\le\frac{c_\delta}{k+1},\) where \(c_\delta\) is a constant detailed in Appendix~\ref{app:inner-gap-decay}. 
\end{proposition}

Proposition~\ref{prop:inner-gap-decay} establishes that every prox-regularized subproblem is solved with expected gap at most $c_\delta/(k+1)$. The detailed analysis, given in Appendix~\ref{app:inner-gap-decay}, proceeds in three steps. We first bound the distance between the approximate solution \(z_{k+1}\) returned by the inner loop (Algorithm~\ref{alg:pe-pgda}, line 13) and the exact solution \(\bar z_{k+1}\) of the \(k\)-th prox-regularized subproblem, which contains two terms: the optimization error of the inner iterations which decreases gemoetrically with \(T_k\), and the error introduced by stochastic gradient estimation which is controlled by the Monte-Carlo budgets (Lemma~\ref{lem:inner-stochastic-EG-convergence}). We then relate this distance to the variational-inequality gap \(\mathbb E_k\left[\operatorname{Gap}_k(z_{k+1})\right]\) (Proposition~\ref{prop:inner-solver-complexity}). Finally, by choosing \(T_k\) to grow logarithmically with \(k\) (Corollary~\ref{cor:inner-gap-decay}) and increasing the Monte-Carlo budgets (Lemma~\ref{lem:MC_decay}) so that both terms are of order \(\mathcal{O}(1/(k+1)^2)\), we obtain a variational-inequality gap that decreases at the rate \(\mathcal{O}(1/(k+1))\).

The next lemma shows that the variational-inequality gap simultaneously controls the stationarity errors of both the $\theta$- and $\xi$-players.
\begin{lemma}
\label{lem:proximal-gap-mapping}
For every
$k\ge 0$,
\(
\mathbb{E}_k
\left[
\left\|
\mathcal{G}_{\Theta,\sigma}^{k}
(\theta_{k+1},\xi_{k+1})
\right\|^2
\right]\le
\frac{c_\delta}{\alpha(k+1)},
\)
\(
\mathbb{E}_k
\left[
\left\|
\mathcal{G}_{\Xi,\sigma}^{k}
(\theta_{k+1},\xi_{k+1})
\right\|^2
\right]\le \frac{c_\delta}{\alpha(k+1)}.
\)
\end{lemma}

However, accurate solutions of the individual prox-regularized subproblems alone do not guarantee convergence of the original problem. We must additionally ensure that the successive proximal anchor points do not cycle indefinitely. For this purpose, we impose the following standard structural condition.

\begin{assumption}
\label{ass:minty-solution}
There exists
$z^\star=(\theta^\star,\xi^\star)\in\mathcal{Z}$ such that \(
\left\langle
\nabla_\theta f_\xi(\lambda_\theta),
\theta-\theta^\star
\right\rangle
-
\left\langle
\nabla_\xi f_\xi(\lambda_\theta),
\xi-\xi^\star
\right\rangle
\ge 0,
\forall(\theta,\xi)\in\Theta\times\Xi.
\)
\end{assumption}

Assumption~\ref{ass:minty-solution} is commonly referred to as the Minty variation inequality (MVI) condition in the nonconvex-nonconcave minimax literature
\citep{liu2021first,diakonikolas2021efficient,cai2022accelerated}. Geometrically, it requires the saddle-point operator $\mathcal{F}(z)$ at every feasible point $z$ to have a nonnegative alignment with the displacement $z-z^\star$. Equivalently, the joint descent-ascent direction $-\mathcal{F}(z)$ has no component pointing away from $z^\star$. Thus, although the objective may be nonconvex in $\theta$ and nonconcave in $\xi$, its joint gradient field remains globally coherent with respect to at least one solution $z^\star$. Notice that this condition has been widely adopted in the nonconvex-noncave minimax optimization and is satisfied by various structured nonconvex--nonconcave problem classes~\citep{diakonikolas2021efficient,cai2022accelerated}. Therefore, we adopt this assumption to ensure that our proximal update makes aggregate progress toward $z^\star$, allowing the convergence guarantees for Algorithm~\ref{alg:pe-pgda}.

\begin{theorem}
\label{thm:nonconcave-convergence}
Suppose Assumptions~\ref{ass:xi-convex-compact}
to~\ref{ass:xi_lambda_bound},
Assumption~\ref{ass:f-convex-nonconcave}, and
Assumption~\ref{ass:minty-solution} hold. Define \(D_{\mathcal Z}^2:=D_\Theta^2+D_\Xi^2.\) Let
$\{(\theta_k,\xi_k)\}_{k=0}^{K}$ be generated by
Algorithm~\ref{alg:pe-pgda}, and choose
$\sigma>L_{\mathcal F}$. Let \(
\widehat{k}
\sim
\operatorname{Unif}\{0,1,\ldots,K-1\}.
\)
Then
\(
\mathbb{E}
\left[
\mathcal{R}_{\rm GM}
(\theta_{\widehat{k}+1},\xi_{\widehat{k}+1})
\right]\le
\frac{2\sigma^2D_{\mathcal Z}^2}{K}
+
c_\delta
\left(
\frac{4}{\alpha}
+
4\sigma
\right)
\frac{1+\log K}{K}.
\)
Consequently, \(
\mathbb{E}
\left[
\mathcal{R}_{\rm GM}
(\theta_{\widehat{k}+1},\xi_{\widehat{k}+1})
\right]
\longrightarrow 0, \text{as }K\to\infty.
\) Hence, Algorithm~\ref{alg:pe-pgda} converges to a first-order stationary point of the original robust utility problem in~\eqref{eq:robust-utility} at rate $\mathcal{O}(\log K/K)$.
\end{theorem}

\begin{figure}[t]
    \centering

    \begin{subfigure}[b]{0.48\columnwidth}
        \centering
        \includegraphics[width=\linewidth]{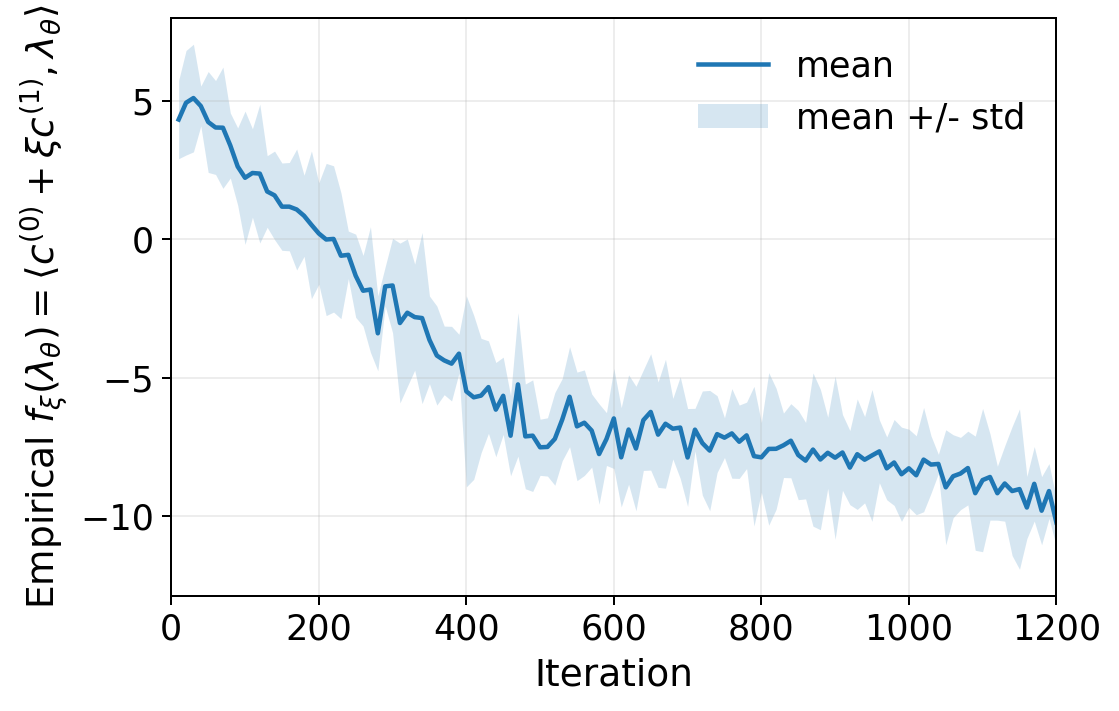}
        \caption{}
        \label{fig:alg1}
    \end{subfigure}
    \hfill
    \begin{subfigure}[b]{0.48\columnwidth}
        \centering
        \includegraphics[width=\linewidth]{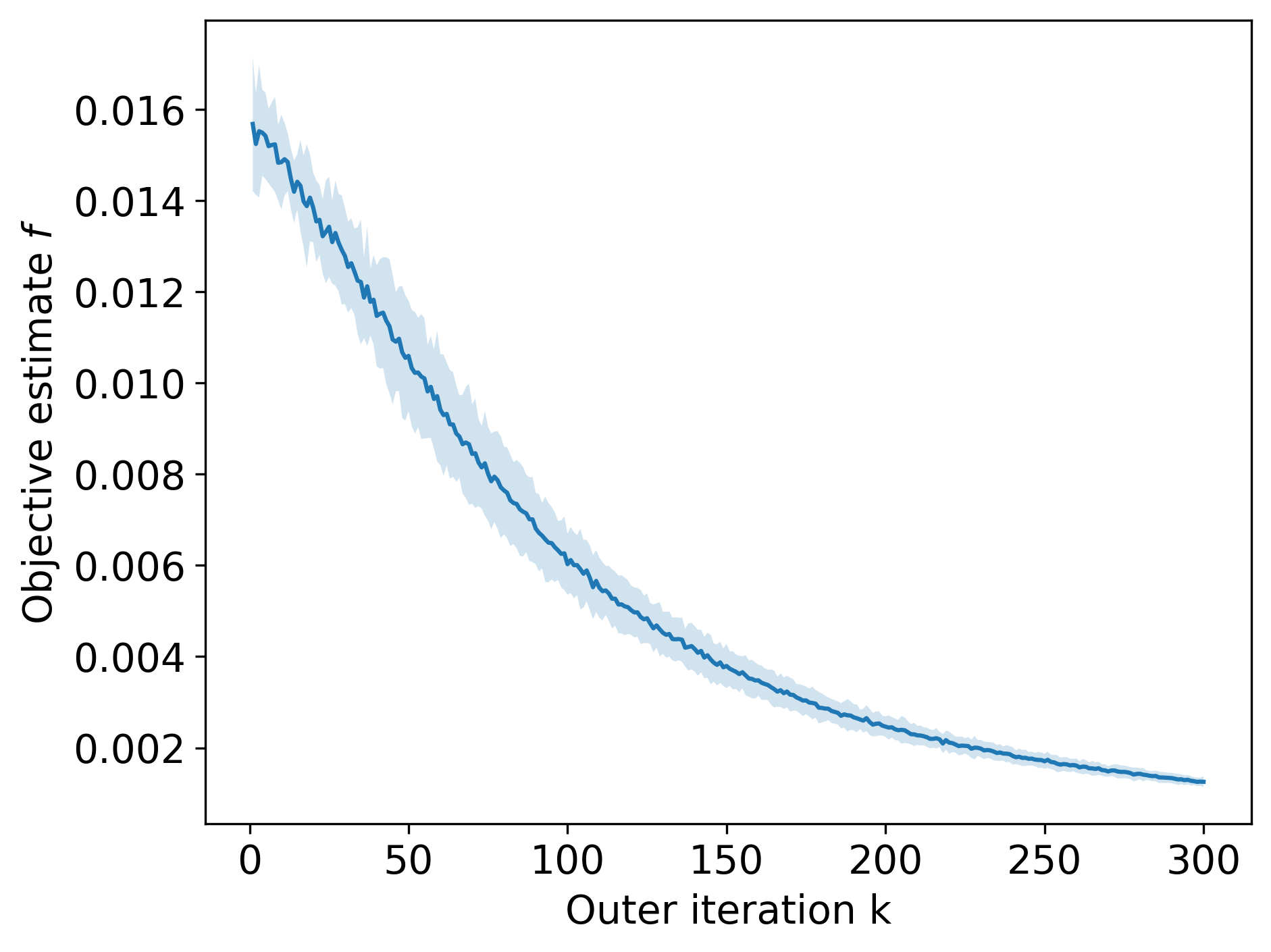}
        \caption{}
        \label{fig:objective}
    \end{subfigure}


    \begin{subfigure}[b]{0.48\columnwidth}
        \centering
        \includegraphics[width=\linewidth]{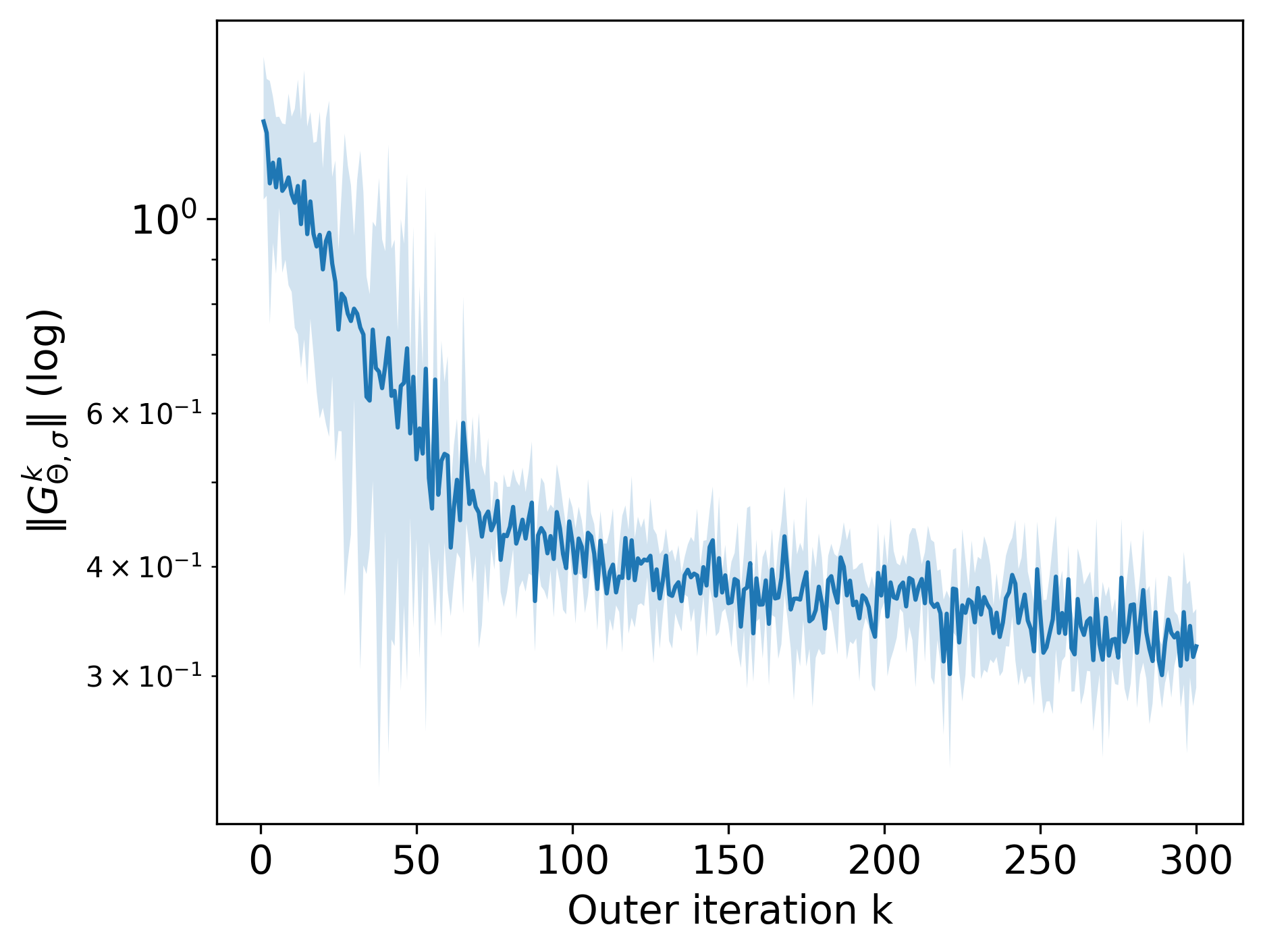}
        \caption{}
        \label{fig:grad}
    \end{subfigure}
    \hfill
    \begin{subfigure}[b]{0.48\columnwidth}
        \centering
        \includegraphics[width=\linewidth]{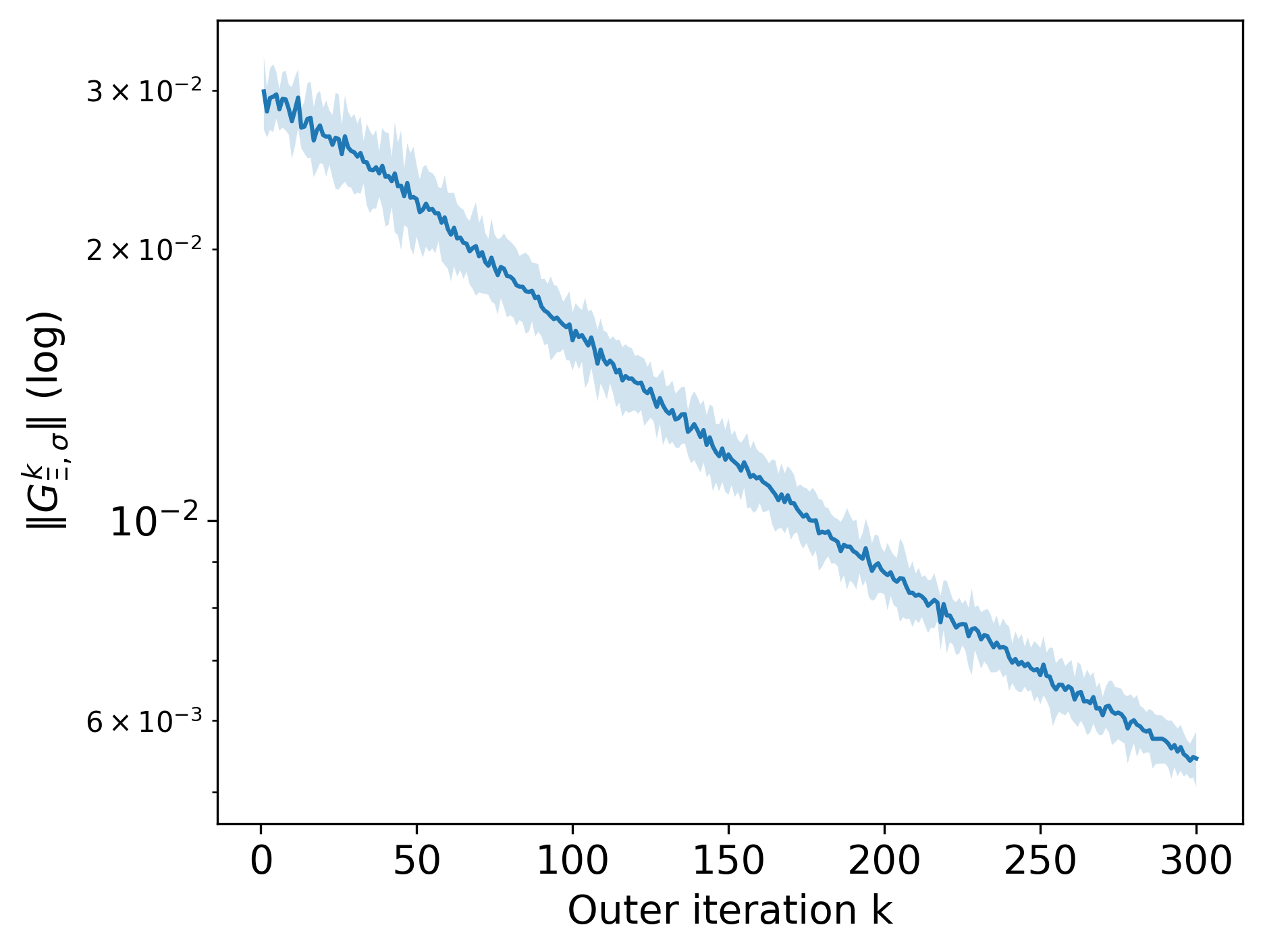}
        \caption{}
        \label{fig:G}
    \end{subfigure}

    \vspace{-1ex}

    \caption{\small{
    (a) Convergence of Algorithm~\ref{alg:pgda}.
    (b)--(d) Results of Algorithm~\ref{alg:pe-pgda}.
    }}
    \label{fig:convergence-results}
    \vspace{-2.5ex}
\end{figure}

\vspace{-2ex}
\section{Experiment}
In this section, we conduct experiments to demonstrate that Algorithms~\ref{alg:pgda} and~\ref{alg:pe-pgda} converge for the robust utility problem~\eqref{eq:robust-utility} in the concave and nonconcave settings, respectively. To evaluate Algorithm~\ref{alg:pgda}, 
unlike prior work that focuses on tabular settings~\citep{barakat2023reinforcement,chen2024robust}, we study an LLM safety alignment task as in Section~\ref{sec:examples} Example 3, where the goal is to maximize the helpfulness of the LLM, evaluated by \(c^{(0)}\), while satisfying specified cost constraints measured by \(c^{(1)}\). In particular, we fine-tune a Pythia-70m model~\citep{biderman2023pythia} on the PKU-SafeRLHF-10k dataset~\citep{dai2023safe}, a human-preference dataset designed to evaluate safety alignment in LLMs. To evaluate the convergence behavior of Algorithm~\ref{alg:pe-pgda} on the nonconcave utilities, we follow the synthetic tabular RL setting of~\citep{chen2024robust} but optimize the utility function defined in Section~\ref{sec:nonconcave-example}. Additional implementation details are in Appendix~\ref{app:implementation}.

\noindent\textbf{Results.} 
Figure~\ref{fig:alg1} reports the experimental results of Algorithm~\ref{alg:pgda} on the LLM safety alignment task. To evaluate its convergence behavior, we track the empirical utility objective \(f_\xi(\lambda_\theta)=\langle c^{(0)}+\xi c^{(1)},\lambda_\theta\rangle\) across training iterations. As shown in Figure~\ref{fig:alg1}, the utility objective decreases steadily during training and eventually stabilizes, demonstrating that Algorithm~\ref{alg:pgda} exhibits convergent behavior in a representative concave-utility setting. To empirically evaluate convergence of Algorithm~\ref{alg:pe-pgda}, we track three metrics: (1) the objective value at each iterate, $f_{\xi}(\lambda_{\theta})$, which directly reflects progress on the robust utility problem; (2) the norm of the proximal gradient mapping for the variable $\theta$ at iteration $k$, $\|\mathcal{G}^{k}_{\Theta,\sigma}(\theta,\xi)\|$, which tracks the progress toward a stationary solution for the outer maximization; (3) the norm of the proximal gradient mapping for the variable $\xi$ at iteration $k$, $\|\mathcal{G}^{k}_{\Xi,\sigma}(\theta,\xi)\|$, which tracks the progress toward a stationary solution for the inner maximization. As shown in Figures~\ref{fig:objective}-\ref{fig:G}, all three curves decrease and then stabilize, indicating that Algorithm~\ref{alg:pe-pgda} converges, consistent with Theorem~\ref{thm:nonconcave-convergence}. 

\section{Conclusion}
In this work, we introduced \textbf{robust general-utility RL}, a minimax framework for learning policies that remain reliable under misspecification of the utility functional. 
We developed provably convergent stochastic algorithms for both the widely used concave utility function and the more challenging nonconcave regime, and empirically verified their convergence behavior on LLM safety alignment and exploration-maximization tasks. 

\bibliography{aaai2027}

\clearpage
\appendix
\onecolumn

\section{Related Work}\label{sec:related}
\paragraph{RL with General Utility.}
RL with general utility has recently been introduced for moving beyond the standard cumulative-reward criterion and optimizing more general utility objectives~\citep{zhang2020variational,zhang2021convergence,kumar2022policy,barakat2023reinforcement,ying2023scalable,barakat2024towards}.~\citep{zhang2020variational} introduced the general-utility RL framework, deriving a policy gradient theorem for arbitrary concave utility functions of the occupancy measure and proposing a variational Monte Carlo algorithm that converges globally to the optimal policy. To improve optimization efficiency,~\citep{zhang2021convergence} further developed a variance-reduced policy gradient method with a gradient truncation mechanism to relax importance-weight assumptions, achieving improved sample complexity bounds.~\citep{kumar2022policy} further established a simplified policy gradient theorem and provided a sample-based algorithm. More recently,~\citep{barakat2023reinforcement} proposed a parameter-free, single-loop policy gradient algorithm with recursive momentum variance reduction that attains better sample complexities. Later,~\citep{ying2023scalable} extended the general-utility perspective to multi-agent RL.~\citep{barakat2024towards} moved beyond tabular settings by proposing a maximum-likelihood approach to estimating occupancy measures in larger state-action spaces. In contrast, our framework generalizes this line of work by allowing the utility functional itself to be misspecified and by explicitly targeting robustness to such utility uncertainty.

Prior work has also studied improving the robustness of general-utility policies to environmental changes~\citep{chen2024robust}, modeling uncertainty via changes to the transition dynamics. Specifically, they formulate a minimax problem with a fixed utility function and seek a policy that performs well under the worst-case environment. However, transition dynamics are not the only source of mismatch at deployment time: the utility function used for evaluation can itself be biased or shifted relative to the one specified during training. Our framework complements this line of work by targeting robustness to such utility misspecification.

\paragraph{Reward-Robust RL.}
Reward-robust RL aims to produce policies that remain reliable when the reward function is incorrectly specified~\citep{mannor2007bias}. This objective has been investigated in a variety of settings and applications~\citep{morimoto2005robust,wang2020reinforcement,he2023robotic,yan2024reward}. A widely used formulation introduces an uncertainty set over rewards and optimizes against the worst-case element, yielding a max-min objective that can be computationally intensive and thus challenging for large-scale problems~\citep{iyengar2005robust,wiesemann2013robust,tessler2019action}.
To reduce this computational burden, prior work has developed several lines of approaches. One direction exploits structured uncertainty sets that admit more efficient optimization procedures~\citep{mannor2016robust,mannor2012lightning,goyal2023robust}. Another common direction assumes rectangular uncertainty across states or state-action pairs, enabling decompositions and scalable solvers akin to dynamic programming~\citep{behzadian2021fast,ho2021partial,bagnell2001solving,grand2021scalable}. Most existing reward-robust RL methods, however, focus on the linear-utility case (expected cumulative reward), whereas our framework extends robustness to general utility functionals.

\section{Proofs}
\label{app:proofs}
\subsection{Proof of Proposition~\ref{prop:reward_robust_as_ru}}
\begin{proof}
By the construction of the proposition, for each $\xi\in\Xi=\mathcal{R}$, the utility
functional is linear in $\lambda$. Substituting this $f_\xi$ and $\Xi=\mathcal{R}$ into the robust-utility problem
(Equation~\ref{eq:robust-utility}) gives the bilinear minimax
\[
\min_{\theta\in\Theta}\;\max_{R\in\mathcal{R}}\; - \langle R,\lambda_\theta\rangle .
\label{eq:minimax-linear}
\]
Finally, observe that
\[
\min_{\theta\in\Theta}\;\max_{R\in\mathcal{R}}\; - \langle R,\lambda_\theta\rangle
\;=\;
\max_{\theta\in\Theta}\;\min_{R\in\mathcal{R}}\; \langle R,\lambda_\theta\rangle 
\;=\;
\max_{\theta \in \Theta} \min_{R \in \mathcal{R}} V_\theta,
\]
Comparing with
Equation~\ref{eq:reward-robust-rl}, we see that the reward-robust objective is precisely the special case of the robust-utility problem under the linear utility
class $f_\xi(\lambda)=\langle \xi,\lambda\rangle$, which proves the claim.
\end{proof}

\subsection{Proof of Proposition~\ref{prop:crl_as_ru}}
\begin{proof}
The Lagrangian of \eqref{eq:constrained_rl} is
\[
\mathcal{L}(\theta,\xi)
= V_{\theta}^{(0)} + \sum_{k=1}^K \xi_k\big(V_{\theta}^{(k)}-\tau_k\big)
= \Big\langle c^{(0)} + \sum_{k=1}^K \xi_k\,c^{(k)},\,\lambda_{\theta}\Big\rangle
\;-\;\sum_{k=1}^K \xi_k\,\tau_k
= f_{\xi}(\lambda_{\theta}),
\]
with multipliers $\xi\in\mathbb{R}_+^K$. The constrained RL problem in \eqref{eq:constrained_rl} can be transferred into:
\[
\min_{\theta\in\Theta}\; \max_{\xi\in\mathbb{R}_+^{K}}\; \mathcal{L}(\theta,\xi)
\]
By Slater’s condition, the optimal multipliers $\xi^\star$ are bounded, consequently, we can pick any compact, convex set $\Xi \in \mathbb{R}_+^{K}$ that contains $\xi^\star$ (the intersection of $\mathbb{R}_+^{K}$ with a large enough $l_1$-ball). Restricting the maximization to $\Xi$ does not change the optimal value and we further transformed into
\[
\min_{\theta\in\Theta}\; \max_{\xi\in\Xi}\; \mathcal{L}(\theta,\xi)
\;=\;
\min_{\theta\in\Theta}\; \max_{\xi\in\Xi}\;f_\xi(\lambda_\theta)
\]
This has exactly the robust-utility form with linear utilities $f_{\xi}$ indexed by $\xi\in\Xi$, proving the claim.
\end{proof}

\subsection{Algorithm~\ref{alg:pgda} with $\ell_p$-norm Uncertainty Set}
\label{app:pgda-lp}

Algorithm~\ref{alg:pgda} is a general projected descent-ascent method for arbitrary ambiguity sets $\Xi$. Note that for certain ambiguity sets with additional structure as commonly used in the literature, the inner maximization problem $\max_{\xi \in \Xi} f_\xi(\lambda_\theta)$ admits a (near) closed-form solution. This can be used in place of projected ascent to obtain a faster inner maximization oracle. To make this concrete, we consider a simple and widely used $\ell_p$-norm ambiguity set: 
\[
   \Xi = \{\xi: \|\xi-\tilde \xi\|_p \le d_0\}
   \]
where $\tilde \xi$ is a normal parameter (estimate from data), and the true utility parameter is believed to lie within radius $d_0$ of $\tilde \xi$ under a choosing norm. Consider only the inner updates on $\xi_t$, let $g_t^{(\xi)}=\nabla_\xi f_{\xi_t}(\lambda)$, a standard update in Algorithm~\ref{alg:pgda} is:
\[
   \xi_{t+1}=\text{proj}_{\|\xi-\tilde \xi\|_p \le d_0} (\xi_t+\beta g_t^{(\xi)})
   \]
Alternatively, one may approximately solve the inner maximization $\max_{\|\xi-\tilde \xi\|_p \le d_0} f_\xi(\lambda)$ by maximizing the
linearization of $f_\xi(\lambda)$ at $\tilde \xi$, which admits a closed-form optimizer over an
$\ell_p$-ball~\citep{bertsimas2004robust,nemirovski2012lectures}. We include the derivation for completeness.
\begin{lemma}
\label{lem:pgda-lp}
For $\ell_p$-norm ambiguity sets with $p>1$, the inner maximization $\max_{\|\xi-\tilde \xi\|_p \le d_0} f_\xi(\lambda)$ can be approximated by maximizing the first-order (linear) approximation of $f_\xi(\lambda)$ at $\tilde \xi$. The resulting optimizer $\hat{\xi}$ admits the following closed-form characterization via dual norms:
\begin{equation}
\label{eq:lp-solution}
\hat{\xi} = \tilde \xi + d_0 s(g), \quad
s_i(g)=\frac{\mathrm{sign}(g_i)\,|g_i|^{q-1}}{\|g\|_q^{q-1}},
\end{equation}
where $g=\nabla_\xi f_{\tilde \xi}(\lambda) \in \mathbb{R}^{d_\Xi}$ is a (sub)gradient vector and the dual exponent is $q=\frac{p}{p-1}$.
\end{lemma}

\begin{proof}
Around the nominal parameter $\tilde \xi$, we have:
\[
      f_\xi (\lambda) \approx f_{\tilde \xi} (\lambda) + \nabla_\xi f_{\tilde \xi}(\lambda)^\top(\xi-\tilde \xi)
\]
Let $u:=\xi-\tilde \xi$. Then, the inner step of maximizing $f_\xi(\lambda)$ within the norm ball approximately becomes:
\[
   \max_{\|u\|_p \le d_0} g^\top u
   \]
Because the objective is linear and the feasible set is convex compact, an optimum exists and lies on the boundary when $g \neq 0$.
Form the Lagrangian:
\[
\mathcal{L}(u,\alpha)
=\sum_{i=1}^{d_\Xi} g_i u_i
-\alpha\Bigl(\sum_{i=1}^{d_\Xi}|u_i|^{p}-d_0^{p}\Bigr),
\qquad \alpha \ge 0.
\]
KKT stationarity for each coordinate $i$ (for $p>1$, derivative exists for $u_i\neq 0$): 
\[
\frac{\partial}{\partial u_i}\mathcal{L}
= g_i - \alpha p\,|u_i|^{p-1}\,\mathrm{sign}(u_i) = 0.
\]
i.e.,
\[
g_i=\alpha p\,|u_i|^{p-1}\,\mathrm{sign}(u_i).
\]
Since $\alpha\ge 0$, $p>1$, and $|u_i|^{p-1}\ge 0$, it follows that for $g_i\neq 0$ we must have
$\mathrm{sign}(u_i)=\mathrm{sign}(g_i)$ (and if $g_i=0$, one may take $u_i=0$). Therefore,
\[
  |g_i| = \alpha p\,|u_i|^{p-1}
  \quad \Longrightarrow \quad
  |u_i| = \left(\frac{|g_i|}{\alpha p}\right)^{\frac{1}{p-1}}.
\]
Therefore, we have:
\[
u_i \;=\; \mathrm{sign}(g_i)\left(\frac{|g_i|}{\alpha p}\right)^{\frac{1}{p-1}} .
\]
Now enforce the constraint \(
\sum_{i=1}^{d}|u_i|^{p}=d_0^{p}
\), we have
\[
\sum_{i=1}^{d}\left(\frac{|g_i|}{\alpha p}\right)^{\frac{p}{p-1}}=d_0^{p}.
\]
Define the dual exponent $q$ by
\[
q := \frac{p}{p-1}
\qquad\Longleftrightarrow\qquad
\frac{1}{p}+\frac{1}{q}=1.
\]
Then $\frac{p}{p-1}=q$. So:
\[
\sum_{i=1}^{d}\left(\frac{|g_i|}{\alpha p}\right)^{q}=d_0^{p}
\;\Longrightarrow\;
\left(\frac{1}{\alpha p}\right)^{q}\sum_{i=1}^{d}|g_i|^{q}=d_0^{p}.
\]
Thus
\[
\frac{1}{\alpha p}
=\left(\frac{d_0^{p}}{\sum_{i}|g_i|^{q}}\right)^{\frac{1}{q}}.
\]
And finally:
\[
u_i
=\mathrm{sign}(g_i)\,|g_i|^{\frac{1}{p-1}}
\left(\frac{d_0^{p}}{\sum_{j}|g_j|^{q}}\right)^{\frac{1}{q(p-1)}} .
\]
Now simplify exponents:
\begin{itemize}
\item $\frac{1}{p-1}=q-1$ (since $q=\frac{p}{p-1}\Rightarrow q-1=\frac{1}{p-1}$).
\item $\frac{1}{q(p-1)}=\frac{1}{p}$ (because $q(p-1)=p$).
\end{itemize}
So:
\[
u_i
=\mathrm{sign}(g_i)\,|g_i|^{q-1}
\left(\frac{d_0^{p}}{\sum_{j}|g_j|^{q}}\right)^{\frac{1}{p}} .
\]
Finally, note that
\[
\left(\sum_{j}|g_j|^{q}\right)^{1/q}=\|g\|_{q},
\qquad\text{and}\qquad
\left(\frac{d_0^{p}}{\sum_{j}|g_j|^{q}}\right)^{1/p}
=d_0\left(\sum_{j}|g_j|^{q}\right)^{-1/p}
=d_0\,\|g\|_{q}^{-q/p}.
\]
But $\frac{q}{p}=\frac{p/(p-1)}{p}=\frac{1}{p-1}=q-1$. Hence
$\|g\|_{q}^{-q/p}=\|g\|_{q}^{-(q-1)}$.
So the final closed form is:
\begin{equation}
u_i^\star
=d_0\,
\frac{\mathrm{sign}(g_i)\,|g_i|^{q-1}}{\|g\|_{q}^{\,q-1}},
\qquad (g\neq 0).
\tag{6}
\end{equation}
Equivalently, define the vector $s(g)$ by
\[
s_i(g):=\frac{\mathrm{sign}(g_i)\,|g_i|^{q-1}}{\|g\|_{q}^{\,q-1}},
\]
and
\[
   u^\star = d_0 s(g), \qquad \hat{\xi} = \tilde \xi + d_0 s(g).
   \]
This completes the proof.
\end{proof}

\begin{algorithm}[t]
\caption{PGDA on $\ell_p$ Ambiguity Set}
\label{alg:pgda-lp}
\begin{algorithmic}[1]
\State \textbf{Hyperparameters:} Iteration numbers $K$, stepsizes $\eta$, Monte-Carlo budgets $m, H, m', H'$, uncertainty set budget $d_0$, nominal $\tilde \xi$.
\State \textbf{Initialize:} $\theta_0,\in \Theta$.
\For{iterations $k = 0,1,\ldots,K-1$}
    \State Obtain $g^{(\xi)}_{k} \approx \nabla_\xi f_{\tilde \xi}(\lambda_{\theta_{k}})$ by Algorithm~\ref{alg:stochastic_gradient} with hyperparameters $m, H, m', H'$.
    \State Obtain $\xi_{k+1}$ using Equation~\ref{eq:lp-solution} with $g^{(\xi)}_{k}$
    \State Obtain $g^{(\theta)}_{k} \approx \nabla_\theta f_{\xi_{k+1}}(\lambda_{\theta_{k}})$ by Algorithm~\ref{alg:stochastic_gradient} with hyperparameters $m, H, m', H'$
    
    \State Apply the projected stochastic gradient descent step: \(
        \theta_{k+1}
            = \operatorname{proj}_\Theta\,\!\bigl(\theta_{k} - \eta\,g^{(\theta)}_{k}\bigr).
    \)
\EndFor
\State \textbf{Output:} $(\theta_K, \xi_K)$
\end{algorithmic}
\end{algorithm}

Given Lemma~\ref{lem:pgda-lp}, we can replace lines 4-9 of Algorithm~\ref{alg:pgda} with the near closed-form update in~\eqref{eq:lp-solution}, yielding Algorithm~\ref{alg:pgda-lp}. This provides a projection-free inner update that can be used as an efficient inner maximization oracle. Next, we bound the suboptimality gap between the solution in~\eqref{eq:lp-solution} and the exact optimum.

\begin{proposition}
\label{prop:lp_gap}
Fix an outer iterate and consider the $\ell_p$-norm uncertainty set. Assume Assumptions~\ref{ass:f-convex-concave} and~\ref{ass:xi_bound} hold. Let $\xi^\star \in \arg\max_{\|\xi-\tilde{\xi}\|_p\le d_0} f_\xi(\lambda)$ be an optimal inner solution, and let $\hat{\xi}$ be the output of Lemma~\ref{lem:pgda-lp}. 
Then the suboptimality gap satisfies
\[
 f_{\xi^\star}(\lambda) - f_{\hat{\xi}}(\lambda) 
\ \le\ \frac{L_\xi}{2}\,c_p^2\,d_0^2,
\]
where $c_p:=\sup_{\|\hat{\xi}-\tilde{\xi}\|_p\le 1}\|\hat{\xi}-\tilde{\xi}\|_2 = d^{\max\{0,\frac12-\frac1p\}}$.
\end{proposition}

\begin{proof}
Let $g := \nabla_\xi f_{\tilde{\xi}}(\lambda)$. By Assumption~\ref{ass:f-convex-concave} (concavity of $f_\xi(\lambda)$ in $\xi$), for any $\xi$ we have 
\[
f_\xi(\lambda)\ \le\ f_{\tilde{\xi}}(\lambda) + \langle g,\ \xi-\tilde{\xi}\rangle .
\]
Applying at $\xi=\xi^\star$ yields
\[
f_{\xi^\star}(\lambda)\ \le\ f_{\tilde{\xi}}(\lambda) + \langle g,\ \xi^\star-\tilde{\xi}\rangle.
\]
Next, by the definition of $\hat{\xi}$ in Lemma~\ref{lem:pgda-lp}, $\hat{\xi}$ maximizes the linearized objective over the $\ell_p$-ball, i.e.,
\[
\hat{\xi}\in \arg\max_{\|\xi-\tilde{\xi}\|_p\le d_0}\ \langle g,\ \xi-\tilde{\xi}\rangle,
\]
and therefore
\[
\langle g,\ \xi^\star-\tilde{\xi}\rangle \ \le\ \langle g,\ \hat{\xi}-\tilde{\xi}\rangle .
\]
Combinin above gives
\begin{equation}
\label{eq:opt_upper2}
f_{\xi^\star}(\lambda)\ \le\ f_{\tilde{\xi}}(\lambda) + \langle g,\ \hat{\xi}-\tilde{\xi}\rangle .
\end{equation}
On the other hand, by Assumption~\ref{ass:f-convex-concave} (smoothness of $f_\xi(\lambda)$ in $\xi$ with constant $L_\xi$), we have for any $\xi$:
\[
f_\xi(\lambda)\ \ge\ f_{\tilde{\xi}}(\lambda) + \langle g,\ \xi-\tilde{\xi}\rangle - \frac{L_\xi}{2}\|\xi-\tilde{\xi}\|_2^2 .
\]
Applying at $\xi=\hat{\xi}$ yields
\begin{equation}
\label{eq:hat_lower}
f_{\hat{\xi}}(\lambda)\ \ge\ f_{\tilde{\xi}}(\lambda) + \langle g,\ \hat{\xi}-\tilde{\xi}\rangle - \frac{L_\xi}{2}\|\hat{\xi}-\tilde{\xi}\|_2^2 .
\end{equation}
Subtracting~\eqref{eq:hat_lower} from~\eqref{eq:opt_upper2} gives
\[
f_{\xi^\star}(\lambda) - f_{\hat{\xi}}(\lambda)
\ \le\ \frac{L_\xi}{2}\|\hat{\xi}-\tilde{\xi}\|_2^2 .
\]
Finally, since $\hat{\xi}$ satisfies $\|\hat{\xi}-\tilde{\xi}\|_p\le d_0$, Assumption~\ref{ass:xi_bound} implies 
\[
\|\hat{\xi}-\tilde{\xi}\|_2 \ \le\ c_p\,\|\hat{\xi}-\tilde{\xi}\|_p \ \le\ c_p\,d_0,
\]
and therefore
\[
f_{\xi^\star}(\lambda) - f_{\hat{\xi}}(\lambda)
\ \le\ \frac{L_\xi}{2}\,c_p^2\,d_0^2,
\]
which completes the proof.
\end{proof}

\begin{lemma}
\label{lem:lp_xi_gap}
Fix an outer iterate and consider the $\ell_p$-norm uncertainty set. Assume Assumptions~\ref{ass:f-convex-concave} and~\ref{ass:xi_bound} hold. Let $\xi^\star \in \arg\max_{\|\xi-\tilde{\xi}\|_p\le d_0} f_\xi(\lambda)$ be an optimal inner solution, and let $\hat{\xi}$ be the output of Lemma~\ref{lem:pgda-lp}. Then 
\[
\|\hat\xi-\xi^\star\|^2
\le
L^2_{\theta,\xi}\frac{L_\xi}{\mu_\xi}c_p^2d_0^2
\]
\end{lemma}
\begin{proof}
Since $f_\xi(\lambda)$ is $\mu_\xi$-strongly concave in $\xi$ by Assumption~\ref{ass:f-convex-concave}, for any $\xi\in\Xi$,
\[
f_{\xi^\star}(\lambda)-f_\xi(\lambda)\ge \frac{\mu_\xi}{2}\|\xi-\xi^\star\|^2.
\]
Since Proposition~\ref{prop:lp_gap} gives
\[
f_{\xi^\star}(\lambda)-f_{\hat\xi}(\lambda)\le \frac{L_\xi}{2}\,c_p^2\,d_0^2,
\]
then
\[
\|\hat\xi-\xi^\star\|^2
\le \frac{2}{\mu_\xi}\bigl(f_{\xi^\star}(\lambda)-f_{\hat\xi}(\lambda)\bigr)
\le \frac{L_\xi}{\mu_\xi}\,c_p^2\,d_0^2.
\]
\end{proof}
Next, we establish the convergence bound for Algorithm~\ref{alg:pgda-lp}, where the third term accounting for inner-maximization inexactness is replaced using Proposition~\ref{prop:lp_gap}. 
\begin{theorem}
\label{theorem:pgda_lp_converge}
Consider Algorithm~\ref{alg:pgda-lp} for the RL with robust utility problem in~\eqref{eq:robust-utility} on $\ell_p$-norm uncertainty set, suppose Assumption~\ref{ass:xi-convex-compact} to Assumption~\ref{ass:f-convex-concave} hold. Let the step size $\eta \le 1/L_\Gamma$ and $E_\theta$ be the error bounds of stochastic gradients estimation of $\theta$ as in Proposition~\ref{prop:error_grad_estimate}. Let $\Gamma(\theta)$ be the envelop function and $\Gamma^\star := \inf_{\theta\in\Theta}\Gamma(\theta)$ be the optimal envelop value. Define the outer projected gradient mapping for the envelope: \(
G_{\Gamma}(\theta)\;\triangleq\;\frac{1}{\eta}\left(\theta-\operatorname{proj}_{\Theta}\!\bigl(\theta-\eta\nabla\Gamma(\theta)\bigr)\right).
\) Then for any $K\ge 1$,
\[
\frac{1}{K}\sum_{k=0}^{K-1}\mathbb{E}\bigl\|G_{\Gamma}(\theta_k)\bigr\|^2
\le
\frac{8(\Gamma(\theta_0)-\Gamma^\star)}{\eta K}
+ 20 E_\theta
+ 20 L^2_{\theta,\xi}\frac{1}{\mu^2_\xi} L^2_{\theta,\xi}\frac{L_\xi}{\mu_\xi}c_p^2d_0^2.
\] 

\end{theorem}

\begin{proof}
It follows the proof of Theorem~\ref{theorem:pgda_converge}, except that the inexactness of the inner estimation is controlled by:
\[
\|b_k\|
=\bigl\|\nabla_\theta f_{\hat\xi_k}(\lambda_{\theta_k})
-\nabla_\theta f_{\xi^\star(\theta_k)}(\lambda_{\theta_k})\bigr\|
\le L_{\theta,\xi}\,\|\hat\xi_k-\xi^\star(\theta_k)\|.
\]
Therefore, by Lemma~\ref{lem:lp_xi_gap}
\[
\mathbb{E}\|b_k\|^2
\le L_{\theta,\xi}^2\,\|\hat\xi_k-\xi^\star(\theta_k)\|^2
\le L_{\theta,\xi}^2\,\frac{L_\xi}{\mu_\xi}\,c_p^2\,d_0^2.
\]
\end{proof}

\subsection{Proof of Theorem~\ref{theorem:gradient}}
\begin{proof}
For the gradient w.r.t $\theta$, we follow Equation 5, 6 from~\citep{barakat2023reinforcement}. Specifically, let $R\in \mathbb{R}^{\mathcal{S}\times \mathcal{A}}$ be the general reward function in the standard RL problem. Define the value function $V_{\theta}(R)=\langle \lambda_{\theta},\, R\rangle$, it follows from the policy gradient theorem that:
\[
[\nabla_\theta \lambda_\theta]^{\top} R
\;=\;
\nabla_\theta V_{\theta}(R)
\;=\;
\mathbb{E}_{\pi_\theta,\,p}\!\left[
\sum_{t=0}^{+\infty} \gamma^{t}\, R(s_t,a_t)
\left(\sum_{h=0}^{t} \nabla_\theta \log \pi_\theta(a_{h}\mid s_{h})\right)
\right],
\]
where $\nabla_\theta \lambda_\theta$ is the Jacobian matrix of the vector mapping $\lambda_\theta$.
Using the chain rule, we have
\[
\nabla_\theta f_\xi\big(\lambda_\theta\big)
\;=\;
\big[\nabla_\theta \lambda_\theta\big]^{\top}\, \nabla_\lambda f_\xi\big(\lambda_\theta\big)
\;=\;
\nabla_\theta V_{\theta}(R)\Big|_{\,R=\nabla_\lambda f_\xi(\lambda_\theta)}.
\]
which completes the proof.
\end{proof}

\subsection{Stochastic Gradients Estimation}
\label{app:grad_estimate}

To get the stochastic estimation of the gradients $\nabla_{\theta} f_{\xi}(\lambda_{\theta})$ and $\nabla_{\xi} f_{\xi}(\lambda_{\theta})$, we first estimate the occupancy measure following previous work~\citep{barakat2023reinforcement,chen2024robust}. Given \(\tau := \{\tau_i\}_{i=1}^{m}\) contains \(m\) independent trajectories $\tau_i := \{s_{i,h}, a_{i,h}\}_{h=0}^{H-1}$ of length \(H\) generated from the policy \(\pi_\theta\), we estimate the occupancy measure by:
\begin{equation}
\label{eq:occ_estimate}
\widehat{\lambda}\bigl(\tau; s,a\bigr)
:= \frac{1-\gamma}{m}
   \sum_{i=1}^{m} \sum_{h=0}^{H-1}
   \gamma^{h}\,\mathbf{1}\bigl\{ s_{i,h} = s,\,
                                a_{i,h} = a \bigr\},
\end{equation}
where \(\mathbf{1}\{\cdot\}\) is an indicator function. To compute policy gradients, we require the derivative of the utility functional with respect to the occupancy measure, i.e.,
\(
\nabla_\lambda f_\xi(\lambda)\in\mathbb{R}^{|\mathcal{S}||\mathcal{A}|},
\)
whose $(s,a)$-th coordinate is
\[
\bigl[\nabla_\lambda f_\xi(\lambda)\bigr](s,a)\;=\;\frac{\partial f_\xi(\lambda)}{\partial \lambda(s,a)}.
\]
In our framework, $f_\xi$ is assumed to be known and differentiable as a function of $\lambda$ for any fixed $\xi$ (Assumption~\ref{ass:lambda_bound}), so $\nabla_\lambda f_\xi(\lambda)$ can be evaluated once a value of $\lambda$ is provided. Accordingly, given the empirical occupancy estimate $\widehat{\lambda}(\tau)$ in~\eqref{eq:occ_estimate}, we estimate the occupancy gradient by
\[
\widehat{c}(s,a)\;=\;\left[\nabla_\lambda f_\xi(\lambda)\right]_{\lambda=\widehat{\lambda}(\tau)}(s,a),
\qquad \forall (s,a)\in\mathcal{S}\times\mathcal{A}.
\]
Then the stochastic gradients of $\nabla_{\lambda} f_{\xi}(\lambda_{\theta})$ and $\nabla_{\xi} f_{\xi}(\lambda_{\theta})$ can be approximated respectively by the following
stochastic sample averaged values known as GPOMDP~\citep{yuan2022general}:
\begin{equation}
\label{eq:theta_grad_estimate}
g^{(\theta)}\bigl(\tau',\theta,\xi,\hat{c}\bigr)= \frac{1}{m'}
   \sum_{i=1}^{m'}
   \biggl[
      \sum_{t=0}^{H'-1}
      \gamma^{t}\,\hat{c}\,\!\bigl(s'_{i,t}, a'_{i,t}\bigr)
      \sum_{h=0}^{t} \nabla_\theta
      \log \pi_\theta\!\bigl(a'_{i,h}\mid s'_{i,h}\bigr)
   \biggr],  
\end{equation}
\begin{equation}
\label{eq:xi_grad_estimate}
g^{(\xi)}\bigl(\tau,\theta,\xi\bigr)= \nabla_\xi f_\xi\bigl[\widehat{\lambda}(\tau)\bigr].
\end{equation}
where \(\tau' := \{\tau_i'\}_{i=1}^{m'}\) contain \(m'\) independent trajectories $\tau'_i := \{s'_{i,h}, a'_{i,h}\}_{h=0}^{H'-1}$ generated from the policy \(\pi_\theta\). Notice that for estimating $g^{(\xi)}\bigl(\tau,\theta,\xi\bigr)$, we do not need additional independent samples and reuse \(\tau \) because $\xi$ does not affect the sampling distribution. We summarize the whole procedure of obtaining the stochastic gradients in
Algorithm~\ref{alg:stochastic_gradient}. Next, we establish the error bounds of approximating the true gradient with the stochastic gradients.

\begin{algorithm}[t]
\caption{Obtain Stochastic Gradients at $(\theta,\xi)$}
\label{alg:stochastic_gradient}
\begin{algorithmic}[1]
\State \textbf{Input:} $\theta \in \Theta,\xi \in \Xi$.
\State \textbf{Hyperparameters:} Monte-Carlo budgets $m, H, m', H'$.
\State Generate independent trajectories
       $\tau_i := \{s_{i,h}, a_{i,h}\}_{h=0}^{H-1}$
       $(i = 1,\ldots,m)$ from $\pi_\theta$.
\State Obtain $\widehat{\lambda}(\tau; s,a)$ for every
       $(s,a) \in \mathcal{S} \times \mathcal{A}$ by~\eqref{eq:occ_estimate} with
       $\tau := \{\tau_i\}_{i=1}^{m}$.
\State Obtain $\widehat{c} := \nabla_\lambda f_\xi[\widehat{\lambda}(\tau)]$.
\State Generate independent trajectories
       $\tau'_i := \{s'_{i,h}, a'_{i,h}\}_{h=0}^{H'-1}$
       $(i = 1,\ldots,m')$ from $\pi_\theta$.
\State Obtain $g^{(\theta)}(\tau', \theta, \xi, \widehat{c})$
       by~\eqref{eq:theta_grad_estimate} with $\tau' := \{\tau_i'\}_{i=1}^{m'}$.
\State Obtain $g^{(\xi)}(\tau, \theta, \xi)$
       by~\eqref{eq:xi_grad_estimate} with $\tau := \{\tau_i\}_{i=1}^{m}$.
\State \textbf{Output:}
       $g^{(\theta)}(\tau', \theta, \xi, \widehat{c})
       \approx \nabla_\theta f_\xi(\lambda_\theta)$,
       $g^{(\xi)}(\tau, \theta, \xi)
       \approx \nabla_\xi f_\xi(\lambda_\theta)$,
\end{algorithmic}
\end{algorithm}

\begin{proposition}
\label{prop:error_grad_estimate}
Under Assumptions~\ref{ass:pi_theta_bound} to~\ref{ass:xi_lambda_bound}, the stochastic gradients in Equation~\ref{eq:theta_grad_estimate} and~\ref{eq:xi_grad_estimate} have
the following error bounds:
\[
   \mathbb{E}_{\pi_\theta}
\bigl\|
g^{(\theta)}(\tau',\theta,\xi,\hat{c})
- \nabla_\theta f_\xi(\lambda_\theta)
\bigr\|^{2}\le E_\theta, \qquad  \mathbb{E}_{\pi_\theta}
\bigl\|
g^{(\xi)}(\tau,\theta,\xi,\hat{c})
- \nabla_\xi f(\lambda_\theta,\xi)
\bigr\|^{2}
\le E_\xi,
   \]
where \[
   E_\theta = \frac{3\ell_{\pi_\theta}^{2}}{(1-\gamma)^{4}}
\Bigl[
L_\lambda^{2}|\mathcal{S}||\mathcal{A}|
\Bigl(\frac{1}{m} + \gamma^{2H}\Bigr)
+ \frac{\ell_\lambda^{2}}{m'}
+ \ell_\lambda^{2}\bigl(1 + H_\theta(1-\gamma)\bigr)^{2}\gamma^{2H'}
\Bigr], \quad E_\xi = L_{\xi,\lambda}^{2}
\Bigl(\frac{1}{m} + \gamma^{2H}\Bigr).
\]
\end{proposition}

\begin{proof}
The estimated occupancy measure in Equation~\ref{eq:occ_estimate} is an unbiased estimator of the following
truncated occupancy measure with truncation level $H$:
\[
\lambda_{\theta}^{(H)}(s,a)
\;\overset{\mathrm{def}}{=}\;
(1-\gamma)\sum_{t=0}^{H-1}
\gamma^{t}\,\mathbb{P}_{\pi_\theta}\bigl(s_t = s,\, a_t = a \mid s_0 \sim \rho\bigr).
\]
Denote $\widehat{\lambda}(\tau)
:= \bigl[\widehat{\lambda}(\tau;s,a)\bigr]_{s,a\in\mathcal{S}\times\mathcal{A}}
\in \mathbb{R}^{|\mathcal{S}||\mathcal{A}|},
\lambda_{\theta}^{(H)}
:= \bigl[\lambda_{\theta}^{(H)}(s,a)\bigr]_{s,a\in\mathcal{S}\times\mathcal{A}}
\in \mathbb{R}^{|\mathcal{S}||\mathcal{A}|}, \lambda_{\theta}
:= \bigl[\lambda_{\theta}(s,a)\bigr]_{s,a\in\mathcal{S}\times\mathcal{A}}
\in \mathbb{R}^{|\mathcal{S}||\mathcal{A}|}.$
Then the estimation error of the occupancy measure has the following upper bound:
\begin{align*}
\mathbb{E}_{\pi_\theta}
\bigl\|
\widehat{\lambda}(\tau) - \lambda_{\theta}
\bigr\|^{2}
&\overset{(i)}{=}
\operatorname{Var}_{\pi_\theta}
\bigl[\widehat{\lambda}(\tau)\bigr]
+
\bigl\|\mathbb{E}_{\pi_\theta}
\big[\widehat{\lambda}(\tau) - \lambda_{\theta} \big]
\bigr\|^{2}
\\
&\overset{(ii)}{=}
\operatorname{Var}_{\pi_\theta}
\bigl[\widehat{\lambda}(\tau)\bigr]
+
\bigl\|\mathbb{E}_{\pi_\theta}
\big[\lambda_{\theta}^{(H)} - \lambda_{\theta} \big]
\bigr\|^{2}
\\
&\overset{(iii)}{\le}
\frac{1}{m}\,
\operatorname{Var}
\bigl[\widehat{\lambda}_1(\tau_1)\bigr]
+
\sum_{s,a}
\left|
(1-\gamma)\sum_{t=H}^{+\infty}
\gamma^{t}\,
\mathbb{P}_{\pi_\theta}
\bigl(s_t=s,a_t=a \mid s_0\sim\rho\bigr)
\right|^{2}
\\
&\overset{(iv)}{\le}
\frac{1}{m}\,
\mathbb{E}
\bigl\|
\widehat{\lambda}_1(\tau_1)
\bigr\|^{2}
+
\Bigl[(1-\gamma)\sum_{t=H}^{+\infty}\gamma^{t}\Bigr]
\sum_{s,a}
\Bigl[(1-\gamma)\sum_{t=H}^{+\infty}
\gamma^{t}\,
\mathbb{P}_{\pi_\theta}
\bigl(s_t=s,a_t=a \mid s_0\sim\rho\bigr)\Bigr]
\\
&\overset{(v)}{\le}
\frac{1}{m} + \gamma^{2H}.
\end{align*}
where (i) uses
\(\mathbb{E}\|X\|^{2} = \operatorname{Var}X + \|\mathbb{E}X\|^{2}\)
for random vector
\(X := \widehat{\lambda}(\tau) - \lambda_{\theta}\);
(ii) uses the fact that $\widehat{\lambda}(\tau)$ is an unbiased estimator of $\lambda_{\theta}^{(H)}$; (iii) uses the fact that
\(\widehat{\lambda}\) is the average among the
\(m\) i.i.d.\ individual estimators
$\widehat{\lambda}_i(\tau_i; s,a)
:= (1-\gamma)\sum_{h=0}^{H-1}
\gamma^{h}\,\mathbf{1}\{s_{i,h} = s,\,
a_{i,h} = a\}, i = 1,\ldots,m$ and the definition of $\lambda_\theta^{H},\lambda_\theta$;
(iv) uses
\(\operatorname{Var}X \le \mathbb{E}\|X\|^{2}\) for random vector
\(X := \widehat{\lambda}_1(\tau_1)\) and
\(\mathbb{P}_{\pi_\theta}(s_t = s, a_t = a \mid s_0 \sim \rho) \in [0,1]\);
and (v) uses
\(0 \le \|\widehat{\lambda}_1(\tau_1)\|^{2}
\le \sum_{s,a}\widehat{\lambda}_1(\tau_1; s,a) = 1\)
and
\(\sum_{s,a}\mathbb{P}_{\pi_\theta}(s_t = s, a_t = a \mid s_0 \sim \rho) = 1\).

Then $g^{(\xi)}(\tau',\theta,\xi,\hat{c})$ can be bounded as follows: 
\begin{align*}
\mathbb{E}\bigl\|g^{(\xi)}(\tau',\theta,\xi,\hat{c}) - \nabla_{\xi} f_{\xi}(\lambda_{\theta})\bigr\|^{2}
&= \mathbb{E}\bigl\|\nabla_{\xi} f_{\xi}(\widehat{\lambda}(\tau))
- \nabla_{\xi} f_{\xi}(\lambda_{\theta})\bigr\|^{2}\\
&\overset{(i)}{\le} L_{\xi,\lambda}^{2}\,
\mathbb{E}\bigl\|\widehat{\lambda}(\tau) - \lambda_{\theta}\bigr\|^{2}\\
&\overset{(ii)}{\le} L_{\xi,\lambda}^{2}
\left(\frac{1}{m} + \gamma^{2H}\right).
\end{align*}
where (i) uses Assumption~\ref{ass:xi_lambda_bound} and (ii) uses the results derived above. 

As for the error bounds of $g^{(\theta)}(\tau',\theta,\xi,\hat{c})$, it follows Propostion 9 of~\citep{chen2024robust} where we fix the dynamic. 
\end{proof}

\subsection{Proof of Proposition~\ref{prop:theta_bound}}
\begin{proof}
Recall that the gradient w.r.t $\theta$ is:
\[
    \nabla_{\theta} f_{\xi}(\lambda_{\theta})
\;=\;
\mathbb{E}_{\pi_{\theta},\,p}\!\left[
\sum_{t=0}^{\infty} \gamma^{t}\,
\frac{\partial f_{\xi}(\lambda_{\theta})}{\partial \lambda_{\theta}(s_{t},a_{t})}
\left(\sum_{h=0}^{t} \nabla_{\theta}\log \pi_{\theta}(a_{h}\mid s_{h})\right)
\ \Bigg|\ s_{0}\sim \rho
\right].
\]
Taking norms and using Jensen's inequality (norm is convex) yields
\begin{align*}
\bigl\|\nabla_{\theta} f_{\xi}(\lambda_{\theta})\bigr\|
&=
\left\|
\mathbb{E}_{\pi_{\theta},\,p}\!\left[
\sum_{t=0}^{\infty} \gamma^{t}\,
\frac{\partial f_{\xi}(\lambda_{\theta})}{\partial \lambda_{\theta}(s_{t},a_{t})}
\left(\sum_{h=0}^{t} \nabla_{\theta}\log \pi_{\theta}(a_{h}\mid s_{h})\right)
\right]
\right\| \\
&\le
\mathbb{E}_{\pi_{\theta},\,p}\!\left[
\left\|
\sum_{t=0}^{\infty} \gamma^{t}\,
\frac{\partial f_{\xi}(\lambda_{\theta})}{\partial \lambda_{\theta}(s_{t},a_{t})}
\left(\sum_{h=0}^{t} \nabla_{\theta}\log \pi_{\theta}(a_{h}\mid s_{h})\right)
\right\|
\right] \\
&\le
\mathbb{E}_{\pi_{\theta},\,p}\!\left[
\sum_{t=0}^{\infty} \gamma^{t}\,
\left|
\frac{\partial f_{\xi}(\lambda_{\theta})}{\partial \lambda_{\theta}(s_{t},a_{t})}
\right|\,
\left\|
\sum_{h=0}^{t} \nabla_{\theta}\log \pi_{\theta}(a_{h}\mid s_{h})\right\|
\right],
\end{align*}
where the last step uses the triangle inequality $\|\sum_t v_t\|\le \sum_t \|v_t\|$ and the fact that
$\|c v\| = |c|\,\|v\|$ for scalar $c$.
Now, since by Assumption~\ref{ass:lambda_bound}
\[
\left|
\frac{\partial f_{\xi}(\lambda_{\theta})}{\partial \lambda_{\theta}(s_{t},a_{t})}
\right|
\le \ell_\lambda,
\]
we obtain
\[
\bigl\|\nabla_{\theta} f_{\xi}(\lambda_{\theta})\bigr\|
\le
\ell_\lambda\,
\mathbb{E}_{\pi_{\theta},\,p}\!\left[\sum_{t=0}^{\infty}\gamma^t\left\|\sum_{h=0}^{t} \nabla_{\theta}\log \pi_{\theta}(a_{h}\mid s_{h})\right\|\right].
\]
Finally, by Assumption~\ref{ass:pi_theta_bound}, we have
$\bigl\|\nabla_{\theta}\log \pi_{\theta}(a\mid s)\bigr\|\le \ell_{\pi_\theta}$ for all $(s,a)$,
and hence
\[
\left\|\sum_{h=0}^{t} \nabla_{\theta}\log \pi_{\theta}(a_{h}\mid s_{h})\right\|
\le
\sum_{h=0}^{t}\bigl\|\nabla_{\theta}\log \pi_{\theta}(a_{h}\mid s_{h})\bigr\|
\le
(t+1)\ell_{\pi_\theta}.
\]
Substituting this bound and use the standard series $\displaystyle \sum_{t=0}^{\infty}\gamma^{t}(t+1)=\frac{1}{(1-\gamma)^{2}}$ yields
\[
\bigl\|\nabla_{\theta} f_{\xi}(\lambda_{\theta})\bigr\|
\le
\ell_{\lambda}\,\ell_{\pi_\theta}
\sum_{t=0}^{\infty}\gamma^t(t+1) = \frac{\ell_{\pi_\theta}\ell_{\lambda}}{(1-\gamma)^{2}} .
\]
\end{proof}

\subsection{Proof of Proposition~\ref{prop:lips-theta-xi}}
Before proving Proposition~\ref{prop:lips-theta-xi}, we will first prove some useful lemmas (adopted from Lemmas 2,3 of~\citep{chen2024robust}).
\begin{lemma}
\label{lem:occ-bellman}
For any $\theta\in\Theta$, the occupancy measure
$\lambda_{\theta}\;\overset{\mathrm{def}}{=}\; (1-\gamma)\sum_{t=0}^{+\infty}\gamma^{t}\,
\mathbb{P}_{\pi_{\theta},\,p}\!\big( s_t = s,\, a_t = a \,\big|\, s_0 \sim \rho \big)$ is the unique solution to the
following Bellman equation of $\lambda\in\mathbb{R}^{|\mathcal S|\times|\mathcal A|}$:
\[
\lambda(s',a')
=
\Bigl[(1-\gamma)\,\rho(s') + \gamma \sum_{s,a} \lambda(s,a)\,p(s'\mid s,a)\Bigr]\,
\pi_{\theta}(a'\mid s'),
\qquad s'\in\mathcal S,\ a'\in\mathcal A .
\]
Therefore, the state-occupancy measure
$\lambda_{\theta}(s) \overset{\text{def}}{=} \sum_{a\in\mathcal A}\lambda_{\theta}(s,a)$
satisfies
\[
\lambda_{\theta}(s,a)
= \lambda_{\theta}(s)\,\pi_{\theta}(a\mid s).
\]
\end{lemma}
\begin{proof}
See Lemma 2 from~\citep{chen2024robust} where we consider a fixed environment dynamic.
\end{proof}
\begin{lemma}
\label{lem:lipschitz-occ}
Under Assumption~\ref{ass:pi_theta_bound}, the occupancy measure $\lambda_{\theta}$ satisfies, for any
$\theta,\theta'\in\Theta$,
\[
\|\lambda_{\theta'}-\lambda_{\theta}\|
\;\le\;
\frac{\ell_{\pi_\theta}\sqrt{|\mathcal A|}\,\|\theta'-\theta\|}{1-\gamma}.
\]
\end{lemma}
\begin{proof}
For any $\theta,\theta'\in\Theta$, we have
\begin{align*}
\|\lambda_{\theta'}-\lambda_{\theta}\|
&= \Bigg[\sum_{s',a'}\big|\lambda_{\theta'}(s',a')-\lambda_{\theta}(s',a')\big|^2\Bigg]^{1/2} \\[3pt]
&\overset{(i)}{=}
\Bigg[\sum_{s',a'}
\Big|\big[(1-\gamma)\rho(s')+\gamma\!\sum_{s,a}\lambda_{\theta'}(s,a)p(s'|s,a)\big]\pi_{\theta'}(a'|s') \\
&-\big[(1-\gamma)\rho(s')+\gamma\!\sum_{s,a}\lambda_{\theta}(s,a)p(s'|s,a)\big]\pi_{\theta}(a'|s')\Big|^{2} \Bigg]^{1/2}\\
&=\Bigg[\sum_{s',a'}
\Big| (1-\gamma)\rho(s')\big[\pi_{\theta'}(a'|s')-\pi_{\theta}(a'|s')\big] + \gamma \pi_{\theta'}(a'|s')\!\sum_{s,a}\lambda_{\theta'}(s,a)p(s'|s,a) \\
&+\gamma \pi_{\theta'}(a'|s')\!\sum_{s,a}\lambda_{\theta}(s,a)p(s'|s,a)-\gamma \pi_{\theta'}(a'|s')\!\sum_{s,a}\lambda_{\theta}(s,a)p(s'|s,a)-\gamma \pi_{\theta}(a'|s')\!\sum_{s,a}\lambda_{\theta}(s,a)p(s'|s,a)\Big|^{2}\Bigg]^{1/2} \\
&=\Bigg[\sum_{s',a'}
\Big|\gamma\,\pi_{\theta'}(a'|s')\sum_{s,a}\big[\lambda_{\theta'}(s,a)-\lambda_{\theta}(s,a)\big]p(s'|s,a) \\
&+\Big((1-\gamma)\rho(s')+\gamma\!\sum_{s,a}\lambda_{\theta}(s,a)p(s'|s,a)\Big)
      \big[\pi_{\theta'}(a'|s')-\pi_{\theta}(a'|s')\big]\Big|^{2}\Bigg]^{1/2} \\[3pt]
&\overset{(ii)}{\le}
\Bigg[\sum_{s',a'}\Big|\gamma\,\pi_{\theta'}(a'|s')\sum_{s,a}\big[\lambda_{\theta'}(s,a)-\lambda_{\theta}(s,a)\big]p(s'|s,a)\Big|^{2}\Bigg]^{1/2} \\
&\quad
+\Bigg[\sum_{s',a'}\Big|
\Big((1-\gamma)\rho(s')+\gamma\!\sum_{s,a}\lambda_{\theta}(s,a)p(s'|s,a)\Big)
      \big[\pi_{\theta'}(a'|s')-\pi_{\theta}(a'|s')\big]\Big|^{2}\Bigg]^{1/2} \\
\end{align*}
where (i) uses the occupancy Bellman equation in Lemma~\ref{lem:occ-bellman}, (ii) uses the
triangle inequality. Next, for the first term, we use the fact that $\sum_{a'}\pi_{\theta'}(a'|s')^{2}\le 1$. For the second term, we use $\sum_{s'}\big((1-\gamma)\rho(s')+\gamma\sum_{s,a}\lambda_{\theta}(s,a)p(s'|s,a)\big)=1$ and get:
\begin{align*}
\|\lambda_{\theta'}-\lambda_{\theta}\|&\le
\gamma\Bigg[\sum_{s'}\Big(\sum_{s,a}p(s'|s,a)\,
      \big|\lambda_{\theta'}(s,a)-\lambda_{\theta}(s,a)\big|\Big)^{2}\Bigg]^{1/2}
+\Bigg[\sum_{a'}\big(\pi_{\theta'}(a'|s')-\pi_{\theta}(a'|s')\big)^{2}\Bigg]^{1/2} \\[3pt]
&\overset{(i)}{\le}
\gamma\,\|\lambda_{\theta'}-\lambda_{\theta}\|
+\Bigg[\sum_{a'}\ell_{\pi_\theta}\|\theta'-\theta\|^{2}\Bigg]^{1/2} \\[3pt]
&\le
\gamma\,\|\lambda_{\theta'}-\lambda_{\theta}\|
+\ell_{\pi_\theta}\sqrt{|\mathcal A|}\,\|\theta'-\theta\| ,
\end{align*}
(i) uses $\|\nabla_\theta\pi_\theta\|\le\|\nabla_\theta\log\pi_\theta\|\le \ell_{\pi_\theta}$ (Assumption~\ref{ass:pi_theta_bound}) and by Lagrange mean value theorem, we have $\big\|  \pi_{\theta'}(a \mid s)-   \pi_{\theta}(a \mid s) \big\|
\;\le\; \ell_{\pi_\theta}\, \|\theta' - \theta\|$.
Rearranging yields the result and completes the proof.
\end{proof}

Next, we prove Proposition~\ref{prop:lips-theta-xi}:
\begin{proof}
Let's start with the first inequality. We decompose the difference:
\[
\nabla_{\theta} f_{\xi'}(\lambda_{\theta'})-\nabla_{\theta} f_{\xi}(\lambda_{\theta})
\;=\underbrace{\bigl(\nabla_{\theta} f_{\xi'}(\lambda_{\theta'})-\nabla_{\theta} f_{\xi}(\lambda_{\theta'})\bigr)}_{\text{change in }\xi}
\;+\;
\underbrace{\bigl(\nabla_{\theta} f_{\xi}(\lambda_{\theta'})-\nabla_{\theta} f_{\xi}(\lambda_{\theta})\bigr)}_{\text{change in }\theta}.
\]
For the first term, according to Theorem~\ref{theorem:gradient}, we have:
\[
\nabla_{\theta} f_{\xi'}(\lambda_{\theta'})-\nabla_{\theta} f_{\xi}(\lambda_{\theta'})=\mathbb{E}_{\pi_{\theta'},\,p}\!\left[
\sum_{t=0}^{\infty} \gamma^{t}\,
\left(\frac{\partial f_{\xi'}(\lambda_{\theta'})}{\partial \lambda_{\theta}}-\frac{\partial f_{\xi}(\lambda_{\theta'})}{\partial \lambda_{\theta}}\right)
\left(\sum_{h=0}^{t} \nabla_{\theta}\log \pi_{\theta'}(a_{h}\mid s_{h})\right)\right]
\]
Take norms, use Jensen's inequality and triangle inequality,  and by Assumptions~\ref{ass:pi_theta_bound} and~\ref{ass:xi_lambda_bound}, we have:
\begin{align*}
\bigl\|\nabla_{\theta} f_{\xi'}(\lambda_{\theta'})-\nabla_{\theta} f_{\xi}(\lambda_{\theta'})\bigr\|&\le\mathbb{E}_{\pi_{\theta'},\,p}\!\left[
\sum_{t=0}^{\infty} \gamma^{t}\,
\|\frac{\partial f_{\xi'}(\lambda_{\theta'})}{\partial \lambda_{\theta}}-\frac{\partial f_{\xi}(\lambda_{\theta'})}{\partial \lambda_{\theta}}\|
\sum_{h=0}^{t} \bigl\|\nabla_\theta\log \pi_{\theta'}(a_h\mid s_h)\bigr\|\right]\\
&\le \|\frac{\partial f_{\xi'}(\lambda_{\theta'})}{\partial \lambda_{\theta}}-\frac{\partial f_{\xi}(\lambda_{\theta'})}{\partial \lambda_{\theta}}\|_{\infty}\,\ell_{\pi_\theta}
\sum_{t\ge 0}\gamma^{t}(t+1) \\
&\le \frac{\ell_{\pi_\theta}}{(1-\gamma)^2}\,L_{\lambda,\xi}\,\|\xi'-\xi\|
\end{align*}
For the second term, add-subtract two terms:
\begin{align*}
    \nabla_{\theta} f_{\xi}(\lambda_{\theta'})-\nabla_{\theta} f_{\xi}(\lambda_{\theta}) &=  \mathbb{E}_{\pi_{\theta'}}\!\left[\sum_{t} \gamma^{t}\, \frac{\partial f_{\xi}(\lambda_{\theta'})}{\partial \lambda_{\theta}} \sum_{h\le t}\nabla_\theta \log \pi_{\theta'}\right]
 - \mathbb{E}_{\pi_{\theta}}\!\left[\sum_{t} \gamma^{t}\, \frac{\partial f_{\xi}(\lambda_{\theta})}{\partial \lambda_{\theta}} \sum_{h\le t}\nabla_\theta \log \pi_{\theta}\right]   \\
    &=\underbrace{\mathbb{E}_{\pi_{\theta}}\!\left[\sum_{t} \gamma^{t}\, \frac{\partial f_{\xi}(\lambda_{\theta})}{\partial \lambda_{\theta}}
      \sum_{h\le t}\bigl(\nabla_\theta \log \pi_{\theta'}-\nabla_\theta \log \pi_{\theta}\bigr)\right]}_{\text{(B1)}} \nonumber\\
&\quad +\underbrace{\bigl(\mathbb{E}_{\pi_{\theta'}}-\mathbb{E}_{\pi_{\theta}}\bigr)
      \!\left[\sum_{t} \gamma^{t}\, \frac{\partial f_{\xi}(\lambda_{\theta'})}{\partial \lambda_{\theta}}
      \sum_{h\le t}\nabla_\theta \log \pi_{\theta'}\right]}_{\text{(B2)}} \nonumber\\
&\quad +\underbrace{\mathbb{E}_{\pi_{\theta}}\!\left[\sum_{t} \gamma^{t}\,
      \bigl(\frac{\partial f_{\xi}(\lambda_{\theta'})}{\partial \lambda_{\theta}}-\frac{\partial f_{\xi}(\lambda_{\theta})}{\partial \lambda_{\theta}}\bigr)
      \sum_{h\le t}\nabla_\theta \log \pi_{\theta'}\right]}_{\text{(B3)}} .
\end{align*}
Below, we bound each term. For (B1), using Assumption~\ref{ass:pi_theta_bound} and~\ref{ass:lambda_bound}, we have:
\[
\|( \text{B1} )\|
\le \ell_\lambda\, L_{\pi_\theta}\,\|\theta'-\theta\|
   \sum_{t\ge 0}\gamma^{t}(t+1)
= \frac{\ell_\lambda\, L_{\pi_\theta}}{(1-\gamma)^2}\,\|\theta'-\theta\|.
\]
For (B3), by Assumption~\ref{ass:lambda_bound} and Lemma~\ref{lem:lipschitz-occ} we have:
\[
\|\frac{\partial f_{\xi}(\lambda_{\theta'})}{\partial \lambda_{\theta}}-\frac{\partial f_{\xi}(\lambda_{\theta})}{\partial \lambda_{\theta}}\|
\;\le\; L_\lambda \,\|\lambda_{\theta'}-\lambda_\theta\|
\;\le\; L_\lambda\, \,\frac{\ell_{\pi_\theta}\sqrt{|\mathcal A|}\,\|\theta'-\theta\|}{1-\gamma}.
\]
Thus, using \(\sum_{h\le t}\|\nabla_\theta \log \pi_{\theta'}\|\le (t+1)\,\ell_{\pi_\theta}\) for all $\theta$,
\[
\|( \text{B3} )\|
\;\le\; L_\lambda \frac{\ell_{\pi_\theta}\sqrt{|\mathcal A|}\,\|\theta'-\theta\|}{1-\gamma} \,
\ell_{\pi_\theta}\sum_{t\ge 0}\gamma^{t}(t+1)
\;=\; \frac{\ell_{\pi_\theta}^{2}
        L_\lambda\sqrt{|{\mathcal A}|} }
        {(1-\gamma)^3}\|\theta'-\theta\|
\]
For (B2), define the trajectory
\(\tau = (s_0,a_0,s_1,a_1,\ldots)\) and the functional
\[
Y(\tau;\theta') \;:=\; \sum_{t\ge 0} \gamma^{t} \, U_t(\tau;\theta'),
\qquad
U_t(\tau;\theta') \;:=\; \frac{\partial f_{\xi}(\lambda_{\theta'})}{\partial \lambda_{\theta}} \;
\sum_{h=0}^{t} \nabla_\theta \log \pi_{\theta'}(a_h\mid s_h).
\]
We will bound \(\bigl\|(\text{B2})\bigr\|=\,\bigl\|\mathbb{E}_{\pi_{\theta'}}[Y]-\mathbb{E}_{\pi_\theta}[Y]\bigr\|.\) Let \(\theta_\alpha := \theta + \alpha(\theta' - \theta)\), \(\alpha\in[0,1]\).
By the fundamental theorem of calculus,
\[
\mathbb{E}_{\pi_{\theta'}}[Y] - \mathbb{E}_{\pi_\theta}[Y]
\;=\;
\int_{0}^{1} \frac{d}{d\alpha}\,
\mathbb{E}_{\pi_{\theta_\alpha}}[Y] \, d\alpha .
\]
Notice that here \(Y(\cdot;\theta')\) is held fixed while differentiating
with respect to \(\alpha\) and  only the sampling distribution \(\pi_{\theta_\alpha}\) changes. Factor the trajectory density under fixed dynamics:
\[
p_{\theta_\alpha}(\tau)
= \rho(s_0)\prod_{j\ge 0}
\Bigl[\pi_{\theta_\alpha}(a_j\mid s_j)\,p(s_{j+1}\mid s_j,a_j)\Bigr].
\]
For each fixed \(t\), the term \(\gamma^{t}U_t(\tau;\theta')\) depends only on the
prefix \((s_0,a_0,\ldots,s_t,a_t)\). Thus
\[
\mathbb{E}_{\pi_{\theta_\alpha}}\!\big[\gamma^{t}U_t\big]
=
\sum_{s_{0:t},\,a_{0:t}}
\gamma^{t}\,U_t\;
\rho(s_0)\prod_{j=0}^{t}
\Bigl[\pi_{\theta_\alpha}(a_j\mid s_j)\,p(s_{j+1}\mid s_j,a_j)\Bigr],
\]
where we have already marginalized out the future \(j>t\).
Differentiate the expectation by differentiating the product
\(\prod_{j=0}^{t}\pi_{\theta_\alpha}(a_j\mid s_j)\).
By the chain rule,
\[
\frac{d}{d\alpha}\,\pi_{\theta_\alpha}(a_j\mid s_j)
= \bigl(\nabla_\theta \pi_{\theta_\alpha}(a_j\mid s_j)\bigr)^\top(\theta'-\theta)
= \pi_{\theta_\alpha}(a_j\mid s_j)\,
\bigl(\nabla_\theta \log \pi_{\theta_\alpha}(a_j\mid s_j)\bigr)^\top(\theta'-\theta).
\]
Applying the product rule to \(\prod_{j=0}^{t}\pi_{\theta_\alpha}(a_j\mid s_j)\) gives a
sum of score terms up to \(t\):
\[
\frac{d}{d\alpha}\!\left(\prod_{j=0}^{t}\pi_{\theta_\alpha}(a_j\mid s_j)\right)
=
\left(\prod_{j=0}^{t}\pi_{\theta_\alpha}(a_j\mid s_j)\right)
\left[\sum_{j=0}^{t}\bigl(\nabla_\theta \log \pi_{\theta_\alpha}(a_j\mid s_j)\bigr)^\top(\theta'-\theta)\right].
\]
Plugging this into the derivative of the expectation yields
\[
\frac{d}{d\alpha}\,
\mathbb{E}_{\pi_{\theta_\alpha}}\!\big[\gamma^{t}U_t\big]
=
\mathbb{E}_{\pi_{\theta_\alpha}}\!\left[
\gamma^{t}U_t(\tau;\theta')\,
\sum_{j=0}^{t}\nabla_\theta \log \pi_{\theta_\alpha}(a_j\mid s_j)
\right]^{\!\top}(\theta'-\theta).
\]
Notice that only the action factors up to time $t$ depend on $\theta_\alpha$, summing over $t$ yields:
\[
\frac{d}{d\alpha}\,\mathbb{E}_{\pi_{\theta_\alpha}}[Y]
=
\mathbb{E}_{\pi_{\theta_\alpha}}\!\left[
\sum_{t\ge 0}\gamma^{t}\,U_t(\tau;\theta')
\sum_{j=0}^{t}\nabla_\theta \log \pi_{\theta_\alpha}(a_j\mid s_j)
\right]^{\!\top}(\theta'-\theta).
\]
Taking norms and applying the triangle inequality gives
\[
\left\|
\frac{d}{d\alpha}\,\mathbb{E}_{\pi_{\theta_\alpha}}[Y]
\right\|
\le
\mathbb{E}_{\pi_{\theta_\alpha}}\!\left[
\sum_{t\ge 0}\gamma^{t}\,\|U_t(\tau;\theta')\|
\sum_{j=0}^{t}\bigl\|\nabla_\theta \log \pi_{\theta_\alpha}(a_j\mid s_j)\bigr\|
\right]\cdot \|\theta'-\theta\|.
\]
Using Assumpion~\ref{ass:lambda_bound} and~\ref{ass:pi_theta_bound}, we have:
\[
\|U_t(\tau;\theta')\|
= |\frac{\partial f_{\xi}(\lambda_{\theta'})}{\partial \lambda_{\theta}}|\;
\sum_{h=0}^{t}\bigl\|\nabla_\theta \log \pi_{\theta'}(a_h\mid s_h)\bigr\|
\;\le\; \ell_\lambda\,(t+1)\,\ell_{\pi_\theta}.
\]
Also \(\displaystyle \sum_{j=0}^{t}\bigl\|\nabla_{\theta_\alpha}\log \pi_{\theta_\alpha}\bigr\|
\le (t+1)\,\ell_{\pi_\theta}\).
Plugging into the bound gives
\[
\left\|
\frac{d}{d\alpha}\,\mathbb{E}_{\pi_{\theta_\alpha}}[Y]
\right\|
\le
\ell_\lambda\,\ell_{\pi_\theta}^{\,2}
\sum_{t\ge 0}\gamma^{t}(t+1)^{2}\;\|\theta'-\theta\|.
\]
The series has the closed form
\[
\sum_{t=0}^{\infty}\gamma^{t}(t+1)^{2}
= \frac{1+\gamma}{(1-\gamma)^{3}}.
\]
Hence
\[
\left\|
\frac{d}{d\alpha}\,\mathbb{E}_{\pi_{\theta_\alpha}}[Y]
\right\|
\le
\frac{(1+\gamma)\,\ell_\lambda\,\ell_{\pi_\theta}^{\,2}}{(1-\gamma)^{3}}\;
\|\theta'-\theta\|.
\]
Finally,
\[
\|(\text{B2})\|
= \bigl\|\mathbb{E}_{\pi_{\theta'}}[Y]-\mathbb{E}_{\pi_\theta}[Y]\bigr\|
= \left\|\int_{0}^{1} \frac{d}{d\alpha}\,\mathbb{E}_{\pi_{\theta_\alpha}}[Y]\; d\alpha \right\|
\;\le\;
\frac{(1+\gamma)\,\ell_\lambda\,\ell_{\pi_\theta}^{\,2}}{(1-\gamma)^{3}}\;\|\theta'-\theta\|.
\]
Absorb the factor
\(1+\gamma \le 2\):
\[
\|(\text{B2})\|
\;\le\;
\frac{2\,\ell_\lambda\,\ell_{\pi_\theta}^{\,2}}{(1-\gamma)^{3}}\;\|\theta'-\theta\|.
\]
Combine (B1)-(B3) complete the proof for the first inequality.

As for the second inequality, start with a add-subtract decomposition and the triangle inequality:
\[
\bigl\|\nabla_{\xi} f_{\xi'}(\lambda_{\theta'})-\nabla_{\xi} f_{\xi}(\lambda_{\theta})\bigr\|
\;\le\;
\underbrace{\bigl\|\nabla_{\xi} f_{\xi'}(\lambda_{\theta'})-\nabla_{\xi} f_{\xi'}(\lambda_{\theta})\bigr\|}_{\text{change in }\lambda}
\;+\;
\underbrace{\bigl\|\nabla_{\xi} f_{\xi'}(\lambda_{\theta})-\nabla_{\xi} f_{\xi}(\lambda_{\theta})\bigr\|}_{\text{change in }\xi}.
\]
For the first term, by Assumption~\ref{ass:xi_lambda_bound},
\[
\bigl\|\nabla_{\xi} f_{\xi'}(\lambda_{\theta'})-\nabla_{\xi} f_{\xi'}(\lambda_{\theta})\bigr\|
\;\le\; L_{\xi,\lambda}\,\|\lambda_{\theta'}-\lambda_{\theta}\|.
\]
Apply Lemma~\ref{lem:lipschitz-occ}:
\[
\bigl\|\nabla_{\xi} f_{\xi'}(\lambda_{\theta'})-\nabla_{\xi} f_{\xi'}(\lambda_{\theta})\bigr\|
\;\le\; L_{\xi,\lambda}\,\frac{\ell_{\pi_\theta}\sqrt{|\mathcal{A}|}}{1-\gamma}\,\|\theta'-\theta\|.
\]
For the second term, by Assumption~\ref{ass:xi_lambda_bound}, we have directly:
\[
\bigl\|\nabla_{\xi} f_{\xi'}(\lambda_{\theta})-\nabla_{\xi} f_{\xi}(\lambda_{\theta})\bigr\|
\;\le\; L_{\xi,\xi}\,\|\xi'-\xi\|.
\]
Combine the above results to complete the proof. 
\end{proof}

\subsection{Algorithm~\ref{alg:pgda} in the Model-based Tabular Settings}
\begin{algorithm}[t]
\caption{PGDA for Model-based Tabular Robust Utility}
\label{alg:pgda-tabular}
\begin{algorithmic}[1]
\State \textbf{Hyperparameters:} Iteration numbers $K,T$, stepsizes $\eta,\beta$.
\State \textbf{Initialize:} $\lambda_0\in\Lambda$, $\xi_0\in\Xi$.
\For{iterations $k = 0,1,\ldots,K-1$}
    \State Apply the projected stochastic gradient ascent: \(
    \xi_{k+1}= \operatorname{proj}_\Xi \bigl(\xi_{k} + \beta\,\nabla_\xi f_{\xi_{k}}(\lambda_k)\bigr).
    \)
    
    \State Apply the projected stochastic gradient descent: \(
        \lambda_{k+1}=\mathrm{proj}_{\Lambda}\!\left(\lambda_k-\eta\,  \nabla_\lambda f_{\xi_{k+1}}(\lambda_k)\right) 
    \) 
\EndFor
\State \textbf{Output:} averaged iterate $\bar\lambda_K=\frac1K\sum_{k=0}^{K-1}\lambda_{k}$ and $\bar\xi_K=\frac1K\sum_{k=0}^{K-1}\xi_{k}$.
\end{algorithmic}
\end{algorithm}

\label{app:pgda-tabular}
In general RL, the outer problem, where we update the policy parameter $\theta$, is nonconvex. However, in the model-based tabular setting, one can recover a policy from any feasible occupancy $\lambda$ via
\[
   \pi_\lambda (a|s) = \frac{\lambda(s,a)}{\sum_{a'}\lambda(s,a')}.
   \]
So optimizing over policies $\pi_\theta$ is equivalent to optimizing over occupancies~\citep{laroche2023occupancy,abbasi2019large}. Moreover, the achievable occupancies $\Lambda$ is a convex polytope. Therefore, instead of optimizing over a nonconvex policy parameter $\theta$, for this setting, we can optimize over a the convex polytope $\Lambda$. In particular, our robust utility problem in \eqref{eq:robust-utility} becomes:
\[
   \min_{\lambda \in \Lambda} \max_{\xi \in \Xi} f_\xi(\lambda)
\]
And the algorithm can be simplifed as in Algorithm~\ref{alg:pgda-tabular}, where we do not need gradient estimation and the problem is convex-concave in $(\lambda, \xi)$, we use the standard single-loop projected gradient descent-ascent method and replace the outer projected gradient descent step in Algorithm~\ref{alg:pgda} (line 11) with:
 \[
        \lambda_{k+1}=\mathrm{proj}_{\Lambda}\!\left(\lambda_k-\eta\,  \nabla_\lambda f_{\xi_{k+1}}(\lambda_k)\right)
    \]
Next, we establish a convergence guarantee for Algorithm~\ref{alg:pgda-tabular}. Since the model-based tabular formulation yields a convex-concave saddle-point problem, Algorithm~\ref{alg:pgda-tabular} enjoys a global convergence guarantee measured by the primal-dual saddle gap, rather than a local stationarity bound. We first state a useful lemma, then we perform the convergence analysis.
\begin{lemma}
\label{lem:proj_nonexpansive}
Under Assumption~\ref{ass:xi-convex-compact}, the Euclidean projection
$\mathrm{proj}_{\Xi}(\cdot)$ is nonexpansive. In particular, for all $u,v\in\mathbb{R}^d$,
\[
\|\mathrm{proj}_{\Xi}(u)-\mathrm{proj}_{\Xi}(v)\|\le \|u-v\|.
\]
\end{lemma}

\begin{proof}
Let $p=\mathrm{proj}_{\Xi}(u)$ and $q=\mathrm{proj}_{\Xi}(v)$. By the optimality condition of Euclidean projection onto a closed convex set,
\[
\langle u-p, x-p\rangle \le 0 \quad \forall x\in\Xi,
\qquad
\langle v-q, x-q\rangle \le 0 \quad \forall x\in\Xi.
\]
Choose $x=q$ in the first inequality and $x=p$ in the second to get
\[
\langle u-p, q-p\rangle \le 0,
\qquad
\langle v-q, p-q\rangle \le 0.
\]
Adding them yields
\[
\langle (u-v)-(p-q),\, q-p\rangle \le 0
\quad\Longrightarrow\quad
\langle u-v,\, p-q\rangle \ge \|p-q\|^2.
\]
By Cauchy-Schwarz, $\|p-q\|^2 \le \|u-v\|\,\|p-q\|$, hence $\|p-q\|\le\|u-v\|$.
\end{proof}

\begin{theorem}
\label{thm:tabular_pgda_gap_pgda}
Consider Algorithm~\ref{alg:pgda-tabular} and assume Assumptions~\ref{ass:xi-convex-compact} to~\ref{ass:f-convex-concave} hold. Define the saddle gap
\[
\mathrm{Gap}(\lambda,\xi):=\max_{\xi'\in\Xi} f_{\xi'}(\lambda)\;-\;\min_{\lambda'\in\Lambda} f_{\xi}(\lambda').
\]
Assume $D_\Lambda:=\max_{\lambda,\lambda'\in\Lambda}\|\lambda-\lambda'\|$ be the diameter of $\Lambda$. Let $\bar\lambda_K=\frac1K\sum_{k=0}^{K-1}\lambda_{k}$ and $\bar\xi_K=\frac1K\sum_{k=0}^{K-1}\xi_{k}$, then for all $K\ge 1$, 
\[
\mathrm{Gap}(\bar{\lambda}_K,\bar{\xi}_K)
\;\le\;
\frac{D_\Lambda^2}{2\eta K}+\frac{D_\Xi^2}{2\beta K}
+\frac{\eta}{2}\ell_\lambda^2+\frac{\beta}{2}\ell_\xi^2. 
\]
Choosing $\eta = D_\Lambda/(\ell_\lambda\sqrt{K})$ and $\beta=D_\Xi/(\ell_\xi\sqrt{K})$ yields
\[
\mathrm{Gap}(\bar{\lambda}_K,\bar{\xi}_K)\ \le\ \frac{D_\Lambda \ell_\lambda + D_\Xi \ell_\xi}{\sqrt{K}}.
\]
\end{theorem}

\begin{proof}
According to Lemma~\ref{lem:proj_nonexpansive}, for any closed convex set $\mathcal{C}$, any $u\in\mathcal{C}$, and any $v$,
\begin{equation}
\label{eq:proj-ineq-basic}
\|\mathrm{proj}_{\mathcal{C}}(v)-u\|^2 \le \|v-u\|^2.
\end{equation}
Fix any $\lambda\in\Lambda$ and apply~\eqref{eq:proj-ineq-basic} with $\mathcal{C}=\Lambda$ and
$v=\lambda_k-\eta \nabla_\lambda f_{\xi_{k}}(\lambda_k)$:
\[
\|\lambda_{k+1}-\lambda\|^2
\le
\|\lambda_k-\eta \nabla_\lambda f_{\xi_{k}}(\lambda_k)-\lambda\|^2
=
\|\lambda_k-\lambda\|^2
-2\eta\langle \nabla_\lambda f_{\xi_{k}}(\lambda_k),\lambda_k-\lambda\rangle
+\eta^2\|\nabla_\lambda f_{\xi_{k}}(\lambda_k)\|^2.
\]
Rearranging gives
\begin{equation}
\label{eq:lambda-step}
\langle \nabla_\lambda f_{\xi_{k}}(\lambda_k),\lambda_k-\lambda\rangle
\le
\frac{\|\lambda_k-\lambda\|^2-\|\lambda_{k+1}-\lambda\|^2}{2\eta}
+\frac{\eta}{2}\|\nabla_\lambda f_{\xi_{k}}(\lambda_k)\|^2.
\end{equation}
Fix any $\xi\in\Xi$ and apply~\eqref{eq:proj-ineq-basic} with $\mathcal{C}=\Xi$ and
$v=\xi_k+\beta \nabla_\xi f_{\xi_k}(\lambda_k)$:
\[
\|\xi_{k+1}-\xi\|^2
\le
\|\xi_k+\beta \nabla_\xi f_{\xi_k}(\lambda_k)-\xi\|^2
=
\|\xi_k-\xi\|^2
+2\beta\langle \nabla_\xi f_{\xi_k}(\lambda_k),\xi_k-\xi\rangle
+\beta^2\|\nabla_\xi f_{\xi_k}(\lambda_k)\|^2.
\]
Rearranging and flipping the inner product yields
\begin{equation}
\label{eq:xi-step}
\langle \nabla_\xi f_{\xi_k}(\lambda_k),\xi-\xi_k\rangle
\le
\frac{\|\xi_k-\xi\|^2-\|\xi_{k+1}-\xi\|^2}{2\beta}
+\frac{\beta}{2}\|\nabla_\xi f_{\xi_k}(\lambda_k)\|^2.
\end{equation}
By convexity of $f_{\xi_k}(\lambda)$ w.r.t. $\lambda$,
\[
f_{\xi_{k}}(\lambda_k)-f_{\xi_k}(\lambda)\le
\langle \nabla_\lambda f_{\xi_k}(\lambda_k),\lambda_k-\lambda\rangle.
\]
By concavity of $f_{\xi}(\lambda_k)$ w.r.t. $\xi$,
\[
f_{\xi}(\lambda_k)-f_{\xi_k}(\lambda_k)\le
\langle \nabla_\xi f_{\xi_k}(\lambda_k),\xi-\xi_k\rangle.
\]
Adding above gives, for any $(\lambda,\xi)\in\Lambda\times\Xi$,
\begin{equation}
\label{eq:key}
f_{\xi}(\lambda_k)-f_{\xi_k}(\lambda)\le
\langle \nabla_\lambda f_{\xi_k}(\lambda_k),\lambda_k-\lambda\rangle
+
\langle \nabla_\xi f_{\xi_k}(\lambda_k),\xi-\xi_k\rangle.
\end{equation}
Plug~\eqref{eq:lambda-step} and~\eqref{eq:xi-step} into~\eqref{eq:key} and sum from $k=0$ to $K-1$:
\begin{align*}
\sum_{k=0}^{K-1}\bigl(f_{\xi}(\lambda_k)-f_{\xi_k}(\lambda)\bigr)
&\le
\frac{\|\lambda_0-\lambda\|^2-\|\lambda_K-\lambda\|^2}{2\eta}
+
\frac{\|\xi_0-\xi\|^2-\|\xi_K-\xi\|^2}{2\beta} \\
&\quad
+\frac{\eta}{2}\sum_{k=0}^{K-1}\|\nabla_\lambda f_{\xi_k}(\lambda_k)\|^2
+\frac{\beta}{2}\sum_{k=0}^{K-1}\|\nabla_\xi f_{\xi_k}(\lambda_k)\|^2.
\end{align*}
Drop the nonpositive terms $-\|\lambda_K-\lambda\|^2$ and $-\|\xi_K-\xi\|^2$ and use the uniform bounds
$\|\nabla_\lambda f_{\xi_k}(\lambda_k)\|\le \ell_\lambda$ (Assumption~\ref{ass:lambda_bound}), $\|\nabla_\xi f_{\xi_k}(\lambda_k)\|\le \ell_\xi$ (Assumption~\ref{ass:xi_bound}):
\[
\frac{1}{K}\sum_{k=0}^{K-1}\bigl(f_{\xi}(\lambda_k)-f_{\xi_k}(\lambda)\bigr)
\le
\frac{\|\lambda_0-\lambda\|^2}{2\eta K}
+
\frac{\|\xi_0-\xi\|^2}{2\beta K}
+\frac{\eta}{2}\ell_\lambda^2+\frac{\beta}{2}\ell_\xi^2.
\]
By convexity in $\lambda$ and concavity in $\xi$, Jensen's inequality yields
\[
f_{\xi}(\bar{\lambda}_K)\le \frac{1}{K}\sum_{k=0}^{K-1} f_{\xi}(\lambda_k),
\qquad
f_{\bar{\xi}_K}(\lambda)\ge \frac{1}{K}\sum_{k=0}^{K-1} f_{\xi_k}(\lambda).
\]
Hence for any $(\lambda,\xi)$,
\[
f_{\xi}(\bar{\lambda}_K)-f_{\bar{\xi}_K}(\lambda)
\le
\frac{1}{K}\sum_{k=0}^{K-1}\bigl(f_{\xi}(\lambda_k)-f_{\xi_k}(\lambda)\bigr).
\]
Combine above and then maximize over $\xi\in\Xi$ and minimize over $\lambda\in\Lambda$, to obtain 
\begin{align*}
\mathrm{Gap}(\bar{\lambda}_K,\bar{\xi}_K)
&=
\max_{\xi\in\Xi} f_{\xi}(\bar{\lambda}_K)-\min_{\lambda\in\Lambda} f_{\bar{\xi}_K}(\lambda) \\
&\le
\frac{1}{2\eta K}\max_{\lambda\in\Lambda}\|\lambda_0-\lambda\|^2
+\frac{1}{2\beta K}\max_{\xi\in\Xi}\|\xi_0-\xi\|^2
+\frac{\eta}{2}\ell_\lambda^2+\frac{\beta}{2}\ell_\xi^2, \\ 
&\le 
\frac{D_\Lambda^2}{2\eta K}+\frac{D_\Xi^2}{2\beta K}
+\frac{\eta}{2}\ell_\lambda^2+\frac{\beta}{2}\ell_\xi^2.
\end{align*}
This completes the proof.
\end{proof}

\subsection{Proof of Proposition~\ref{prop:Gamma-smooth}}

Let's first establish some useful lemmas that will be used throughout the proof.
\begin{lemma}
\label{lem:strong-convex-error-bound}
Let $g:\mathbb{R}^d \to (-\infty,+\infty]$ be a proper, closed, $\mu$-strongly convex function with $\mu>0$.
Then:
\begin{enumerate}
  \item The subdifferential $\partial g$ is $\mu$-strongly monotone, i.e., for all $x,x' \in \operatorname{dom} g$ and all
  $s \in \partial g(x)$, $s' \in \partial g(x')$,
  \begin{equation}
  \label{eq:strong-monotone-subdiff}
  \langle s - s', x - x' \rangle
  \;\ge\;
  \mu \,\|x - x'\|^2.
  \end{equation}
  \item Let $x^\star \in \arg\min g$ (so $0 \in \partial g(x^\star)$). Then for any $x \in \operatorname{dom} g$,
  \begin{equation}
  \label{eq:grad-distance-error-bound}
  \|x - x^\star\|
  \;\le\;
  \frac{1}{\mu}\,
  \inf_{s \in \partial g(x)} \|s\|.
  \end{equation}
\end{enumerate}
\end{lemma}

\begin{proof}
(i) By $\mu$-strong convexity, for any $x,x' \in \operatorname{dom} g$ and any $s \in \partial g(x)$, $s' \in \partial g(x')$,
\begin{align*}
g(x') &\ge g(x) + \langle s, x' - x \rangle + \frac{\mu}{2}\|x' - x\|^2, \\
g(x)  &\ge g(x') + \langle s', x - x' \rangle + \frac{\mu}{2}\|x - x'\|^2. 
\end{align*}
Adding the above equations together and using $\|x'-x\| = \|x-x'\|$ gives
\[
0
\;\ge\;
\langle s, x' - x \rangle + \langle s', x - x' \rangle + \mu \|x - x'\|^2
=
- \langle s - s', x - x' \rangle + \mu \|x - x'\|^2.
\]
Rearranging completes the proof.

(ii) Let $x^\star \in \arg\min g$. Then $0 \in \partial g(x^\star)$ by optimality. Apply \eqref{eq:strong-monotone-subdiff} with $x' = x^\star$ and $s' = 0$:
for any $x$ and any $s \in \partial g(x)$,
\[
\langle s, x - x^\star \rangle
\;\ge\;
\mu \|x - x^\star\|^2.
\]
By Cauchy-Schwarz,
\[
\mu \|x - x^\star\|^2
\;\le\;
\langle s, x - x^\star \rangle
\;\le\;
\|s\|\,\|x - x^\star\|.
\]
If $x \neq x^\star$, divide both sides by $\|x - x^\star\|$ to obtain
\[
\|x - x^\star\| \;\le\; \frac{1}{\mu}\,\|s\| \qquad \forall\, s \in \partial g(x).
\]
Taking the infimum over $s \in \partial g(x)$ gives \eqref{eq:grad-distance-error-bound}.
If $x = x^\star$, both sides of \eqref{eq:grad-distance-error-bound} are zero and the inequality is trivial.
\end{proof}

\begin{lemma}
\label{lem:inner-error-bound}
Fix $\theta\in\Theta$. Let $\Xi^\star(\theta):=\arg\max_{\xi\in\Xi} f_\xi(\lambda_\theta)$ be the set of maximizers and let
\[
\delta(\theta,\xi)
:=
\inf_{v \in N_\Xi(\xi)} \bigl\| -\nabla_\xi f_\xi(\lambda_\theta) + v \bigr\|,
\]
where $N_\Xi(\xi)$ is the normal cone of $\Xi$ at $\xi$. Under Assumption~\ref{ass:f-convex-concave}, for any $\xi\in\Xi$ and any
$\xi^\star(\theta) \in \Xi^\star(\theta)$, we have
\[
\|\xi - \xi^\star(\theta)\|
\;\le\;
\frac{1}{\mu_\xi}\,
\delta(\theta,\xi).
\]
In particular, for the iterates $\{\xi_k\}$ in Algorithm~\ref{alg:pgda}, we obtain
\[
\|\xi_k - \xi^\star(\theta_k)\|
\;\le\;
\frac{1}{\mu_\xi}\,
\delta_k,\qquad
\delta_k := \inf_{v \in N_\Xi(\xi_k)} \bigl\| -\nabla_\xi f_{\xi_k}(\lambda_{\theta_k}) + v \bigr\|.
\]
\end{lemma}

\begin{proof}
Fix $\theta\in\Theta$ and define the convex objective
\[
g_\theta(\xi) := - f_\xi(\lambda_\theta) + I_\Xi(\xi),
\]
where $I_\Xi(\xi) = 0$ if $\xi\in\Xi$ and $I_\Xi(\xi)=+\infty$ otherwise is the indicator
of the feasible set. Then minimizers of $g_\theta$ are exactly maximizers of $f_\xi(\lambda_\theta)$ over $\Xi$. Therefore,
\[
\Xi^\star(\theta)
=
\arg \max_{\xi\in\Xi} f_\xi(\lambda_\theta)
=
\arg \min_{\xi \in \mathbb{R}^{d_\Xi}} g_\theta(\xi),
\] 
where $d_\Xi$ is the dimension of $\xi$. By Assumption~\ref{ass:f-convex-concave}, $f_\xi(\lambda_\theta)$ is $\mu_\xi$-strongly concave
on $\Xi$, hence $-f_\xi(\lambda_\theta)$ is $\mu_\xi$-strongly convex on $\Xi$.
Adding the indicator $I_\Xi$ (which is convex) preserves strong convexity, so $g_\theta$ is
$\mu_\xi$-strongly convex on $\mathbb{R}^d$. Moreover, the subdifferential of $g_\theta$  satisfies
\begin{equation}
\label{eq:subdiff-g}
\partial g_\theta(\xi)
=
- \nabla_\xi f_\xi(\lambda_\theta) + N_\Xi(\xi),
\end{equation}
where $N_\Xi(\xi):=\{v: \langle v, \xi'-\xi \rangle \le 0, \forall \xi'\in \Xi \}$ is the normal cone of $\Xi$ at $\xi$, i.e., all vectors that point outwards from the feasible set at the point $\xi$. Let $\xi^\star(\theta)\in\Xi^\star(\theta)$ be any minimizer of $g_\theta$. Then
$0\in\partial g_\theta(\xi^\star(\theta))$ is a first-order optimality condition. Using Lemma~\ref{lem:strong-convex-error-bound}, take $x=\xi$, $x^\star = \xi^\star(\theta)$, $s=\partial g_\theta(\xi)$, and $s'=0\in\partial g_\theta(\xi^\star(\theta))$ completes the proof.
\end{proof}
Next, we will prove a useful lemma showing the projection optimality condition:
\begin{lemma}
\label{lem:proj_opt_cond}
Let $\mathcal{C}\subseteq\mathbb{R}^d$ be a nonempty closed convex set and let
$ \operatorname{proj}_{\mathcal{C}}(z)$ denote the Euclidean projection of $z$ onto $\mathcal{C}$.
Then the following optimality conditions hold:
\[
\langle z-\operatorname{proj}_{\mathcal{C}}(z),\ x-\operatorname{proj}_{\mathcal{C}}(z)\rangle \le 0,\qquad \forall\, x\in\mathcal{C}.
\]
\end{lemma}
\begin{proof}
Let $z^+ := \operatorname{proj}_{\mathcal{C}}(z)$. By definition, $z^+$ is a minimizer of the convex function
\[
\min_{y\in\mathcal{C}} \ \frac12\|y-z\|^2.
\]
For any $x\in\mathcal{C}$ and any $t\in[0,1]$, the point $y_t := z^+ + t(x-z^+)$ also lies in $\mathcal{C}$ by convexity.
Define $\varphi(t) := \frac12\|y_t - z\|^2$. Since $z^+$ is optimal, $\varphi(t)\ge \varphi(0)$ for all $t\in[0,1]$, hence
$\varphi'(0)\ge 0$. Differentiating gives
\[
\varphi'(t) = \bigl\langle y_t - z,\ x - z^+\bigr\rangle,
\quad\text{so}\quad
\varphi'(0) = \bigl\langle z^+ - z,\ x - z^+\bigr\rangle \ge 0.
\]
Rearranging yields
\[
\langle z - z^+,\ x - z^+\rangle \le 0,\qquad \forall x\in\mathcal{C},
\]
which is exactly the desired optimality condition.
\end{proof}
Next, we prove Proposition~\ref{prop:Gamma-smooth}.
\begin{proof}
By Assumption~\ref{ass:f-convex-concave}, for each fixed $\theta$ the inner
problem
\[
\max_{\xi\in\Xi} f_\xi(\lambda_\theta)
\]
is a concave maximization over a compact convex set~$\Xi$ (Assumption~\ref{ass:xi-convex-compact}). Hence the maximum can be attained and the envelope
\[
\Gamma(\theta) := \max_{\xi\in\Xi} f_\xi(\lambda_\theta)
\]
is well-defined. Moreover, by Danskin's theorem~\citep{danskin1966theory},
for any selection $\xi^\star(\theta) \in \Xi^\star(\theta)=\arg\max_{\xi\in\Xi} f_\xi(\lambda_\theta)$ we
have
\begin{equation}
\label{eq:Gamma-grad}
\nabla_\theta \Gamma(\theta) = \nabla_\theta f_{\xi^\star(\theta)}(\lambda_\theta).
\end{equation}
Fix $\theta,\theta'\in\Theta$ and denote
\[
\xi^\star := \xi^\star(\theta),
\qquad
\xi^{\star'} := \xi^\star(\theta').
\]
For each $\theta$, define the strongly convex function
\[
g_\theta(\xi) := - f_\xi(\lambda_\theta) + I_\Xi(\xi),
\]
where $I_\Xi$ is the indicator of $\Xi$. Assumption~\ref{ass:f-convex-concave}
implies that $g_\theta$ is $\mu_\xi$-strongly convex in $\xi$, and
\[
\Xi^\star(\theta) = \arg\min_{\xi} g_\theta(\xi).
\]
For $\theta$ fixed, the subdifferential of $g_\theta$ has the form
\[
\partial g_\theta(\xi) = -\nabla_\xi f_\xi(\lambda_\theta) + N_\Xi(\xi).
\]
Since $\xi^\star$ minimizes $g_\theta$, we have $0\in\partial g_\theta(\xi^\star)$,
i.e., there exists $v \in N_\Xi(\xi^\star)$ such that
\[
-\nabla_\xi f_{\xi^\star}(\lambda_\theta) + v = 0.
\]
Similarly, there exists $v' \in N_\Xi(\xi^{\star'})$ with
\[
-\nabla_\xi f_{\xi^{\star'}}(\lambda_{\theta'}) + v' = 0.
\]
Now apply Lemma~\ref{lem:strong-convex-error-bound} to
$g_\theta$, with $x = \xi^{\star'}$ and $x^\star = \xi^\star$:
\[
\|\xi^{\star'} - \xi^\star\|
\;\le\;
\frac{1}{\mu_\xi}\, \text{inf} \,
\|\partial g_\theta(\xi^{\star'})\|.
\]
Using the explicit form of $\partial g_\theta$, we have
\[
\partial g_\theta(\xi^{\star'})
=
- \nabla_\xi f_{\xi^{\star'}}(\lambda_\theta) + N_\Xi(\xi^{\star'}),
\]
so by choosing the particular normal vector $v' \in N_\Xi(\xi^{\star'})$ above, we get
\[
\text{inf} \,
\|\partial g_\theta(\xi^{\star'})\|
\;\le\;
\bigl\| -\nabla_\xi f_{\xi^{\star'}}(\lambda_\theta) + v' \bigr\|
=
\bigl\| \nabla_\xi f_{\xi^{\star'}}(\lambda_{\theta'}) - \nabla_\xi f_{\xi^{\star'}}(\lambda_\theta) \bigr\|.
\]
By Proposition~\ref{prop:lips-theta-xi}, $\nabla_\xi f_\xi(\lambda_\theta)$ is
$L_{\xi,\theta}$-Lipschitz in $\theta$ for fixed $\xi$, hence
\[
\bigl\| \nabla_\xi f_{\xi^{\star'}}(\lambda_{\theta'}) - \nabla_\xi f_{\xi^{\star'}}(\lambda_\theta) \bigr\|
\;\le\;
L_{\xi,\theta}\, \|\theta' - \theta\|.
\]
Combining the above results yields
\[
\|\xi^{\star'} - \xi^\star\|
\;\le\;
\frac{1}{\mu_\xi}\,
\text{inf} \,
\|\partial g_\theta(\xi^{\star'})\|
\;\le\;
\frac{L_{\xi,\theta}}{\mu_\xi}\, \|\theta' - \theta\|,
\]
Using the envelope gradient formula \eqref{eq:Gamma-grad}, we write
\[
\nabla_\theta \Gamma(\theta')
 - \nabla_\theta \Gamma(\theta)
=
\nabla_\theta f_{\xi^{\star'}}(\lambda_{\theta'})
 - \nabla_\theta f_{\xi^\star}(\lambda_\theta).
\]
Add and subtract $\nabla_\theta f_{\xi^{\star'}}(\lambda_\theta)$:
\begin{align*}
\|\nabla_\theta \Gamma(\theta') - \nabla_\theta \Gamma(\theta)\|
&\le
\big\|\nabla_\theta f_{\xi^{\star'}}(\lambda_{\theta'})
      - \nabla_\theta f_{\xi^{\star'}}(\lambda_\theta)\big\|
\\ &\quad
+ \big\|\nabla_\theta f_{\xi^{\star'}}(\lambda_\theta)
      - \nabla_\theta f_{\xi^\star}(\lambda_\theta)\big\|.
\end{align*}
The first term is controlled by the $\theta$-Lipschitz constant $L_{\theta,\theta}$
from Proposition~\ref{prop:lips-theta-xi}:
\[
\big\|\nabla_\theta f_{\xi^{\star'}}(\lambda_{\theta'})
      - \nabla_\theta f_{\xi^{\star'}}(\lambda_\theta)\big\|
\;\le\;
L_{\theta,\theta}\,\|\theta' - \theta\|.
\]
The second term is controlled by the $\xi$-Lipschitz constant $L_{\theta,\xi}$:
\[
\big\|\nabla_\theta f_{\xi^{\star'}}(\lambda_\theta)
      - \nabla_\theta f_{\xi^\star}(\lambda_\theta)\big\|
\;\le\;
L_{\theta,\xi}\,\|\xi^{\star'} - \xi^\star\|.
\]
Finally we have
\[
\big\|\nabla_\theta \Gamma(\theta') - \nabla_\theta \Gamma(\theta)\big\|
\;\le\;
L_{\theta,\theta}\,\|\theta' - \theta\|
+ L_{\theta,\xi}\,\frac{L_{\xi,\theta}}{\mu_\xi}\,\|\theta' - \theta\|
=
\Bigl(L_{\theta,\theta} + \frac{L_{\theta,\xi}L_{\xi,\theta}}{\mu_\xi}\Bigr)
\|\theta' - \theta\|,
\]
which completes the proof, with
$L_\Gamma := L_{\theta,\theta} + (L_{\theta,\xi}L_{\xi,\theta})/\mu_\xi$.
\end{proof}

\subsection{Proof of Theorem~\ref{theorem:pgda_converge}}

\begin{proof}
According to Proposition~\ref{prop:Gamma-smooth}, $\Gamma$ is well-defined and $L$-smooth on $\Theta$. We will use $\Gamma$ to express the true $\theta$ gradient we care about. We want to compare the update direction $g_k^{(\theta)}$ with the true envelope gradient $\nabla_\theta \Gamma (\theta_k)$. Fix some measurable selection \(\xi^{\star}(\theta_k) \in \Xi^{\star}(\theta_k)\), we have
\[
   \nabla_\theta \Gamma (\theta_k)=\nabla_\theta f_{\xi^{\star}(\theta_k)}(\lambda_{\theta_k}).
   \]
Define
\[
    d_k := g_k^{(\theta)} - \nabla_\theta \Gamma(\theta_k).
\]
We decompose:
\[
   d_k=g_k^{(\theta)}-\nabla_\theta \Gamma (\theta_k) = \big(g_k^{(\theta)}-\nabla_\theta f_{\xi_k}(\lambda_{\theta_k})\big) + \big(\nabla_\theta f_{\xi_k}(\lambda_{\theta_k})-\nabla_\theta f_{\xi^{\star}(\theta_k)}(\lambda_{\theta_k})\big)
   \]
We define the MC/statistical error $e_k:=g_k^{(\theta)}-\nabla_\theta f_{\xi_k}(\lambda_{\theta_k})$ and the envelope bias $b_k:=\nabla_\theta f_{\xi_k}(\lambda_{\theta_k})-\nabla_\theta f_{\xi^{\star}(\theta_k)}(\lambda_{\theta_k})$. From Proposition~\ref{prop:error_grad_estimate}, we have that:
\[
 \mathbb{E}[\|e_k\|^2|\theta_k,\xi_k] \le E_\theta.
   \]
From Proposition~\ref{prop:lips-theta-xi}, we have:
\[
   \|b_k\|=\|\nabla_\theta f_{\xi_k}(\lambda_{\theta_k})-\nabla_\theta f_{\xi^{\star}(\theta_k)}(\lambda_{\theta_k}) \|\le L_{\theta,\xi}\|\xi_k-\xi^{\star}(\theta_k)\|
   \]
So bias is controlled by how far $\xi_k$ is from an inner maximizer $\xi^{\star}(\theta_k)$. Based on Lemma~\ref{lem:inner-error-bound}, we have 
\[\|\xi_k - \xi^\star(\theta_k)\|
\;\le\;
\frac{1}{\mu_\xi}\,
\delta_k. \]
Hence, 
\[
   \|b_k\| \le L_{\theta,\xi}\frac{1}{\mu_\xi}\,\delta_k.
   \]
And we have:
\[
   \mathbb{E}\|b_k\|^2 \le L^2_{\theta,\xi}\frac{1}{\mu^2_\xi}\mathbb{E}[\delta^2_k]
   \]
Therefore,
\[
\mathbb{E}\|d_k\|^2 \le 2\mathbb{E}\|e_k\|^2 + 2\mathbb{E}\|b_k\|^2\le 2E_\theta + 2 L^2_{\theta,\xi}\frac{1}{\mu^2_\xi}\mathbb{E}[\delta^2_k]
\]
Notice that the actual projected step used by the Algorithm~\ref{alg:pgda} is:
\[
\tilde G_k \;\triangleq\; \frac{1}{\eta}\,(\theta_k-\theta_{k+1})
= \frac{1}{\eta}\Bigl(\theta_k-\operatorname{proj}_{\Theta}\bigl(\theta_k-\eta\,g_k^{(\theta)}\bigr)\Bigr).
\]
While the exact projected gradient mapping for the envelope is:
\[
G_{\Gamma}(\theta_k) \;\triangleq\; \frac{1}{\eta}\Bigl(\theta_k-\operatorname{proj}_{\Theta}\bigl(\theta_k-\eta\,\nabla\Gamma(\theta_k)\bigr)\Bigr).
\]
Now we use the smoothness of $\Gamma$ to get a descent inequality for the $\theta$-update. Start from $L_\Gamma$-smoothness of $\Gamma$, we have:
\[
\Gamma(\theta_{k+1})
\le
\Gamma(\theta_k)
+\langle \nabla \Gamma(\theta_k),\,\theta_{k+1}-\theta_k\rangle
+\frac{L_\Gamma}{2}\|\theta_{k+1}-\theta_k\|^2.
\]
Since $\theta_{k+1}-\theta_k=-\eta \tilde G_k$, we have
\[
\Gamma(\theta_{k+1})
\le
\Gamma(\theta_k)
-\eta\langle \nabla \Gamma(\theta_k),\,\tilde G_k\rangle
+\frac{L_\Gamma \eta^2}{2}\|\tilde G_k\|^2.
\]
Now use Lemma~\ref{lem:proj_opt_cond}. Because
\[
\theta_{k+1}=\operatorname{proj}_{\Theta}\!\bigl(\theta_k-\eta g_k^{(\theta)}\bigr),
\]
we have, by Lemma~\ref{lem:proj_opt_cond} with $x=\theta_k\in\Theta$ and $z=\theta_k-\eta g_k^{(\theta)}$
\[
\langle \theta_k-\eta g_k^{(\theta)}-\theta_{k+1},\,\theta_k-\theta_{k+1}\rangle \le 0.
\]
Divide by $\eta^2$ and use $\tilde G_k=(\theta_k-\theta_{k+1})/\eta$:
\[
\langle g_k^{(\theta)},\,\tilde G_k\rangle \ge \|\tilde G_k\|^2.
\]
Since $g_k^{(\theta)}=\nabla\Gamma(\theta_k)+d_k$,
\[
\langle \nabla\Gamma(\theta_k),\,\tilde G_k\rangle
=
\langle g_k^{(\theta)},\,\tilde G_k\rangle-\langle d_k,\,\tilde G_k\rangle
\ge \|\tilde G_k\|^2-\langle d_k,\,\tilde G_k\rangle.
\]
Substitute this into the smoothness inequality:
\[
\Gamma(\theta_{k+1})
\le
\Gamma(\theta_k)
-\eta\Bigl(1-\frac{L_\Gamma\eta}{2}\Bigr)\|\tilde G_k\|^2
+\eta\langle d_k,\,\tilde G_k\rangle.
\]
Under $\eta\le 1/L_\Gamma$, we have $1-\frac{L_\Gamma\eta}{2}\ge \frac{1}{2}$, so
\[
\Gamma(\theta_{k+1})
\le
\Gamma(\theta_k)
-\frac{\eta}{2}\|\tilde G_k\|^2
+\eta\langle d_k,\,\tilde G_k\rangle.
\]
Apply Young's inequality~\citep{alzer2019young} $\langle a,b \rangle \le \frac{1}{2c} \|a\|^2 + \frac{c}{2} \|b\|^2$:
\[
\eta\langle d_k,\,\tilde G_k\rangle
\le \frac{\eta}{4}\|\tilde G_k\|^2+\eta\|d_k\|^2,
\]
and get
\[
\Gamma(\theta_{k+1})
\le
\Gamma(\theta_k)
-\frac{\eta}{4}\|\tilde G_k\|^2
+\eta\|d_k\|^2.
\]
Use nonexpansiveness of projection in Lemma~\ref{lem:proj_nonexpansive}:
\[
\bigl\|\operatorname{proj}_{\Theta}(\theta_k-\eta g_k^{(\theta)})
-\operatorname{proj}_{\Theta}(\theta_k-\eta\nabla\Gamma(\theta_k))\bigr\|
\le \eta\|d_k\|.
\]
Therefore,
\[
\|\tilde G_k-G_\Gamma(\theta_k)\|\le \|d_k\|.
\]
Hence
\[
\|G_\Gamma(\theta_k)\|^2\le 2\|\tilde G_k\|^2+2\|d_k\|^2\]
Equivalently, 
\[
\|\tilde G_k\|^2\ge \frac{1}{2}\|G_\Gamma(\theta_k)\|^2-\|d_k\|^2.
\]
Plug this into the descent bound:
\[
\Gamma(\theta_{k+1})
\le
\Gamma(\theta_k)
-\frac{\eta}{8}\|G_\Gamma(\theta_k)\|^2
+\frac{5\eta}{4}\|d_k\|^2.
\]
Taking expectation,
\[
\mathbb{E}\Gamma(\theta_{k+1})
\le
\mathbb{E}\Gamma(\theta_k)
-\frac{\eta}{8}\mathbb{E}\|G_\Gamma(\theta_k)\|^2
+\frac{5\eta}{4}\mathbb{E}\|d_k\|^2.
\]
Using
\[
\mathbb{E}\|d_k\|^2
\le 2E_\theta + 2\frac{L_{\theta,\xi}^2}{\mu_\xi^2}\,\mathbb{E}[\delta_{k+1}^2],
\]
we obtain
\[
\mathbb{E}\Gamma(\theta_{k+1})
\le
\mathbb{E}\Gamma(\theta_k)
-\frac{\eta}{8}\mathbb{E}\|G_\Gamma(\theta_k)\|^2
+\frac{5\eta}{2}E_\theta
+\frac{5\eta}{2}\frac{L_{\theta,\xi}^2}{\mu_\xi^2}\,\mathbb{E}[\delta_{k+1}^2].
\]
Sum from $k=0$ to $K-1$, telescope, and use $\Gamma^\star\le \Gamma(\theta_k)$:
\[
\frac{1}{K}\sum_{k=0}^{K-1}\mathbb{E}\|G_\Gamma(\theta_k)\|^2
\le
\frac{8(\Gamma(\theta_0)-\Gamma^\star)}{\eta K}
+20E_\theta
+20\frac{L_{\theta,\xi}^2}{\mu_\xi^2}\cdot \frac{1}{K}\sum_{k=0}^{K-1}\mathbb{E}[\delta_{k+1}^2].
\]
This completes the proof.
\end{proof}

\subsection{Connection between Algorithm~\ref{alg:pgda} with Prior Work.}
\label{app:connection}

To further illustrate how our framework provides a unified view of existing RL formulations (Section~\ref{sec:examples}), we show that several representative algorithms from each setting can be viewed as special cases of Algorithm~\ref{alg:pgda}.

\noindent\textbf{RL with General Utility.}
In this case, $\Xi=\{\xi\}$ and the inner loop (lines 4-9) is omitted from Algorithm~\ref{alg:pgda}. Existing work on RL with general utility~\citep{barakat2024towards,barakat2023reinforcement} typically focuses on improving occupancy-measure estimation in large state-action spaces or reducing sample complexity via variance-reduction techniques; these can be viewed as variants of Algorithm~\ref{alg:pgda} augmented with such tools.

\noindent\textbf{Reward-Robust RL.}
In this case, by Proposition~\ref{prop:reward_robust_as_ru}, $\xi$ corresponds to the reward function $-R$, and $\Xi$ specifies the uncertainty set of rewards. A common approach is to compute (exactly or approximately) the worst-case reward in the uncertainty set~\citep{behzadian2021fast,ho2021partial,bagnell2001solving,grand2021scalable}, which plays the role of the inner maximization over $\xi$ (with projection onto $\Xi$) in Algorithm~\ref{alg:pgda} (lines 4-9), followed by an outer policy update using this worst-case reward. Accordingly, robust policy-gradient methods under reward uncertainty can be viewed as instances of Algorithm~\ref{alg:pgda} in which the inner ascent is replaced by a worst-case oracle.

\noindent\textbf{Constrained RL.}
In this case, by Proposition~\ref{prop:crl_as_ru}, $\xi$ denotes the Lagrange multipliers and $\Xi=\mathbb{R}_{+}^{d_\Xi}$. Standard primal-dual methods for Constrained RL~\citep{ding2020natural,bai2023achieving,ding2025convergence} typically alternate between (i) a policy update in $\theta$ and (ii) a projected gradient-ascent step on the multipliers $\xi$. These algorithms are therefore direct instances of Algorithm~\ref{alg:pgda}, up to differences in how gradients are estimated. For example,~\citep{ding2020natural} performs the policy update using the natural policy gradient (Fisher geometry) and uses the same projected-ascent dual update for $\xi$ as in our algorithm (see Eq.~(7) in~\citep{ding2020natural}).

\subsection{Proof of Lemma~\ref{lem:prox_to_orig}}

\begin{proof}
We first prove the bound for the $\theta$-mapping. By the definitions
of $\mathcal{G}_{\Theta}$ and
$\mathcal{G}_{\Theta,\sigma}^{k}$,
\begin{align*}
\mathcal{G}_{\Theta}(\theta,\xi)
-
\mathcal{G}_{\Theta,\sigma}^{k}(\theta,\xi)=
\frac{1}{\alpha}
\Bigl[
\operatorname{proj}_{\Theta}
\left(
\theta
-
\alpha
\left[
\nabla_\theta f_\xi(\lambda_\theta)
+
\sigma(\theta-\theta_k)
\right]
\right)
-
\operatorname{proj}_{\Theta}
\left(
\theta
-
\alpha\nabla_\theta f_\xi(\lambda_\theta)
\right)
\Bigr].
\end{align*}
By the triangle inequality, we have
\begin{align*}
\left\|
\mathcal{G}_{\Theta}(\theta,\xi)
\right\|\le
\left\|
\mathcal{G}_{\Theta,\sigma}^{k}(\theta,\xi)
\right\|+
\left\|
\mathcal{G}_{\Theta}(\theta,\xi)
-
\mathcal{G}_{\Theta,\sigma}^{k}(\theta,\xi)
\right\|.
\end{align*}
Using the nonexpansiveness of the projection operator from
Lemma~\ref{lem:proj_nonexpansive}, we obtain
\begin{align*}
\left\|
\mathcal{G}_{\Theta}(\theta,\xi)
-
\mathcal{G}_{\Theta,\sigma}^{k}(\theta,\xi)
\right\|\le
\frac{1}{\alpha}
\left\|
\alpha\sigma(\theta-\theta_k)
\right\|=
\sigma\|\theta-\theta_k\|.
\end{align*}
Substituting this bound into the preceding triangle inequality proves the first inequality.

We next prove the corresponding bound for the $\xi$-mapping. By the
definitions of $\mathcal{G}_{\Xi}$ and
$\mathcal{G}_{\Xi,\sigma}^{k}$,
\begin{align*}
\mathcal{G}_{\Xi}(\theta,\xi)
-
\mathcal{G}_{\Xi,\sigma}^{k}(\theta,\xi)=
\frac{1}{\alpha}
\Bigl[
\operatorname{proj}_{\Xi}
\left(
\xi
+
\alpha
\left[
\nabla_\xi f_\xi(\lambda_\theta)
-
\sigma(\xi-\xi_k)
\right]
\right)
-
\operatorname{proj}_{\Xi}
\left(
\xi+\alpha\nabla_\xi f_\xi(\lambda_\theta)
\right)
\Bigr].
\end{align*}
By the triangle inequality,
\begin{align*}
\left\|
\mathcal{G}_{\Xi}(\theta,\xi)
\right\|\le
\left\|
\mathcal{G}_{\Xi,\sigma}^{k}(\theta,\xi)
\right\|+
\left\|
\mathcal{G}_{\Xi}(\theta,\xi)
-
\mathcal{G}_{\Xi,\sigma}^{k}(\theta,\xi)
\right\|.
\end{align*}
Applying Lemma~\ref{lem:proj_nonexpansive} again gives
\begin{align*}
\left\|
\mathcal{G}_{\Xi}(\theta,\xi)
-
\mathcal{G}_{\Xi,\sigma}^{k}(\theta,\xi)
\right\|\le
\frac{1}{\alpha}
\left\|
\alpha\sigma(\xi-\xi_k)
\right\|=
\sigma\|\xi-\xi_k\|.
\end{align*}
Substituting this inequality into the preceding bound completes the proof.
\end{proof}

\subsection{Monotonicity of Saddle-point Operator \(\mathcal{F}_k(z)\)}
\label{app:mon-saddle-oper}
Let
\[
\mathcal{Z}:=\Theta\times\Xi,
\qquad
z:=(\theta,\xi),
\qquad
z_k:=(\theta_k,\xi_k),
\]
and define the saddle-point operator associated with the original
robust utility problem in~\eqref{eq:robust-utility} as
\[
\mathcal{F}(z)
:=
\bigl(
\nabla_\theta f_\xi(\lambda_\theta),
-\nabla_\xi f_\xi(\lambda_\theta)
\bigr),
\]
For the $k$-th outer iteration, define the prox-regularized saddle
objective
\[
\Phi_k(\theta,\xi)
:=
f_\xi(\lambda_\theta)
+
\frac{\sigma}{2}\|\theta-\theta_k\|^2
-
\frac{\sigma}{2}\|\xi-\xi_k\|^2.
\]
Then its associated saddle-point operator is
\[
\mathcal{F}_k(z)
:=
\bigl(
\nabla_\theta \Phi_k(\theta,\xi),
-\nabla_\xi \Phi_k(\theta,\xi)
\bigr)
=
\mathcal{F}(z)+\sigma(z-z_k).
\]
Define the joint gradient Lipschitz constants as 
\[
L_{\mathcal{F}}
:=
\left|
\begin{pmatrix}
L_{\theta,\theta} & L_{\theta,\xi}\\
L_{\xi,\theta} & L_{\xi,\xi}
\end{pmatrix}
\right|_2,
\]
where \(|\cdot|_2 \) denotes the spectral norm. 

\begin{lemma}
\label{lem:saddle-operator-regularity}
Under Assumptions~\ref{ass:xi-convex-compact}
to~\ref{ass:xi_lambda_bound}, the operator
$\mathcal{F}$ is $L_{\mathcal{F}}$-Lipschitz continuous and
$L_{\mathcal{F}}$-weakly monotone. That is, for all
$z,z'\in\mathcal{Z}$,
\begin{align}
\|\mathcal{F}(z)-\mathcal{F}(z')\|
&\le
L_{\mathcal{F}}\|z-z'\|,
\label{eq:F-Lipschitz}\\
\left\langle
\mathcal{F}(z)-\mathcal{F}(z'),
z-z'
\right\rangle
&\ge
-L_{\mathcal{F}}\|z-z'\|^2.
\label{eq:F-weakly-monotone}
\end{align}
Consequently, if $\sigma>L_{\mathcal{F}}$, the proximal saddle-point operator $\mathcal{F}_k(z)$ is $(\sigma-L_{\mathcal{F}})$-strongly monotone and $(L_{\mathcal F}+\sigma)$-Lipschitz continuous.
\end{lemma}

\begin{proof}
Recall that an operator $\mathcal{T}:\mathcal{Z}\to\mathbb{R}^{d}$
is $\mu$-strongly monotone if
\[
\left\langle
\mathcal{T}(z)-\mathcal{T}(z'),
z-z'
\right\rangle
\ge
\mu\|z-z'\|^2,
\qquad
\forall z,z'\in\mathcal{Z},
\]
for some $\mu>0$~\citep{bauschke2020correction}. It is weakly monotone with constant $\rho\ge 0$
if
\[
\left\langle
\mathcal{T}(z)-\mathcal{T}(z'),
z-z'
\right\rangle
\ge
-\rho\|z-z'\|^2,
\qquad
\forall z,z'\in\mathcal{Z}.
\]
Let $z=(\theta,\xi)$ and $z'=(\theta',\xi')$. By Proposition~\ref{prop:lips-theta-xi},
\[
\begin{pmatrix}
\|\nabla_\theta f_\xi(\lambda_\theta)
-\nabla_\theta f_{\xi'}(\lambda_{\theta'})\|\\[1mm]
\|\nabla_\xi f_\xi(\lambda_\theta)
-\nabla_\xi f_{\xi'}(\lambda_{\theta'})\|
\end{pmatrix}
\le
\begin{pmatrix}
L_{\theta,\theta} & L_{\theta,\xi}\\
L_{\xi,\theta} & L_{\xi,\xi}
\end{pmatrix}
\begin{pmatrix}
\|\theta-\theta'\|\\
\|\xi-\xi'\|
\end{pmatrix}.
\]
Taking the Euclidean norm on both sides yields~\eqref{eq:F-Lipschitz}. By the Cauchy-Schwarz inequality,
\[
\left\langle
\mathcal{F}(z)-\mathcal{F}(z'),
z-z'
\right\rangle
\ge
-\|\mathcal{F}(z)-\mathcal{F}(z')\|
\|z-z'\|,
\]
which, together with~\eqref{eq:F-Lipschitz}, gives~\eqref{eq:F-weakly-monotone}. Finally,
\begin{align*}
\mathcal{F}_k(z)-\mathcal{F}_k(z')
&=
\mathcal{F}(z)+\sigma(z-z_k)
-
\mathcal{F}(z')
-
\sigma(z'-z_k)\\
&=
\mathcal{F}(z)-\mathcal{F}(z')
+
\sigma\bigl[(z-z_k)-(z'-z_k)\bigr]\\
&=
\mathcal{F}(z)-\mathcal{F}(z')
+
\sigma(z-z').
\end{align*}
Consequently, 
\begin{align*}
\left\langle
\mathcal{F}_k(z)-\mathcal{F}_k(z'),
z-z'
\right\rangle=
\left\langle
\mathcal{F}(z)-\mathcal{F}(z'),
z-z'
\right\rangle
+
\sigma\|z-z'\|^2
\ge
(\sigma-L_{\mathcal{F}})\|z-z'\|^2,
\end{align*}
which proves the strong monotonicity of $\mathcal{F}_k$. We next establish the Lipschitz continuity of $\mathcal F_k$. Using the triangle inequality and the $L_{\mathcal F}$-Lipschitz continuity of $\mathcal F$, we obtain
\begin{align*}
\|\mathcal F_k(z)-\mathcal F_k(z')\|
&\le
\|\mathcal F(z)-\mathcal F(z')\|
+
\sigma\|z-z'\|\\
&\le
L_{\mathcal F}\|z-z'\|
+
\sigma\|z-z'\|\\
&=
(L_{\mathcal F}+\sigma)\|z-z'\|.
\end{align*}
Hence, $\mathcal F_k$ is $(L_{\mathcal F}+\sigma)$-Lipschitz continuous. This completes the proof.
\end{proof}

\subsection{Proof of Lemma~\ref{lem:prox-saddle-VI-equivalence}}
\begin{proof}
By Lemma~\ref{lem:saddle-operator-regularity}, when $\sigma>L_{\mathcal F}$, the operator $\mathcal F_k$ is $(\sigma-L_{\mathcal F})$-strongly monotone. First, consider two points $z=(\theta,\xi)$ and $z'=(\theta',\xi)$ that have the same $\xi$-component. The strong
monotonicity of $\mathcal F_k$ gives
\begin{align*}
\left\langle
\mathcal F_k(\theta,\xi)
-
\mathcal F_k(\theta',\xi),
(\theta-\theta',0)
\right\rangle \ge
(\sigma-L_{\mathcal F})\|\theta-\theta'\|^2.
\end{align*}
Using the definition of $\mathcal F_k$, this becomes
\[
\left\langle
\nabla_\theta\Phi_k(\theta,\xi)
-
\nabla_\theta\Phi_k(\theta',\xi),
\theta-\theta'
\right\rangle
\ge
(\sigma-L_{\mathcal F})\|\theta-\theta'\|^2.
\]
Hence, for every fixed $\xi\in\Xi$, the function $\Phi_k(\cdot,\xi)$ is $(\sigma-L_{\mathcal F})$-strongly convex. Similarly, consider two points $z=(\theta,\xi)$ and $z'=(\theta,\xi')$ that have the same $\theta$-component. Strong monotonicity yields
\begin{align*}
\left\langle
-\nabla_\xi\Phi_k(\theta,\xi)
+
\nabla_\xi\Phi_k(\theta,\xi'),
\xi-\xi'
\right\rangle \ge
(\sigma-L_{\mathcal F})\|\xi-\xi'\|^2,
\end{align*}
or equivalently,
\[
\left\langle
\nabla_\xi\Phi_k(\theta,\xi)
-
\nabla_\xi\Phi_k(\theta,\xi'),
\xi-\xi'
\right\rangle
\le
-(\sigma-L_{\mathcal F})\|\xi-\xi'\|^2.
\]
Therefore, for every fixed $\theta\in\Theta$, the function
$\Phi_k(\theta,\cdot)$ is
$(\sigma-L_{\mathcal F})$-strongly concave. Thus, $\Phi_k$ defines a
strongly convex-strongly concave saddle-point problem.

Suppose first that
$\bar z_{k+1}=(\bar\theta_{k+1},\bar\xi_{k+1})$ is a saddle point of
$\Phi_k$. Then
\[
\bar\theta_{k+1}
\in
\arg\min_{\theta\in\Theta}
\Phi_k(\theta,\bar\xi_{k+1}).
\]
Since $\Phi_k(\cdot,\bar\xi_{k+1})$ is convex and differentiable, its
first-order optimality condition is
\begin{equation}
\label{eq:prox-theta-first-order}
\left\langle
\nabla_\theta
\Phi_k(\bar\theta_{k+1},\bar\xi_{k+1}),
\theta'-\bar\theta_{k+1}
\right\rangle
\ge 0,
\qquad
\forall\theta'\in\Theta.
\end{equation}
Likewise,
\[
\bar\xi_{k+1}
\in
\arg\max_{\xi\in\Xi}
\Phi_k(\bar\theta_{k+1},\xi).
\]
Since $\Phi_k(\bar\theta_{k+1},\cdot)$ is concave and differentiable,
its first-order optimality condition is
\begin{equation}
\label{eq:prox-xi-first-order}
\left\langle
\nabla_\xi
\Phi_k(\bar\theta_{k+1},\bar\xi_{k+1}),
\xi'-\bar\xi_{k+1}
\right\rangle
\le 0,
\qquad
\forall\xi'\in\Xi.
\end{equation}
Adding \eqref{eq:prox-theta-first-order} and the negative of
\eqref{eq:prox-xi-first-order} gives
\begin{align*}
\left\langle
\nabla_\theta
\Phi_k(\bar\theta_{k+1},\bar\xi_{k+1}),
\theta'-\bar\theta_{k+1}
\right\rangle-
\left\langle
\nabla_\xi
\Phi_k(\bar\theta_{k+1},\bar\xi_{k+1}),
\xi'-\bar\xi_{k+1}
\right\rangle
\ge 0.
\end{align*}
By the definition of $\mathcal F_k$, this is precisely
\[
\left\langle
\mathcal F_k(\bar z_{k+1}),
z'-\bar z_{k+1}
\right\rangle
\ge 0,
\qquad
\forall z'=(\theta',\xi')\in\mathcal Z.
\]
Hence, $\bar z_{k+1}$ solves the variational inequality
\eqref{eq:kth-prox-VI}.

Conversely, suppose that $\bar z_{k+1}$ solves~\eqref{eq:kth-prox-VI}. Choosing
\[
z'=(\theta',\bar\xi_{k+1})
\]
for arbitrary $\theta'\in\Theta$ gives
\[
\left\langle
\nabla_\theta
\Phi_k(\bar\theta_{k+1},\bar\xi_{k+1}),
\theta'-\bar\theta_{k+1}
\right\rangle
\ge 0.
\]
By convexity of $\Phi_k(\cdot,\bar\xi_{k+1})$, this implies
\[
\bar\theta_{k+1}
\in
\arg\min_{\theta\in\Theta}
\Phi_k(\theta,\bar\xi_{k+1}).
\]
Similarly, choosing
\[
z'=(\bar\theta_{k+1},\xi')
\]
for arbitrary $\xi'\in\Xi$ gives
\[
\left\langle
\nabla_\xi
\Phi_k(\bar\theta_{k+1},\bar\xi_{k+1}),
\xi'-\bar\xi_{k+1}
\right\rangle
\le 0.
\]
By concavity of $\Phi_k(\bar\theta_{k+1},\cdot)$, this implies
\[
\bar\xi_{k+1}
\in
\arg\max_{\xi\in\Xi}
\Phi_k(\bar\theta_{k+1},\xi).
\]
Therefore, $\bar z_{k+1}$ satisfies the saddle-point condition, i.e., \(\Phi_k(\bar\theta_{k+1},\xi)\le \Phi_k(\bar\theta_{k+1},\bar\xi_{k+1}) \le \Phi_k(\theta,\bar\xi_{k+1}), \forall(\theta,\xi)\in\Theta\times\Xi\).

Finally, since $\mathcal F_k$ is strongly monotone, the variational inequality \eqref{eq:kth-prox-VI} has at most one solution. Hence, the prox-regularized saddle point is unique. This completes the proof.
\end{proof}

\subsection{Property of $\operatorname{Gap}_k(z)$}
\label{app:gap_k_prop}
\begin{lemma}
\label{lem:prox-gap-solution-distance}
Under Assumptions~\ref{ass:xi-convex-compact}, let $\bar z_{k+1}$ denote the solution of the $k$-th prox-regularized subproblem. If $\sigma>L_{\mathcal F}$, then, for every $z\in\mathcal Z$,
\begin{equation}
\label{eq:gap-distance-bound}
\operatorname{Gap}_k(z)
\ge
(\sigma-L_{\mathcal F})
\|z-\bar z_{k+1}\|^2.
\end{equation}
Moreover,
\begin{equation}
\label{eq:gap-zero-equivalence}
\operatorname{Gap}_k(z)=0
\quad\Longleftrightarrow\quad
z=\bar z_{k+1}.
\end{equation}
\end{lemma}

\begin{proof}
Because $z\in\mathcal Z$, choosing $z'=z$ in~\eqref{eq:proximal-VI-gap} shows that
\[
\operatorname{Gap}_k(z)\ge 0.
\]
Furthermore, $\operatorname{Gap}_k(z)=0$ if and only if
\[
\left\langle
\mathcal{F}_k(z),z-z'
\right\rangle
\le 0,
\qquad
\forall z'\in\mathcal Z,
\]
or, equivalently,
\[
\left\langle
\mathcal{F}_k(z),z'-z
\right\rangle
\ge 0,
\qquad
\forall z'\in\mathcal Z.
\]
This is precisely the variational inequality associated with the $k$-th prox-regularized problem~\eqref{eq:kth-prox-VI}. Since $\mathcal{F}_k$ is $(\sigma-L_{\mathcal F})$-strongly monotone, this variational inequality has at most one solution, namely $\bar z_{k+1}$. This proves~\eqref{eq:gap-zero-equivalence}.

To establish the quantitative bound, choose $z'=\bar z_{k+1}$ in the maximization defining $\operatorname{Gap}_k(z)$. Then
\begin{align*}
\operatorname{Gap}_k(z)\ge
\left\langle
\mathcal{F}_k(z),
z-\bar z_{k+1}
\right\rangle=
\left\langle
\mathcal{F}_k(z)-\mathcal{F}_k(\bar z_{k+1}),
z-\bar z_{k+1}
\right\rangle+
\left\langle
\mathcal{F}_k(\bar z_{k+1}),
z-\bar z_{k+1}
\right\rangle.
\end{align*}
Since $\bar z_{k+1}$ solves
\eqref{eq:kth-prox-VI},
\[
\left\langle
\mathcal{F}_k(\bar z_{k+1}),
z-\bar z_{k+1}
\right\rangle
\ge 0.
\]
In addition, the strong monotonicity of $\mathcal F_k$ gives
\[
\left\langle
\mathcal{F}_k(z)-\mathcal{F}_k(\bar z_{k+1}),
z-\bar z_{k+1}
\right\rangle
\ge
(\sigma-L_{\mathcal F})
\|z-\bar z_{k+1}\|^2.
\]
Combining these inequalities proves
\eqref{eq:gap-distance-bound}. 

\end{proof}

\subsection{Proof of Proposition~\ref{prop:inner-gap-decay}}
\label{app:inner-gap-decay}
\begin{lemma}
\label{lem:inner-stochastic-EG-convergence}
Let \(\mu:=\sigma-L_{\mathcal F}>0, L_k:=L_{\mathcal F}+\sigma.
\) Choose the step size $\alpha$ such that
\(
0<\alpha\le\frac{1}{4L_k}.
\) For the $k$-th outer iteration, let $\bar z_{k+1}$ denote the unique solution of the prox-regularized variational inequality associated with $\mathcal F_k$~\eqref{eq:kth-prox-VI}. Define the total stochastic error at outer iteration $k$ by \( \mathcal E_k:= 2 (E_\theta(m_k,H_k,m_k',H_k') + E_{\xi}(m_k, H_k))\), where $E_{\theta}$ and $E_{\xi}$ denote the stochastic gradient-estimation errors as in Proposition~\ref{prop:error_grad_estimate} and \(m_k,H_k,m_k',H_k'\) denote the associated Monte Carlo budgets.
Then,
\begin{equation}
\label{eq:inner-distance-explicit}
\mathbb E_k
\left[
\|z_{k+1}-\bar z_{k+1}\|^2
\right]
\le
\left(
1-\frac{\alpha\mu}{2}
\right)^{T_k}
D_{\mathcal Z}^2
+
\frac{8\mathcal E_k}{\mu^2}.
\end{equation}
\end{lemma}

\begin{proof}
For notational convenience, we write the inner stochastic extragradient updates in the joint form
\begin{align*}
\widetilde z_{k,t}
&=
\operatorname{proj}_{\mathcal Z}
\left(
z_{k,t}
-
\alpha
\left[
\mathcal F_k(z_{k,t})
+
\varepsilon_{k,t}
\right]
\right),\\
z_{k,t+1}
&=
\operatorname{proj}_{\mathcal Z}
\left(
z_{k,t}
-
\alpha
\left[
\mathcal F_k(\widetilde z_{k,t})
+
\varepsilon_{k,t}'
\right]
\right),
\end{align*}
where $\varepsilon_{k,t}$ and $\varepsilon_{k,t}'$ denote the
stochastic errors in the prediction (Algorithm~\ref{alg:pe-pgda}, Line 6) and correction (Algorithm~\ref{alg:pe-pgda}, Line 9) steps, respectively. The product projection above is equivalent to the separate projections onto $\Theta$ and $\Xi$ used in Algorithm~\ref{alg:pe-pgda} (Lines 7-8, 10-11).

We first establish the contraction of the exact extragradient iteration, where we ignore the stochastic error. For each inner iterate $z_{k,t}$, define the exact
prediction and correction points by
\begin{align*}
\widehat z_{k,t}
&:=
\operatorname{proj}_{\mathcal Z}
\left(
z_{k,t}-\alpha\mathcal F_k(z_{k,t})
\right),\\
\widehat z_{k,t+1}
&:=
\operatorname{proj}_{\mathcal Z}
\left(
z_{k,t}-\alpha\mathcal F_k(\widehat z_{k,t})
\right).
\end{align*}
The projection optimality condition states that (Lemma~\ref{lem:proj_opt_cond}), for every
$z'\in\mathcal Z$,
\[
\left\langle
z_{k,t}
-
\alpha\mathcal F_k(\widehat z_{k,t})
-
\widehat z_{k,t+1},
z'-\widehat z_{k,t+1}
\right\rangle
\le 0.
\]
Since $\bar z_{k+1}\in\mathcal Z$, choosing
$z'=\bar z_{k+1}$ gives
\[
\left\langle
z_{k,t}
-
\widehat z_{k,t+1}
-
\alpha\mathcal F_k(\widehat z_{k,t}),
\bar z_{k+1}-\widehat z_{k,t+1}
\right\rangle
\le 0.
\]
Equivalently,
\begin{align}
\left\langle
z_{k,t}-\widehat z_{k,t+1},
\widehat z_{k,t+1}-\bar z_{k+1}
\right\rangle \ge
\alpha
\left\langle
\mathcal F_k(\widehat z_{k,t}),
\widehat z_{k,t+1}-\bar z_{k+1}
\right\rangle.
\label{eq:correction-projection-condition}
\end{align}
Next, using
\[
z_{k,t}-\bar z_{k+1}
=
\left(
z_{k,t}-\widehat z_{k,t+1}
\right)
+
\left(
\widehat z_{k,t+1}-\bar z_{k+1}
\right),
\]
we have
\begin{align*}
\|z_{k,t}-\bar z_{k+1}\|^2
=
\|z_{k,t}-\widehat z_{k,t+1}\|^2
+
\|\widehat z_{k,t+1}-\bar z_{k+1}\|^2+2\left\langle
z_{k,t}-\widehat z_{k,t+1},
\widehat z_{k,t+1}-\bar z_{k+1}
\right\rangle.
\end{align*}
Rearranging gives
\begin{align*}
\|\widehat z_{k,t+1}-\bar z_{k+1}\|^2=
\|z_{k,t}-\bar z_{k+1}\|^2
-
\|z_{k,t}-\widehat z_{k,t+1}\|^2-
2
\left\langle
z_{k,t}-\widehat z_{k,t+1},
\widehat z_{k,t+1}-\bar z_{k+1}
\right\rangle.
\end{align*}
Applying~\eqref{eq:correction-projection-condition} to the final
inner-product term yields
\begin{align}
\|\widehat z_{k,t+1}-\bar z_{k+1}\|^2\le
\|z_{k,t}-\bar z_{k+1}\|^2-\|z_{k,t}-\widehat z_{k,t+1}\|^2-
2\alpha
\left\langle
\mathcal F_k(\widehat z_{k,t}),
\widehat z_{k,t+1}-\bar z_{k+1}
\right\rangle.
\label{eq:exact-EG-step-1}
\end{align}
Decomposing
\[
\widehat z_{k,t+1}-\bar z_{k+1}
=
\widehat z_{k,t}-\bar z_{k+1}
+
\widehat z_{k,t+1}-\widehat z_{k,t},
\]
we obtain
\begin{align}
\|\widehat z_{k,t+1}-\bar z_{k+1}\|^2
&\le
\|z_{k,t}-\bar z_{k+1}\|^2
-
\|z_{k,t}-\widehat z_{k,t+1}\|^2
\nonumber\\
&\quad
-
2\alpha
\left\langle
\mathcal F_k(\widehat z_{k,t}),
\widehat z_{k,t}-\bar z_{k+1}
\right\rangle
-
2\alpha
\left\langle
\mathcal F_k(\widehat z_{k,t}),
\widehat z_{k,t+1}-\widehat z_{k,t}
\right\rangle.
\label{eq:exact-EG-step-2}
\end{align}
Similarly, by the projection optimality condition
(Lemma~\ref{lem:proj_opt_cond}), 
\[
\left\langle
z_{k,t}
-
\alpha\mathcal F_k(z_{k,t})
-
\widehat z_{k,t},
z'-\widehat z_{k,t}
\right\rangle
\le 0,
\qquad
\forall z'\in\mathcal Z.
\]
Because the exact correction point
$\widehat z_{k,t+1}$ belongs to $\mathcal Z$, we may choose
$z'=\widehat z_{k,t+1}$. Therefore,
\[
\left\langle
z_{k,t}
-
\alpha\mathcal F_k(z_{k,t})
-
\widehat z_{k,t},
\widehat z_{k,t+1}-\widehat z_{k,t}
\right\rangle
\le 0.
\]
Separating the two terms in the first argument yields
\begin{align*}
\left\langle
z_{k,t}-\widehat z_{k,t},
\widehat z_{k,t+1}-\widehat z_{k,t}
\right\rangle-
\alpha
\left\langle
\mathcal F_k(z_{k,t}),
\widehat z_{k,t+1}-\widehat z_{k,t}
\right\rangle
\le 0.
\end{align*}
Rearranging this inequality gives
\[
\alpha
\left\langle
\mathcal F_k(z_{k,t}),
\widehat z_{k,t+1}-\widehat z_{k,t}
\right\rangle
\ge
\left\langle
z_{k,t}-\widehat z_{k,t},
\widehat z_{k,t+1}-\widehat z_{k,t}
\right\rangle.
\]
Therefore,
\begin{align*}
-2\alpha
\left\langle
\mathcal F_k(\widehat z_{k,t}),
\widehat z_{k,t+1}-\widehat z_{k,t}
\right\rangle&=
-2\alpha
\left\langle
\mathcal F_k(z_{k,t}),
\widehat z_{k,t+1}-\widehat z_{k,t} \right\rangle-2\alpha
\left\langle\mathcal F_k(\widehat z_{k,t})
-
\mathcal F_k(z_{k,t}),
\widehat z_{k,t+1}-\widehat z_{k,t}
\right\rangle\\
&\le
-2
\left\langle
z_{k,t}-\widehat z_{k,t},
\widehat z_{k,t+1}-\widehat z_{k,t}
\right\rangle-
2\alpha
\left\langle
\mathcal F_k(\widehat z_{k,t})-
\mathcal F_k(z_{k,t}),
\widehat z_{k,t+1}-\widehat z_{k,t}
\right\rangle.
\end{align*}
Moreover,
\[
z_{k,t}-\widehat z_{k,t+1}
=
(z_{k,t}-\widehat z_{k,t})
-
(\widehat z_{k,t+1}-\widehat z_{k,t}),
\]
and hence
\begin{align*}
-\|z_{k,t}-\widehat z_{k,t+1}\|^2
-
2
\left\langle
z_{k,t}-\widehat z_{k,t},
\widehat z_{k,t+1}-\widehat z_{k,t}
\right\rangle=
-\|z_{k,t}-\widehat z_{k,t}\|^2
-
\|\widehat z_{k,t+1}-\widehat z_{k,t}\|^2.
\end{align*}
Substituting these relations into
\eqref{eq:exact-EG-step-2} gives
\begin{align}
\|\widehat z_{k,t+1}-\bar z_{k+1}\|^2
&\le
\|z_{k,t}-\bar z_{k+1}\|^2
-
\|z_{k,t}-\widehat z_{k,t}\|^2
-
\|\widehat z_{k,t+1}-\widehat z_{k,t}\|^2
\nonumber\\
&\quad
-
2\alpha
\left\langle
\mathcal F_k(\widehat z_{k,t}),
\widehat z_{k,t}-\bar z_{k+1}
\right\rangle
-
2\alpha
\left\langle
\mathcal F_k(\widehat z_{k,t})
-
\mathcal F_k(z_{k,t}),
\widehat z_{k,t+1}-\widehat z_{k,t}
\right\rangle.
\label{eq:exact-EG-step-3}
\end{align}
Since $\mathcal F_k$ is $L_k$-Lipschitz continuous (Lemma~\ref{lem:saddle-operator-regularity}),
\begin{align*}
-2\alpha
\left\langle
\mathcal F_k(\widehat z_{k,t})
-
\mathcal F_k(z_{k,t}),
\widehat z_{k,t+1}-\widehat z_{k,t}
\right\rangle
&\le
2\alpha L_k
\|z_{k,t}-\widehat z_{k,t}\|
\|\widehat z_{k,t+1}-\widehat z_{k,t}\|\\
&\le
\alpha^2L_k^2
\|z_{k,t}-\widehat z_{k,t}\|^2
+
\|\widehat z_{k,t+1}-\widehat z_{k,t}\|^2.
\end{align*}
Consequently,
\begin{align}
\|\widehat z_{k,t+1}-\bar z_{k+1}\|^2\le
\|z_{k,t}-\bar z_{k+1}\|^2
-
(1-\alpha^2L_k^2)
\|z_{k,t}-\widehat z_{k,t}\|^2
-
2\alpha
\left\langle
\mathcal F_k(\widehat z_{k,t}),
\widehat z_{k,t}-\bar z_{k+1}
\right\rangle.
\label{eq:exact-EG-step-4}
\end{align}
Because $\bar z_{k+1}$ solves the variational inequality associated
with $\mathcal F_k$ in~\eqref{eq:kth-prox-VI}, we have
\[
\left\langle
\mathcal F_k(\bar z_{k+1}),
\widehat z_{k,t}-\bar z_{k+1}
\right\rangle
\ge 0.
\]
Together with the $\mu$-strong monotonicity of $\mathcal F_k$ (Lemma~\ref{lem:saddle-operator-regularity}), this implies
\begin{align*}
\left\langle
\mathcal F_k(\widehat z_{k,t}),
\widehat z_{k,t}-\bar z_{k+1}
\right\rangle
&=
\left\langle
\mathcal F_k(\widehat z_{k,t})
-
\mathcal F_k(\bar z_{k+1}),
\widehat z_{k,t}-\bar z_{k+1}
\right\rangle+
\left\langle
\mathcal F_k(\bar z_{k+1}),
\widehat z_{k,t}-\bar z_{k+1}
\right\rangle\\
&\ge
\mu
\|\widehat z_{k,t}-\bar z_{k+1}\|^2.
\end{align*}
Hence,
\begin{align}
\|\widehat z_{k,t+1}-\bar z_{k+1}\|^2\le
\|z_{k,t}-\bar z_{k+1}\|^2
-
(1-\alpha^2L_k^2)
\|z_{k,t}-\widehat z_{k,t}\|^2
-
2\alpha\mu
\|\widehat z_{k,t}-\bar z_{k+1}\|^2.
\label{eq:exact-EG-step-5}
\end{align}
Since $\alpha\le 1/(4L_k)$ and $\mu\le L_k$, \(1-\alpha^2L_k^2\ge
\frac{15}{16}
\ge
2\alpha\mu.\) Therefore,
\begin{align*}
(1-\alpha^2L_k^2)
\|z_{k,t}-\widehat z_{k,t}\|^2
+
2\alpha\mu
\|\widehat z_{k,t}-\bar z_{k+1}\|^2
&\ge
2\alpha\mu
\left(
\|z_{k,t}-\widehat z_{k,t}\|^2
+
\|\widehat z_{k,t}-\bar z_{k+1}\|^2
\right)\\
&\ge
\alpha\mu
\|z_{k,t}-\bar z_{k+1}\|^2,
\end{align*}
where the last inequality follows from
\[
\|z_{k,t}-\bar z_{k+1}\|^2
\le
2\|z_{k,t}-\widehat z_{k,t}\|^2
+
2\|\widehat z_{k,t}-\bar z_{k+1}\|^2.
\]
Substituting this bound into \eqref{eq:exact-EG-step-5} gives
\begin{equation}
\label{eq:exact-EG-contraction}
\|\widehat z_{k,t+1}-\bar z_{k+1}\|^2
\le
(1-\alpha\mu)
\|z_{k,t}-\bar z_{k+1}\|^2.
\end{equation}
We next quantify the effect of the stochastic errors. Recall again that the stochastic prediction point and its exact counterpart are defined by
\begin{align*}
\widetilde z_{k,t}
&=
\operatorname{proj}_{\mathcal Z}
\left(
z_{k,t}
-
\alpha
\left[
\mathcal F_k(z_{k,t})
+
\varepsilon_{k,t}
\right]
\right),\\
\widehat z_{k,t}
&=
\operatorname{proj}_{\mathcal Z}
\left(
z_{k,t}
-
\alpha\mathcal F_k(z_{k,t})
\right).
\end{align*}
By the nonexpansiveness of the projection operator
(Lemma~\ref{lem:proj_nonexpansive}),
\begin{align}
\label{eq:prediction-perturbation}
\|\widetilde z_{k,t}-\widehat z_{k,t}\|\le
\Bigl\|
z_{k,t}
-
\alpha
\left[
\mathcal F_k(z_{k,t})
+
\varepsilon_{k,t}
\right]
-
\left(
z_{k,t}
-
\alpha\mathcal F_k(z_{k,t})
\right)
\Bigr\|=
\left\|
-\alpha\varepsilon_{k,t}
\right\|=
\alpha\|\varepsilon_{k,t}\|.
\end{align}
Similarly, recall that the stochastic correction point and its exact counterpart
are
\begin{align*}
z_{k,t+1}
&=
\operatorname{proj}_{\mathcal Z}
\left(
z_{k,t}
-
\alpha
\left[
\mathcal F_k(\widetilde z_{k,t})
+
\varepsilon_{k,t}'
\right]
\right),\\
\widehat z_{k,t+1}
&=
\operatorname{proj}_{\mathcal Z}
\left(
z_{k,t}
-
\alpha\mathcal F_k(\widehat z_{k,t})
\right).
\end{align*}
Applying Lemma~\ref{lem:proj_nonexpansive} again gives
\begin{align*}
\|z_{k,t+1}-\widehat z_{k,t+1}\|
&\le
\Bigl\|
z_{k,t}
-
\alpha
\left[
\mathcal F_k(\widetilde z_{k,t})
+
\varepsilon_{k,t}'
\right]-
\left(
z_{k,t}
-
\alpha\mathcal F_k(\widehat z_{k,t})
\right)
\Bigr\|\\
&=
\alpha
\left\|
\mathcal F_k(\widetilde z_{k,t})
-
\mathcal F_k(\widehat z_{k,t})
+
\varepsilon_{k,t}'
\right\|.
\end{align*}
Using the triangle inequality, we obtain
\begin{align*}
\|z_{k,t+1}-\widehat z_{k,t+1}\|
&\le
\alpha
\left\|
\mathcal F_k(\widetilde z_{k,t})
-
\mathcal F_k(\widehat z_{k,t})
\right\|
+
\alpha\|\varepsilon_{k,t}'\|.
\end{align*}
Since $\mathcal F_k$ is $L_k$-Lipschitz continuous (Lemma~\ref{lem:saddle-operator-regularity}),
\begin{align*}
\|z_{k,t+1}-\widehat z_{k,t+1}\|
&\le
\alpha L_k
\|\widetilde z_{k,t}-\widehat z_{k,t}\|
+
\alpha\|\varepsilon_{k,t}'\|.
\end{align*}
Finally, substituting
\eqref{eq:prediction-perturbation} yields
\begin{equation}
\label{eq:correction-perturbation}
\|z_{k,t+1}-\widehat z_{k,t+1}\|
\le
\alpha^2L_k\|\varepsilon_{k,t}\|
+
\alpha\|\varepsilon_{k,t}'\|.
\end{equation}
Using $(a+b)^2\le 2a^2+2b^2$ yields
\begin{equation}
\label{eq:correction-perturbation-squared}
\|z_{k,t+1}-\widehat z_{k,t+1}\|^2
\le
2\alpha^4L_k^2\|\varepsilon_{k,t}\|^2
+
2\alpha^2\|\varepsilon_{k,t}'\|^2.
\end{equation}
For any $\tau>0$, Young's inequality~\citep{alzer2019young} shows that
\(\|a+b\|^2\le
(1+\tau)\|a\|^2+
\left(1+\frac{1}{\tau}\right)\|b\|^2.\)
Applying this inequality to
\[
z_{k,t+1}-\bar z_{k+1}
=
\left(
\widehat z_{k,t+1}-\bar z_{k+1}
\right)
+
\left(
z_{k,t+1}-\widehat z_{k,t+1}
\right)
\]
with \(\tau=\frac{\alpha\mu}{2(1-\alpha\mu)}\) gives
\begin{align*}
\|z_{k,t+1}-\bar z_{k+1}\|^2 &\le (1+\tau)\|\widehat z_{k,t+1}-\bar z_{k+1}\|^2+\left(1+\frac{1}{\tau}\right)\|z_{k,t+1}-\widehat z_{k,t+1}\|^2 \\
&\le (1+\tau)(1-\alpha\mu)
\|z_{k,t}-\bar z_{k+1}\|^2+\left(1+\frac{1}{\tau}\right) (2\alpha^4L_k^2\|\varepsilon_{k,t}\|^2
+
2\alpha^2\|\varepsilon_{k,t}'\|^2)
\end{align*}
The last inequality combines~\eqref{eq:exact-EG-contraction} and~\eqref{eq:correction-perturbation-squared}. Notice that
\( (1+\tau)(1-\alpha\mu)=1-\frac{\alpha\mu}{2}, 1+\frac{1}{\tau} \le \frac{2}{\alpha\mu},\) we obtain
\begin{align*}
\|z_{k,t+1}-\bar z_{k+1}\|^2\le
\left(
1-\frac{\alpha\mu}{2}
\right)
\|z_{k,t}-\bar z_{k+1}\|^2+
\frac{4\alpha^3L_k^2}{\mu}
\|\varepsilon_{k,t}\|^2+
\frac{4\alpha}{\mu}
\|\varepsilon_{k,t}'\|^2.
\end{align*}
Since $\alpha L_k\le 1/4$, we have \(\frac{4\alpha^3L_k^2}{\mu} \le \frac{\alpha}{4\mu} \le \frac{4\alpha}{\mu}. \) Therefore,
\begin{align*}
\|z_{k,t+1}-\bar z_{k+1}\|^2\le
\left(
1-\frac{\alpha\mu}{2}
\right)
\|z_{k,t}-\bar z_{k+1}\|^2+\frac{4\alpha}{\mu}
\left(
\|\varepsilon_{k,t}\|^2+
\|\varepsilon_{k,t}'\|^2
\right).
\end{align*}
Let $\varepsilon_{k,t}^{\theta,\mathrm{pred}}$ and $\varepsilon_{k,t}^{\xi,\mathrm{pred}}$ denote the stochastic gradient-estimation errors in the prediction step (Algorithm~\ref{alg:pe-pgda}, Line 6) for $\theta$ and $\xi$ respectively, and let $\varepsilon_{k,t}^{\theta,\mathrm{corr}}$ and $\varepsilon_{k,t}^{\xi,\mathrm{corr}}$ denote the corresponding errors in the correction step (Algorithm~\ref{alg:pe-pgda}, Line 9). According to Proposition~\ref{prop:error_grad_estimate}, we have 
\begin{align*} \mathbb E \left[ \left\| \varepsilon_{k,t}^{\theta,\mathrm{pred}} \right\|^2 \right] \le E_{\theta,k}^{\mathrm{pred}},  \mathbb E \left[ \left\| \varepsilon_{k,t}^{\xi,\mathrm{pred}} \right\|^2 \right] \le E_{\xi,k}^{\mathrm{pred}}, \mathbb E \left[ \left\| \varepsilon_{k,t}^{\theta,\mathrm{corr}} \right\|^2 \right] \le E_{\theta,k}^{\mathrm{corr}},  \mathbb E \left[ \left\| \varepsilon_{k,t}^{\xi,\mathrm{corr}} \right\|^2 \right] \le E_{\xi,k}^{\mathrm{corr}}. 
\end{align*} 
We have
\[
\mathbb{E}\!\left[\|\varepsilon_{k,t}\|^2\right]
=
\mathbb{E}\!\left[
\|\varepsilon_{k,t}^{\theta,\mathrm{pred}}\|^2
+
\|\varepsilon_{k,t}^{\xi,\mathrm{pred}}\|^2
\right]
\le
E_{\theta,k}^{\mathrm{pred}}
+
E_{\xi,k}^{\mathrm{pred}},
\]
\[
\mathbb{E}\!\left[\|\varepsilon_{k,t}'\|^2\right]
=
\mathbb{E}\!\left[
\|\varepsilon_{k,t}^{\theta,\mathrm{corr}}\|^2
+
\|\varepsilon_{k,t}^{\xi,\mathrm{corr}}\|^2
\right]
\le
E_{\theta,k}^{\mathrm{corr}}
+
E_{\xi,k}^{\mathrm{corr}},
\]
In particular, if the
prediction and correction steps use the same Monte-Carlo budgets
$m_k,H_k,m_k',H_k'$, then both steps satisfy the same error
bounds, and hence
\[
E_{\theta,k}^{\mathrm{pred}}
=
E_{\theta,k}^{\mathrm{corr}}
\le
E_\theta(m_k,H_k,m_k',H_k'),
\]
and
\[
E_{\xi,k}^{\mathrm{pred}}
=
E_{\xi,k}^{\mathrm{corr}}
\le
E_\xi(m_k,H_k).
\]
Therefore,
\[
\mathbb{E}\!\left[\|\varepsilon_{k,t}\|^2\right]+\mathbb{E}\!\left[\|\varepsilon_{k,t}'\|^2\right] \le
2E_\theta(m_k,H_k,m_k',H_k')
+2E_\xi(m_k,H_k) =\mathcal E_k .
\]
Thus, taking the conditional expectation on the history of the inner loop gives
\begin{align}
\mathbb E
\left[
\|z_{k,t+1}-\bar z_{k+1}\|^2
\mid
\mathcal H_{k,t}
\right]\le
\left(
1-\frac{\alpha\mu}{2}
\right)
\|z_{k,t}-\bar z_{k+1}\|^2
+
\frac{4\alpha}{\mu}\mathcal E_k,
\label{eq:inner-distance-one-step}
\end{align}
where $\mathcal H_{k,t}$ contains all randomness generated before
the two oracle calls for the gradient estimation at inner iteration $t$. We apply the tower property of conditional expectation. For
$t=0$, since $z_{k,0}$ is fixed given the outer-loop history,
\begin{align*}
\mathbb E_k
\|z_{k,1}-\bar z_{k+1}\|^2
&\le
\left(
1-\frac{\alpha\mu}{2}
\right)
\|z_{k,0}-\bar z_{k+1}\|^2
+
\frac{4\alpha}{\mu}\mathcal E_k.
\end{align*}
Applying the same recursion once more gives
\begin{align*}
\mathbb E_k
\|z_{k,2}-\bar z_{k+1}\|^2
&\le
\left(
1-\frac{\alpha\mu}{2}
\right)
\mathbb E_k
\|z_{k,1}-\bar z_{k+1}\|^2
+
\frac{4\alpha}{\mu}\mathcal E_k\\
&\le
\left(
1-\frac{\alpha\mu}{2}
\right)^2
\|z_{k,0}-\bar z_{k+1}\|^2+
\frac{4\alpha}{\mu}\mathcal E_k
\left[
1+
\left(
1-\frac{\alpha\mu}{2}
\right)
\right].
\end{align*}
Similarly, after three inner iterations,
\begin{align*}
\mathbb E_k
\|z_{k,3}-\bar z_{k+1}\|^2
&\le
\left(
1-\frac{\alpha\mu}{2}
\right)^3
\|z_{k,0}-\bar z_{k+1}\|^2+
\frac{4\alpha}{\mu}\mathcal E_k
\left[
1+
\left(
1-\frac{\alpha\mu}{2}
\right)
+
\left(
1-\frac{\alpha\mu}{2}
\right)^2
\right].
\end{align*}
Thus, repeatedly applying
\eqref{eq:inner-distance-one-step} for
$t=0,\ldots,T_k-1$ yields
\begin{align*}
\mathbb E_k
\|z_{k,T_k}-\bar z_{k+1}\|^2
&\le
\left(
1-\frac{\alpha\mu}{2}
\right)^{T_k}
\|z_{k,0}-\bar z_{k+1}\|^2+
\frac{4\alpha}{\mu}\mathcal E_k
\sum_{j=0}^{T_k-1}
\left(
1-\frac{\alpha\mu}{2}
\right)^j.
\end{align*}
Since \(0<1-\frac{\alpha\mu}{2}<1,
\) the finite geometric-series formula gives
\begin{align*}
\sum_{j=0}^{T_k-1}
\left(
1-\frac{\alpha\mu}{2}
\right)^j=
\frac{
1-
\left(
1-\frac{\alpha\mu}{2}
\right)^{T_k}
}{
1-
\left(
1-\frac{\alpha\mu}{2}
\right)
}=
\frac{
1-
\left(
1-\frac{\alpha\mu}{2}
\right)^{T_k}
}{
\alpha\mu/2
}\le
\frac{1}{\alpha\mu/2}
=
\frac{2}{\alpha\mu}.
\end{align*}
Moreover, because
$z_{k,0}=z_k$ and both $z_k$ and
$\bar z_{k+1}$ belong to $\mathcal Z$,
\[
\|z_{k,0}-\bar z_{k+1}\|^2
\le
D_{\mathcal Z}^2.
\]
Therefore,
\begin{align*}
\mathbb E_k
\|z_{k,T_k}-\bar z_{k+1}\|^2
&\le
\left(
1-\frac{\alpha\mu}{2}
\right)^{T_k}
D_{\mathcal Z}^2+
\frac{4\alpha}{\mu}\mathcal E_k
\frac{2}{\alpha\mu}\\
&=
\left(
1-\frac{\alpha\mu}{2}
\right)^{T_k}
D_{\mathcal Z}^2
+
\frac{8\mathcal E_k}{\mu^2}.
\end{align*}
Since $z_{k,T_k}=z_{k+1}$, this proves
\eqref{eq:inner-distance-explicit}.
\end{proof}

\begin{proposition}
\label{prop:inner-solver-complexity}
Suppose Assumptions~\ref{ass:xi-convex-compact}to~\ref{ass:xi_lambda_bound} hold and let \(\mu:=\sigma-L_{\mathcal F}>0, L_k:=L_{\mathcal F}+\sigma.\) Define the total stochastic error at outer iteration $k$ by \( \mathcal E_k:= 2 (E_\theta(m_k,H_k,m_k',H_k') + E_{\xi}(m_k, H_k))\), where $E_{\theta}$ and $E_{\xi}$ denote the stochastic gradient-estimation errors as in Proposition~\ref{prop:error_grad_estimate} and \(m_k,H_k,m_k',H_k'\) denote the associated Monte Carlo budgets. Choose the step size $\alpha$ such that \(0<\alpha\le\frac{1}{4L_k}.\) Furthermore, define \(\ell_{\mathcal F}:=
\sqrt{\ell_\theta^2+\ell_\xi^2}, B_{\rm gap}:=
\ell_{\mathcal F}+(L_{\mathcal F}+2\sigma)D_{\mathcal Z}.\)
Then
\begin{equation}
\label{eq:inner-gap-explicit}
\mathbb E_k
\left[
\operatorname{Gap}_k(z_{k+1})
\right]
\le
B_{\rm gap}
\sqrt{
\left(
1-\frac{\alpha\mu}{2}
\right)^{T_k}
D_{\mathcal Z}^2
+
\frac{8\mathcal E_k}{\mu^2}
}.
\end{equation}
\end{proposition}

\begin{proof}
Fix any $z\in\mathcal Z$ and let $\bar z_{k+1}$ be the exact solution of the $k$-th proximal variational inequality in~\eqref{eq:kth-prox-VI}. For every $z'\in\mathcal Z$,
\begin{align*}
\left\langle
\mathcal F_k(z),z-z'
\right\rangle=
\left\langle
\mathcal F_k(z)-\mathcal F_k(\bar z_{k+1}),
z-z'
\right\rangle+
\left\langle
\mathcal F_k(\bar z_{k+1}),
z-\bar z_{k+1}
\right\rangle+
\left\langle
\mathcal F_k(\bar z_{k+1}),
\bar z_{k+1}-z'
\right\rangle.
\end{align*}
Since $\bar z_{k+1}$ solves the proximal variational inequality,
\[
\left\langle
\mathcal F_k(\bar z_{k+1}),
z'-\bar z_{k+1}
\right\rangle
\ge 0,
\]
and hence the last term is nonpositive. Therefore,
\begin{align*}
\left\langle
\mathcal F_k(z),z-z'
\right\rangle\le
\left\|
\mathcal F_k(z)-\mathcal F_k(\bar z_{k+1})
\right\|
\|z-z'\|+
\left\|
\mathcal F_k(\bar z_{k+1})
\right\|
\|z-\bar z_{k+1}\|.
\end{align*}
Because $\mathcal F_k$ is
$(L_{\mathcal F}+\sigma)$-Lipschitz and
$\|z-z'\|\le D_{\mathcal Z}$,
\[
\left\|
\mathcal F_k(z)-\mathcal F_k(\bar z_{k+1})
\right\|
\|z-z'\|
\le
(L_{\mathcal F}+\sigma)D_{\mathcal Z}
\|z-\bar z_{k+1}\|.
\]
Moreover,
\begin{align*}
\|\mathcal F_k(\bar z_{k+1})\|\le
\|\mathcal F(\bar z_{k+1})\|
+
\sigma\|\bar z_{k+1}-z_k\|\le
\ell_{\mathcal F}+\sigma D_{\mathcal Z}.
\end{align*}
Combining the preceding bounds, for every $z'\in\mathcal Z$ we
have
\begin{align*}
\left\langle
\mathcal F_k(z),z-z'
\right\rangle
&\le
\left\|
\mathcal F_k(z)-\mathcal F_k(\bar z_{k+1})
\right\|
\|z-z'\|+
\left\|
\mathcal F_k(\bar z_{k+1})
\right\|
\|z-\bar z_{k+1}\|\\
&\le
(L_{\mathcal F}+\sigma)D_{\mathcal Z}
\|z-\bar z_{k+1}\|+
\left(
\ell_{\mathcal F}
+
\sigma D_{\mathcal Z}
\right)
\|z-\bar z_{k+1}\|.
\end{align*}
Collecting the two coefficients gives
\begin{align*}
\left\langle
\mathcal F_k(z),z-z'
\right\rangle
&\le
\left[
(L_{\mathcal F}+\sigma)D_{\mathcal Z}
+
\ell_{\mathcal F}
+
\sigma D_{\mathcal Z}
\right]
\|z-\bar z_{k+1}\|\\
&=
\left[
\ell_{\mathcal F}
+
(L_{\mathcal F}+2\sigma)D_{\mathcal Z}
\right]
\|z-\bar z_{k+1}\|.
\end{align*}
Define \(
B_{\rm gap}:=\ell_{\mathcal F}
+(L_{\mathcal F}+2\sigma)D_{\mathcal Z}.
\) Then, for every $z'\in\mathcal Z$,
\begin{equation}
\label{eq:gap-pointwise-upper-bound}
\left\langle
\mathcal F_k(z),z-z'
\right\rangle
\le
B_{\rm gap}
\|z-\bar z_{k+1}\|.
\end{equation}
Notice that the right-hand side of
\eqref{eq:gap-pointwise-upper-bound} does not depend on the
comparison point $z'$. Therefore, maximizing the left-hand side over
all $z'\in\mathcal Z$ preserves the same upper bound:
\begin{align*}
\operatorname{Gap}_k(z)=
\max_{z'\in\mathcal Z}
\left\langle
\mathcal F_k(z),z-z'
\right\rangle\le
\max_{z'\in\mathcal Z}
\left\{
B_{\rm gap}
\|z-\bar z_{k+1}\|
\right\}=
B_{\rm gap}
\|z-\bar z_{k+1}\|.
\end{align*}
Evaluating this inequality at $z=z_{k+1}$, taking conditional
expectations, and applying Jensen's inequality yields
\[
\mathbb E_k
\left[
\operatorname{Gap}_k(z_{k+1})
\right]
\le
B_{\rm gap}
\sqrt{
\mathbb E_k
\left[
\|z_{k+1}-\bar z_{k+1}\|^2
\right]
}.
\]
Substituting~\eqref{eq:inner-distance-explicit} in Lemma~\ref{lem:inner-stochastic-EG-convergence} proves
\eqref{eq:inner-gap-explicit}.
\end{proof}

\begin{lemma}
\label{lem:MC_decay}
Let the finite-sample error bounds in
Proposition~\ref{prop:error_grad_estimate} be written as
\begin{align}
E_\theta(m,H,m',H')
&\le
\frac{C_{\theta,m}}{m}
+
C_{\theta,H}\gamma^{2H}
+
\frac{C_{\theta,m'}}{m'}
+
C_{\theta,H'}\gamma^{2H'},
\label{eq:E-theta-collected}\\
E_\xi(m,H)
&\le
\frac{C_{\xi,m}}{m}
+
C_{\xi,H}\gamma^{2H},
\label{eq:E-xi-collected}
\end{align}
where 
\(C_{\theta,m}=
\frac{3\ell_{\pi_\theta}^{2}}
{(1-\gamma)^{4}}
L_\lambda^{2}|\mathcal{S}||\mathcal{A}|,
C_{\theta,H}=
\frac{3\ell_{\pi_\theta}^{2}}
{(1-\gamma)^{4}}
L_\lambda^{2}|\mathcal{S}||\mathcal{A}|,
C_{\theta,m'}=
\frac{3\ell_{\pi_\theta}^{2}\ell_\lambda^{2}}
{(1-\gamma)^{4}}, C_{\theta,H'}=
\frac{3\ell_{\pi_\theta}^{2}\ell_\lambda^{2}}
{(1-\gamma)^{4}}
\bigl(1+H_\theta(1-\gamma)\bigr)^{2},
C_{\xi,m}=
L_{\xi,\lambda}^{2},
C_{\xi,H}=
L_{\xi,\lambda}^{2},\) and all constants are independent of the outer iteration $k$.
Define \(C_m:=C_{\theta,m}+C_{\xi,m},
C_H:=C_{\theta,H}+C_{\xi,H},\)
and \( C_{m'}:=C_{\theta,m'}, C_{H'}:=C_{\theta,H'}.
\) Define the total stochastic error at outer iteration $k$ by \( \mathcal E_k:= 2 (E_\theta(m_k,H_k,m_k',H_k') + E_{\xi}(m_k, H_k))\). Choose the Monte-Carlo sample sizes according to \(
m_k\ge
\left\lceil
\frac{64C_m}{\mu^2c_{\mathrm{MC}}}
(k+1)^2
\right\rceil,
m_k'\ge
\left\lceil
\frac{64C_{m'}}{\mu^2c_{\mathrm{MC}}}
(k+1)^2
\right\rceil,
\) and choose the truncation horizons according to \(
H_k\ge\lceil
\frac{
\left[
\log\left(
\frac{64C_H(k+1)^2}
{\mu^2c_{\mathrm{MC}}}
\right)
\right]_+
}{
2\log(1/\gamma)
}\rceil,
H_k'\ge\lceil
\frac{
\left[
\log\left(
\frac{64C_{H'}(k+1)^2}
{\mu^2c_{\mathrm{MC}}}
\right)
\right]_+
}{
2\log(1/\gamma)
}\rceil.
\)
Then fix any
constant $c_{\mathrm{MC}}>0$, we have 
\[
\mathcal E_k
\le
\frac{\mu^2c_{\mathrm{MC}}}
{8(k+1)^2}.
\]
\end{lemma}

\begin{proof}
Given \(m_k\ge\frac{64C_m}{\mu^2c_{\mathrm{MC}}}(k+1)^2,\) since all quantities are positive, taking reciprocals gives
\[
\frac{1}{m_k}
\le
\frac{\mu^2c_{\mathrm{MC}}}
{64C_m(k+1)^2}.
\]
Multiplying both sides by $C_m$ yields
\begin{equation}
\label{eq:mk-error-control}
\frac{C_m}{m_k}
\le
\frac{\mu^2c_{\mathrm{MC}}}
{64(k+1)^2}.
\end{equation}
Similarly, the choice of $m_k'$ gives
\begin{equation}
\label{eq:mk-prime-error-control}
\frac{C_{m'}}{m_k'}
\le
\frac{\mu^2c_{\mathrm{MC}}}
{64(k+1)^2}.
\end{equation}
We now consider the truncation error associated with $H_k$. Baed on the property
$\lceil a\rceil\ge a$, we have
\[
H_k
\ge
\frac{
\left[
\log\left(
\frac{64C_H(k+1)^2}
{\mu^2c_{\mathrm{MC}}}
\right)
\right]_+
}{
2\log(1/\gamma)
}.
\]
Since $0<\gamma<1$, we have $\log(1/\gamma)>0$. Multiplying both
sides by $2\log(1/\gamma)$ gives
\[
2H_k\log(1/\gamma)
\ge
\left[
\log\left(
\frac{64C_H(k+1)^2}
{\mu^2c_{\mathrm{MC}}}
\right)
\right]_+.
\]
Because the positive-part operation satisfies $[a]_+\ge a$,
\[
2H_k\log(1/\gamma)
\ge
\log\left(
\frac{64C_H(k+1)^2}
{\mu^2c_{\mathrm{MC}}}
\right).
\]
Multiplying by $-1$ reverses the inequality:
\[
-2H_k\log(1/\gamma)
\le
-\log\left(
\frac{64C_H(k+1)^2}
{\mu^2c_{\mathrm{MC}}}
\right).
\]
Since
\[
-2H_k\log(1/\gamma)
=
2H_k\log\gamma
=
\log(\gamma^{2H_k}),
\]
we obtain
\[
\log(\gamma^{2H_k})
\le
\log\left(
\frac{\mu^2c_{\mathrm{MC}}}
{64C_H(k+1)^2}
\right).
\]
The exponential function is increasing, so
\[
\gamma^{2H_k}
\le
\frac{\mu^2c_{\mathrm{MC}}}
{64C_H(k+1)^2}.
\]
Multiplying by $C_H$ gives
\begin{equation}
\label{eq:Hk-error-control}
C_H\gamma^{2H_k}
\le
\frac{\mu^2c_{\mathrm{MC}}}
{64(k+1)^2}.
\end{equation}
The same argument applied to $H_k'$ gives
\begin{equation}
\label{eq:Hk-prime-error-control}
C_{H'}\gamma^{2H_k'}
\le
\frac{\mu^2c_{\mathrm{MC}}}
{64(k+1)^2}.
\end{equation}

Substituting~\eqref{eq:mk-error-control},~\eqref{eq:mk-prime-error-control},~\eqref{eq:Hk-error-control}, and~\eqref{eq:Hk-prime-error-control} into the definition of $\mathcal E_k$ yields
\begin{align*}
\mathcal E_k
&\le
2\left(
\frac{\mu^2c_{\mathrm{MC}}}{64(k+1)^2}
+
\frac{\mu^2c_{\mathrm{MC}}}{64(k+1)^2}
+
\frac{\mu^2c_{\mathrm{MC}}}{64(k+1)^2}
+
\frac{\mu^2c_{\mathrm{MC}}}{64(k+1)^2}
\right)\\
&=
2\left(
\frac{4\mu^2c_{\mathrm{MC}}}
{64(k+1)^2}
\right)\\
&=
\frac{\mu^2c_{\mathrm{MC}}}
{8(k+1)^2}.
\end{align*}
This completes the proof.
\end{proof}

\begin{corollary}
\label{cor:inner-gap-decay}
Suppose the conditions of
Lemma~\ref{lem:inner-stochastic-EG-convergence} hold. Fix any
constants $c_{\mathrm{opt}}>0$ and $c_{\mathrm{MC}}>0$. At every
outer iteration $k$, choose the number of inner iterations $T_k$
such that
\[
T_k
\ge
\max
\left\{
1,\,
\left\lceil
\frac{2}{\alpha\mu}
\left(
2\log(k+1)
+
\left[
\log\left(
\frac{D_{\mathcal Z}^2}{c_{\mathrm{opt}}}
\right)
\right]_+
\right)
\right\rceil
\right\},
\]
where $[a]_+:=\max\{a,0\}$. In addition, choose the Monte-Carlo
budgets $m_k,H_k,m_k',H_k'$ as in Lemma~\ref{lem:MC_decay} such that \(\mathcal E_k\le\frac{\mu^2c_{\mathrm{MC}}}{8(k+1)^2}.\)
Then
\begin{equation}
\label{eq:inner-gap-rate}
\mathbb E_k
\left[
\operatorname{Gap}_k(z_{k+1})
\right]
\le
\frac{c_\delta}{k+1},
\end{equation}
where \(c_\delta:=B_{\rm gap}
\sqrt{c_{\mathrm{opt}}+c_{\mathrm{MC}}}.\)
\end{corollary}

\begin{proof}
From Lemma~\ref{lem:inner-stochastic-EG-convergence}, we have
\begin{align*}
\mathbb E_k
\left[
\operatorname{Gap}_k(z_{k+1})
\right]
\le
B_{\rm gap}
\sqrt{
\left(
1-\frac{\alpha\mu}{2}
\right)^{T_k}
D_{\mathcal Z}^2
+
\frac{8\mathcal E_k}{\mu^2}
}.
\end{align*}
We first control the optimization error term (first term). Since \(0<\frac{\alpha\mu}{2}<1,\)
the inequality $1-x\le \exp(-x)$ gives \( \left(
1-\frac{\alpha\mu}{2}
\right)^{T_k}
\le
\exp\left(
-\frac{\alpha\mu T_k}{2}
\right).
\)
The choice of $T_k$ implies
\[
T_k
\ge
\left\lceil
\frac{2}{\alpha\mu}
\left(
2\log(k+1)
+
\left[
\log\left(
\frac{D_{\mathcal Z}^2}{c_{\mathrm{opt}}}
\right)
\right]_+
\right)
\right\rceil.
\]
This is because the maximum is no smaller than its second
argument,
For every real number $a$, the ceiling satisfies
$\lceil a\rceil\ge a$. Therefore,
\[
T_k
\ge
\frac{2}{\alpha\mu}
\left(
2\log(k+1)
+
\left[
\log\left(
\frac{D_{\mathcal Z}^2}{c_{\mathrm{opt}}}
\right)
\right]_+
\right).
\]
Because $\alpha>0$ and $\mu>0$, multiplying both sides by the
positive quantity $\alpha\mu/2$ preserves the direction of the
inequality and gives
\begin{align*}
\frac{\alpha\mu T_k}{2}\ge
\frac{\alpha\mu}{2}
\cdot
\frac{2}{\alpha\mu}
\left(
2\log(k+1)
+
\left[
\log\left(
\frac{D_{\mathcal Z}^2}{c_{\mathrm{opt}}}
\right)
\right]_+
\right)=
2\log(k+1)
+
\left[
\log\left(
\frac{D_{\mathcal Z}^2}{c_{\mathrm{opt}}}
\right)
\right]_+.
\end{align*}
Therefore,
\begin{align*}
\left(
1-\frac{\alpha\mu}{2}
\right)^{T_k}
D_{\mathcal Z}^2\le
D_{\mathcal Z}^2
\exp\left(
-\frac{\alpha\mu T_k}{2}
\right)\le
\frac{D_{\mathcal Z}^2}{(k+1)^2}
\exp\left(
-
\left[
\log\left(
\frac{D_{\mathcal Z}^2}{c_{\mathrm{opt}}}
\right)
\right]_+
\right).
\end{align*}
We now consider the two possible cases. If $D_{\mathcal Z}^2\le c_{\mathrm{opt}}$, then
\(\left[\log\left(\frac{D_{\mathcal Z}^2}{c_{\mathrm{opt}}}\right)\right]_+=0,\) and hence
\[
D_{\mathcal Z}^2
\exp\left(
-
\left[
\log\left(
\frac{D_{\mathcal Z}^2}{c_{\mathrm{opt}}}
\right)
\right]_+
\right)
=
D_{\mathcal Z}^2
\le
c_{\mathrm{opt}}.
\]
If $D_{\mathcal Z}^2>c_{\mathrm{opt}}$, then \(
\left[
\log\left(
\frac{D_{\mathcal Z}^2}{c_{\mathrm{opt}}}
\right)
\right]_+
=
\log\left(
\frac{D_{\mathcal Z}^2}{c_{\mathrm{opt}}}
\right), \) and therefore
\begin{align*}
D_{\mathcal Z}^2
\exp\left(
-
\left[
\log\left(
\frac{D_{\mathcal Z}^2}{c_{\mathrm{opt}}}
\right)
\right]_+
\right)=
D_{\mathcal Z}^2
\exp\left(
-\log\left(
\frac{D_{\mathcal Z}^2}{c_{\mathrm{opt}}}
\right)
\right)=
c_{\mathrm{opt}}.
\end{align*}
Thus, in both cases,
\begin{equation}
\label{eq:optimization-error-rate}
\left(
1-\frac{\alpha\mu}{2}
\right)^{T_k}
D_{\mathcal Z}^2
\le
\frac{c_{\mathrm{opt}}}{(k+1)^2}.
\end{equation}

Next, Lemma~\ref{lem:MC_decay} gives
\begin{equation}
\label{eq:MC-error-rate}
\frac{8\mathcal E_k}{\mu^2}
\le
\frac{c_{\mathrm{MC}}}{(k+1)^2}.
\end{equation}
Substituting
\eqref{eq:optimization-error-rate} and
\eqref{eq:MC-error-rate} yields
\begin{align*}
\mathbb E_k
\left[
\operatorname{Gap}_k(z_{k+1})
\right]
&\le
B_{\rm gap}
\sqrt{
\frac{c_{\mathrm{opt}}}{(k+1)^2}
+
\frac{c_{\mathrm{MC}}}{(k+1)^2}
}\\
&=
\frac{
B_{\rm gap}
\sqrt{
c_{\mathrm{opt}}+c_{\mathrm{MC}}
}
}{k+1}\\
&=
\frac{c_\delta}{k+1},
\end{align*} 
This proves
\eqref{eq:inner-gap-rate}.
\end{proof}

\subsection{Proof of Lemma~\ref{lem:proximal-gap-mapping}}
\begin{proof}
For any $z=(\theta,\xi)\in\mathcal{Z}$, let
\[
\theta^+
:=
\operatorname{proj}_{\Theta}
\left(
\theta
-
\alpha
\left[
\nabla_\theta f_\xi(\lambda_\theta)
+
\sigma(\theta-\theta_k)
\right]
\right).
\]

By the projection optimality condition
(Lemma~\ref{lem:proj_opt_cond}),
\[
\left\langle
\nabla_\theta f_\xi(\lambda_\theta)
+
\sigma(\theta-\theta_k),
\theta-\theta^+
\right\rangle
\ge
\frac{1}{\alpha}\|\theta-\theta^+\|^2.
\]
Since $\theta^+\in\Theta$, the point
$z'=(\theta^+,\xi)$ belongs to
$\mathcal{Z}=\Theta\times\Xi$. Therefore, by the definition of
$\operatorname{Gap}_k(z)$,
\begin{align*}
\operatorname{Gap}_k(z)=
\max_{z'\in\mathcal{Z}}
\left\langle
\mathcal{F}_k(z),z-z'
\right\rangle\ge
\left\langle
\mathcal{F}_k(z),
(\theta,\xi)-(\theta^+,\xi)
\right\rangle.
\end{align*}
Recalling that
\[
\mathcal{F}_k(z)
=
\begin{pmatrix}
\nabla_\theta f_\xi(\lambda_\theta)
+
\sigma(\theta-\theta_k)\\[1mm]
-\nabla_\xi f_\xi(\lambda_\theta)
+
\sigma(\xi-\xi_k)
\end{pmatrix},
\]
we have
\begin{align*}
\left\langle
\mathcal{F}_k(z),
(\theta,\xi)-(\theta^+,\xi)
\right\rangle
&=
\left\langle
\nabla_\theta f_\xi(\lambda_\theta)
+
\sigma(\theta-\theta_k),
\theta-\theta^+
\right\rangle+
\left\langle
-\nabla_\xi f_\xi(\lambda_\theta)
+
\sigma(\xi-\xi_k),
\xi-\xi
\right\rangle\\
&=
\left\langle
\nabla_\theta f_\xi(\lambda_\theta)
+
\sigma(\theta-\theta_k),
\theta-\theta^+
\right\rangle\ge
\frac{1}{\alpha}\|\theta-\theta^+\|^2.
\end{align*}
Moreover, by the definition of
$\mathcal{G}_{\Theta,\sigma}^{k}=\frac{1}{\alpha}
\left(
\theta
-
\operatorname{proj}_{\Theta}
\left(
\theta
-
\alpha
\left[
\nabla_\theta f_\xi(\lambda_\theta)
+
\sigma(\theta-\theta_k)
\right]
\right)
\right)$,
\[
\mathcal{G}_{\Theta,\sigma}^{k}(\theta,\xi)
=
\frac{1}{\alpha}(\theta-\theta^+),
\]
and hence
\[
\frac{1}{\alpha}\|\theta-\theta^+\|^2
=
\alpha
\left\|
\mathcal{G}_{\Theta,\sigma}^{k}(\theta,\xi)
\right\|^2.
\]
Combining the above inequalities yields
\[
\operatorname{Gap}_k(z)
\ge
\alpha
\left\|
\mathcal{G}_{\Theta,\sigma}^{k}(\theta,\xi)
\right\|^2.
\]

Similarly, let
\[
\xi^+
:=
\operatorname{proj}_{\Xi}
\left(
\xi
+
\alpha
\left[
\nabla_\xi f_\xi(\lambda_\theta)
-
\sigma(\xi-\xi_k)
\right]
\right).
\]
By Lemma~\ref{lem:proj_opt_cond},
\[
\left\langle
-\nabla_\xi f_\xi(\lambda_\theta)
+
\sigma(\xi-\xi_k),
\xi-\xi^+
\right\rangle
\ge
\frac{1}{\alpha}\|\xi-\xi^+\|^2.
\]
Choosing $z'=(\theta,\xi^+)$ in
\eqref{eq:proximal-VI-gap} yields
\[
\operatorname{Gap}_k(z)
\ge
\alpha
\left\|
\mathcal{G}_{\Xi,\sigma}^{k}(\theta,\xi)
\right\|^2.
\]
Evaluating these inequalities at $z=z_{k+1}$, taking conditional expectations, and applying Corollary~\ref{cor:inner-gap-decay} completes the proof.
\end{proof}

\subsection{Proof of Theorem~\ref{thm:nonconcave-convergence}}

Before presenting the final convergence theorem, we establish separate bounds on \(\mathcal{G}_{\Theta}(\theta_{k+1},\xi_{k+1}),\) and \(\mathcal{G}_{\Xi}(\theta_{k+1},\xi_{k+1})\) in the following proposition.

\begin{proposition}
\label{prop:inner-xi-mapping-bound}
Suppose Assumptions~\ref{ass:xi-convex-compact}
to~\ref{ass:xi_lambda_bound}, and 
Assumption~\ref{ass:f-convex-nonconcave} hold. Let
$(\theta_{k+1},\xi_{k+1})$ be generated by
Algorithm~\ref{alg:pe-pgda}. Then
\begin{equation}
\label{eq:orig-xi-mapping-bound-revised}
\mathbb{E}_k
\left[
\left\|
\mathcal{G}_{\Xi}
(\theta_{k+1},\xi_{k+1})
\right\|^2
\right]
\le
\frac{2c_\delta}{\alpha(k+1)}
+
2\sigma^2
\mathbb{E}_k
\left[
\|\xi_{k+1}-\xi_k\|^2
\right].
\end{equation}
\end{proposition}

\begin{proof}
By Lemma~\ref{lem:prox_to_orig},
\begin{align*}
\left\|
\mathcal{G}_{\Xi}
(\theta_{k+1},\xi_{k+1})
\right\|\le
\left\|
\mathcal{G}_{\Xi,\sigma}^{k}
(\theta_{k+1},\xi_{k+1})
\right\|
+
\sigma\|\xi_{k+1}-\xi_k\|.
\end{align*}
Using $(a+b)^2\le 2a^2+2b^2$, taking conditional expectations,
and applying Lemma~\ref{lem:proximal-gap-mapping} gives
\begin{align*}
\mathbb{E}_k
\left[
\left\|
\mathcal{G}_{\Xi}
(\theta_{k+1},\xi_{k+1})
\right\|^2
\right]
&\le
2\mathbb{E}_k
\left[
\left\|
\mathcal{G}_{\Xi,\sigma}^{k}
(\theta_{k+1},\xi_{k+1})
\right\|^2
\right]+
2\sigma^2
\mathbb{E}_k
\left[
\|\xi_{k+1}-\xi_k\|^2
\right]\\
&\le
\frac{2c_\delta}{\alpha(k+1)}
+
2\sigma^2
\mathbb{E}_k
\left[
\|\xi_{k+1}-\xi_k\|^2
\right].
\end{align*}
This completes the proof.
\end{proof}

\begin{proposition}
\label{prop:theta-mapping-bound}
Suppose Assumptions~\ref{ass:xi-convex-compact}
to~\ref{ass:xi_lambda_bound}, and
Assumption~\ref{ass:f-convex-nonconcave} hold. Let
$(\theta_{k+1},\xi_{k+1})$ be generated by
Algorithm~\ref{alg:pe-pgda}. Then
\begin{equation}
\label{eq:theta-mapping-bound-revised}
\mathbb{E}_k
\left[
\left\|
\mathcal{G}_{\Theta}
(\theta_{k+1},\xi_{k+1})
\right\|^2
\right]
\le
\frac{2c_\delta}{\alpha(k+1)}
+
2\sigma^2
\mathbb{E}_k
\left[
\|\theta_{k+1}-\theta_k\|^2
\right].
\end{equation}
\end{proposition}

\begin{proof}
By Lemma~\ref{lem:prox_to_orig},
\begin{align*}
\left\|
\mathcal{G}_{\Theta}
(\theta_{k+1},\xi_{k+1})
\right\|\le
\left\|
\mathcal{G}_{\Theta,\sigma}^{k}
(\theta_{k+1},\xi_{k+1})
\right\|
+
\sigma\|\theta_{k+1}-\theta_k\|.
\end{align*}
Using $(a+b)^2\le 2a^2+2b^2$, taking conditional expectations,
and applying Lemma~\ref{lem:proximal-gap-mapping} gives
\begin{align*}
\mathbb{E}_k
\left[
\left\|
\mathcal{G}_{\Theta}
(\theta_{k+1},\xi_{k+1})
\right\|^2
\right]
&\le
2\mathbb{E}_k
\left[
\left\|
\mathcal{G}_{\Theta,\sigma}^{k}
(\theta_{k+1},\xi_{k+1})
\right\|^2
\right]+
2\sigma^2
\mathbb{E}_k
\left[
\|\theta_{k+1}-\theta_k\|^2
\right]\\
&\le
\frac{2c_\delta}{\alpha(k+1)}
+
2\sigma^2
\mathbb{E}_k
\left[
\|\theta_{k+1}-\theta_k\|^2
\right].
\end{align*}
This completes the proof.
\end{proof}

We now establish the main convergence result.

\begin{proof}
By Propositions~\ref{prop:inner-xi-mapping-bound} and
\ref{prop:theta-mapping-bound}, and property of conditional
expectation,
\begin{align}
\mathbb{E}
\left[
\mathcal{R}_{\rm GM}
(\theta_{k+1},\xi_{k+1})
\right]
&\le
\frac{4\delta_k}{\alpha}
+
2\sigma^2
\mathbb{E}
\left[
\|\theta_{k+1}-\theta_k\|^2
+
\|\xi_{k+1}-\xi_k\|^2
\right]
\nonumber\\
&=
\frac{4\delta_k}{\alpha}
+
2\sigma^2
\mathbb{E}
\|z_{k+1}-z_k\|^2.
\label{eq:RGM-before-outer-telescope}
\end{align}

It remains to control the displacement $z_{k+1}-z_k$. Let
$z^\star$ be the Minty solution in
Assumption~\ref{ass:minty-solution}. By the definition of
$\operatorname{Gap}_k$,
\[
\left\langle
\mathcal{F}_k(z_{k+1}),
z_{k+1}-z^\star
\right\rangle
\le
\operatorname{Gap}_k(z_{k+1}).
\]
Taking conditional expectations and using Corollary~\ref{cor:inner-gap-decay} yields
\begin{equation}
\label{eq:gap-with-minty-point}
\mathbb{E}_k
\left[
\left\langle
\mathcal{F}(z_{k+1})
+
\sigma(z_{k+1}-z_k),
z_{k+1}-z^\star
\right\rangle
\right]
\le
\frac{c_\delta}{(k+1)}.
\end{equation}
By Assumption~\ref{ass:minty-solution},
\[
\left\langle
\mathcal{F}(z_{k+1}),
z_{k+1}-z^\star
\right\rangle
\ge 0.
\]
Therefore, \eqref{eq:gap-with-minty-point} implies
\begin{equation}
\label{eq:outer-cross-term-bound}
\mathbb{E}_k
\left[
\left\langle
z_k-z_{k+1},
z_{k+1}-z^\star
\right\rangle
\right]
\ge
-\frac{c_\delta}{\sigma(k+1)}.
\end{equation}
Using the identity
\begin{align*}
\|z_k-z^\star\|^2=
\|z_k-z_{k+1}\|^2
+
\|z_{k+1}-z^\star\|^2+
2\left\langle
z_k-z_{k+1},
z_{k+1}-z^\star
\right\rangle
\end{align*}
together with \eqref{eq:outer-cross-term-bound}, we obtain
\begin{equation}
\label{eq:outer-step-recursion}
\mathbb{E}_k
\|z_{k+1}-z_k\|^2
\le
\|z_k-z^\star\|^2
-
\mathbb{E}_k
\|z_{k+1}-z^\star\|^2
+\frac{2c_\delta}{\sigma(k+1)}.
\end{equation}
Taking total expectations and summing
\eqref{eq:outer-step-recursion} over
$k=0,\ldots,K-1$ gives
\begin{align}
\sum_{k=0}^{K-1}
\mathbb{E}
\|z_{k+1}-z_k\|^2\le
\|z_0-z^\star\|^2
-
\mathbb{E}\|z_K-z^\star\|^2
+
\frac{2}{\sigma}
\sum_{k=0}^{K-1}\frac{c_\delta}{(k+1)}\le
D_{\mathcal Z}^2
+
\frac{2}{\sigma}
\sum_{k=0}^{K-1}\frac{c_\delta}{(k+1)}.
\label{eq:outer-step-sum-bound}
\end{align}
Since $\widehat{k}$ is sampled uniformly,
\begin{align*}
\mathbb{E}
\left[
\mathcal{R}_{\rm GM}
(\theta_{\widehat{k}+1},\xi_{\widehat{k}+1})
\right]=
\frac{1}{K}
\sum_{k=0}^{K-1}
\mathbb{E}
\left[
\mathcal{R}_{\rm GM}
(\theta_{k+1},\xi_{k+1})
\right].
\end{align*}
Averaging \eqref{eq:RGM-before-outer-telescope} over $k$ and
substituting \eqref{eq:outer-step-sum-bound} gives
\begin{align*}
\mathbb{E}
\left[
\mathcal{R}_{\rm GM}
(\theta_{\widehat{k}+1},\xi_{\widehat{k}+1})
\right]
&\le
\frac{4}{\alpha K}
\sum_{k=0}^{K-1}\frac{c_\delta}{(k+1)}
+
\frac{2\sigma^2D_{\mathcal Z}^2}{K}
+
\frac{4\sigma}{K}
\sum_{k=0}^{K-1}\frac{c_\delta}{(k+1)}\\
&=
\frac{2\sigma^2D_{\mathcal Z}^2}{K}
+
\left(
\frac{4}{\alpha}
+
4\sigma
\right)
\frac{1}{K}
\sum_{k=0}^{K-1}\frac{c_\delta}{(k+1)},
\end{align*}
Finally, 
\[
\sum_{k=0}^{K-1}\frac{c_\delta}{(k+1)}
\le
c_\delta
\sum_{k=1}^{K}\frac{1}{k}
\le
c_\delta(1+\log K).
\]
Substituting this completes the proof. Since
$(1+\log K)/K\to 0$, the expected gradient-mapping residual converges to zero. 
\end{proof}

\section{Additional Implementation Details}
\label{app:implementation}
All experiments were conducted on a single NVIDIA RTX 4090 GPU (24GB memory) and a 13th Gen Intel Core i9-13900KF CPU (32 threads).

\subsection{Implementation Details for LLM Safety Alignment}
\label{sec:llm-implementation}

To verify Algorithm~\ref{alg:pgda} in the LLM alignment setting, we consider the standard constrained RLHF objective in~\citep{dai2023safe}, where the policy is optimized for helpfulness while its expected harmfulness cost is
required to remain below a prescribed threshold. Specifically, let
$r(x,y)$ and $c(x,y)$ denote the reward and cost assigned to a
completion $y$ for prompt $x$, respectively. We define
\begin{align*}
R(\lambda_\theta):=
\mathbb E_{\substack{x\sim\rho,
y\sim\pi_\theta(\cdot\mid x)}}
\left[r(x,y)\right], \quad
C(\lambda_\theta):=
\mathbb E_{\substack{x\sim\rho, 
y\sim\pi_\theta(\cdot\mid x)}}
\left[c(x,y)\right].
\end{align*}
The constrained RLHF alignment problem is defined as
\[
\max_{\theta\in\Theta}
R(\lambda_\theta)
\qquad
\text{s.t.}
\qquad
C(\lambda_\theta)\le d,
\]
where $d$ is a user-specified safety threshold. Based on Proposition~\ref{prop:crl_as_ru}, this is equivalent to the following Lagrangian problem:
\[
\min_{\theta\in\Theta}
\max_{\xi\in\Xi}
f_\xi(\lambda_\theta),
\]
where
\[
f_\xi(\lambda_\theta):=
-R(\lambda_\theta)
+
\xi\left(C(\lambda_\theta)-d\right),
\]
and the multiplier is restricted to the compact convex set \(\Xi:=[0,\xi_{\max}]\) for some $\xi_{\max}>0$. The variable $\xi$ controls the penalty assigned to violations of the safety constraint. When the expected cost exceeds $d$, maximizing over $\xi$ increases the importance of the cost term in the subsequent policy update.

We next describe how Algorithm~\ref{alg:pgda} is implemented in this LLM alignment setting. At each outer iteration $k$, we uniformly sample a batch of $N$ prompts \(\{x_i\}_{i=1}^{N}\) from the training dataset. For each prompt $x_i$, we sample $J$ completions from the current policy: \(y_{i,j} \sim \pi_{\theta_k}(\cdot\mid x_i), j=1,\ldots,J. \)
The corresponding reward and cost values are
\[
r_{i,j}:=r(x_i,y_{i,j}),
\qquad
c_{i,j}:=c(x_i,y_{i,j}).
\]
We estimate the expected reward and cost by
\begin{align*}
\widehat R_k
:=
\frac{1}{NJ}
\sum_{i=1}^{N}
\sum_{j=1}^{J}
r_{i,j}, \quad
\widehat C_k
:=
\frac{1}{NJ}
\sum_{i=1}^{N}
\sum_{j=1}^{J}
c_{i,j}.
\end{align*}
For a fixed policy $\pi_{\theta_k}$, the batchwise approximation of
the inner objective is
\[
\widehat f_\xi(\lambda_{\theta_k})
=-\widehat R_k+\xi\left(\widehat C_k-d\right).
\]
Its gradient with respect to the multiplier \(\xi\) is
\[
\nabla_\xi
\widehat f_\xi(\lambda_{\theta_k})
=
\widehat C_k-d.
\]
Therefore, each projected gradient-ascent step for the inner
optimization in Algorithm~\ref{alg:pgda} becomes
\[
\xi_{k,t+1}
=
\operatorname{proj}_{[0,\xi_{\max}]}
\left(
\xi_{k,t}
+
\beta\left(\widehat C_k-d\right)
\right).
\]
Projection onto $\Xi=[0,\xi_{\max}]$ is implemented by clipping:
\[
\operatorname{proj}_{[0,\xi_{\max}]}(u)
=
\operatorname{clip}
\left(
u,0,\xi_{\max}
\right).
\]
After $T$ inner iterations, we set \(\xi_{k+1}:=\xi_{k,T}.\) For the outer update, given the final multiplier $\xi_{k+1}$, the
batchwise Lagrangian objective is
\[
\widehat f_{\xi_{k+1}}(\lambda_{\theta_k})=-\widehat R_k
+\xi_{k+1}\left(\widehat C_k-d\right).
\]
Because the threshold $d$ is independent of $\theta$, the
corresponding policy gradient is
\begin{align*}
\nabla_\theta
f_{\xi_{k+1}}(\lambda_{\theta_k})=
-\nabla_\theta R(\lambda_{\theta_k})
+
\xi_{k+1}
\nabla_\theta C(\lambda_{\theta_k}).
\end{align*}
A Monte-Carlo estimator of this gradient is
\begin{align*}
\widehat g_k^\theta=
\frac{1}{NJ}
\sum_{i=1}^{N}
\sum_{j=1}^{J}
\left(
-r_{i,j}
+
\xi_{k+1}c_{i,j}
\right)
\nabla_\theta
\log\pi_{\theta_k}(y_{i,j}\mid x_i).
\end{align*}
In practice, rather than applying a vanilla policy-gradient update, we implement the outer update using PPO~\citep{schulman2017proximal}. Specifically, we define the multiplier-adjusted reward \(\widetilde r_{i,j}:=r_{i,j}-\xi_{k+1}c_{i,j},\) and perform one PPO update using the sampled trajectories $\{(x_i,y_{i,j})\}$. 

To evaluate the convergence behavior of
Algorithm~\ref{alg:pgda}, at each outer iteration we estimate the
Lagrangian objective using an independent evaluation batch:
\[
\widehat f_{\xi_{k+1}}^{\mathrm{eval}}(\lambda_{\theta_k})
=
-\widehat R_k^{\mathrm{eval}}
+\xi_{k+1}(
\widehat C_k^{\mathrm{eval}}-d).
\]
Here, \(c^{(0)}\) corresponds to \(-\widehat R^{\mathrm{eval}}\), while \(c^{(1)}\) correspond to \(\widehat C^{\mathrm{eval}}\), as illustrated in Figure~\ref{fig:convergence-results}, and we set $d=0$. All PPO and sampling hyperparameters used in the experiments are reported in Table~\ref{tab:ppo_hyperparameter}. In particular, we fine-tune a Pythia-70m model~\citep{biderman2023pythia} on the PKU-SafeRLHF-10k dataset~\citep{dai2023safe}, a human-preference dataset designed to evaluate safety alignment in LLM. For the reward signal that quantifies helpfulness, we use the \texttt{PKU-Alignment/beaver-7b-unified-reward}. As for the cost signal that quantifies harmfulness, we use the \texttt{PKU-Alignment/beaver-7b-unified-cost}~\citep{dai2023safe}. Both models are trained on the PKU-SafeRLHF dataset. We run the experiment using five random seeds and report the mean and standard deviation across runs.

\begin{table}[ht]
    \centering
    \caption{Hyperparameters utilized during the Algorithm~\ref{alg:pgda} training process}
    \label{tab:ppo_hyperparameter}
    \begin{tabular}{ll}
        \toprule
        \multicolumn{2}{c}{\textbf{Inner loop hyperparameters}} \\
        \midrule
        Step size $\beta$ & 0.5 \\
        Iteration number $T$ & 10 \\
        Num responses per prompt $J$ & 8 \\
        Safety threshold $d$ & 0\\
        \midrule
        \multicolumn{2}{c}{\textbf{PPO hyperparameters}} \\
        \midrule
        Training strategy & LoRA \\
        LoRA\_r & 16 \\
        LoRA\_alpha & 32 \\
        LoRA\_dropout & 0.05 \\
        Max steps & 20000 \\
        Learning rate & 1e-5 \\        
        Batch size & 16 \\
        Gradient accumulation steps & 1 \\
        PPO epochs & 4 \\
        Target KL & 0.1\\
        Seed & 0 \\
        Init KL coef & 0.2\\
        Adap KL ctrl & True \\
        Top k & 0 \\
        Top p & 1 \\
        Do Sample & True \\
        Temperature & 1 \\
        \bottomrule
    \end{tabular}
\end{table}

\subsection{Implementation Details for Exploration Maximization}
For this experiment, we follow the synthetic tabular RL environment of~\citep{chen2024robust}. We consider a state space $\mathcal{S}=\{1,2,\ldots,S\}$ with $S=10$ and an action space $\mathcal{A}=\{1,2,\ldots,A\}$ with $A=5$. We use a discount factor $\gamma=0.95$ and a uniform initial state distribution $\rho$. The transition kernel $p(\cdot\mid s,a)$ is fixed throughout training and is generated by sampling an i.i.d. strictly positive table and normalizing each $(s,a)$ row to obtain a valid distribution over next states. We use a direct logit parameterization $\theta \in \mathbb{R}^{\mathcal{S}\times \mathcal{A}}$ and define
\[
   \pi_\theta(a|s) = \frac{\exp(\theta_{s,a})}{\sum_{a'}\exp(\theta_{s,a'})}.
\]
To estimate the state-action occupancy measure (Appendix~\ref{app:grad_estimate}), we sample $m=m'=256$ trajectories of horizon $H=H'=50$ from the policy. 

For the utility function, we use a fixed feature map $\psi(s,a)\in \mathbb{R}^{d'}$ with $d'=20$, generated by randomizing the entries. The uncertainty matrix is $W_\xi \in \mathbb{R}^{d\times d'}$ with $d=10$. We constrain the uncertainty set to a Frobenius ball around a nominal matrix $\tilde W$~\citep{boyd2004convex}:
\[
   \Xi = \{W: \|W-\tilde W \|_F \le \epsilon\}.
   \]
We take $\tilde W$ as a partial identity mapping (i.e., $\tilde W(i,i)=1$ for $i \le \min(d,d')$ and 0 otherwise), and project each inner update onto $\Xi$ using the closed-form projection onto a Frobenius ball. We set $\epsilon=1.0$.

We use REINFORCE policy gradients~\citep{williams1992simple} with step size $\alpha=0.03$, set $T=6$, and set the proximal coefficient to $\sigma=1.5$. Results are reported over $K=300$ outer iterations across 5 random seeds.

To empirically evaluate convergence, we track:
\begin{itemize}
   \item The objective evaluated at each algorithm iterate $f_{\xi}(\lambda_{\theta})$.
   \item The norm of the prox-gradient mapping for the outer variable $\theta$ at iteration $k$: $\|\mathcal{G}^{k}_{\Theta,\sigma}(\theta,\xi)\|$, which tracks the progress toward a stationary solution for the constrained outer maximization.
   \item The norm of the prox-gradient mapping for the inner variable $\xi$ at iteration $k$: $\|\mathcal{G}^{k}_{\Xi,\sigma}(\theta,\xi)\|$, which tracks the progress toward a stationary solution for the constrained inner maximization.
\end{itemize}

\end{document}